\documentclass{article}

\usepackage[main,final]{neurips_2026}

\usepackage[utf8]{inputenc} 
\usepackage[T1]{fontenc} 
\usepackage{hyperref} 
\hypersetup{hidelinks} 
\usepackage{url} 
\usepackage{booktabs}
\usepackage{wrapfig}
\usepackage{amsfonts} 
\usepackage{nicefrac} 
\usepackage{microtype} 
\usepackage[table]{xcolor}

\usepackage{array}
\newcolumntype{C}[1]{>{\centering\arraybackslash}p{#1}}
\newcolumntype{G}[1]{>{\columncolor{gray!10}\centering\arraybackslash}p{#1}}

\usepackage{amsmath, amssymb}
\usepackage{algorithm}
\usepackage{algpseudocode}

\usepackage{graphicx}

\usepackage{enumitem}

\usepackage{amsthm}
\newtheorem{assumption}{Assumption}

\newtheorem{definition}{Definition}
\newtheorem{proposition}{Proposition}
\newtheorem{lemma}{Lemma}

\newtheorem{remark}{Remark}
\newtheorem{theorem}{Theorem}
\newtheorem{principle}{Rule}

\usepackage{multirow} 

\usepackage{titletoc}

\usepackage[most]{tcolorbox}

\newlength{\rulelabelwidth}
\newtcolorbox{leftrulebox}{
 colback=gray!12,
 colframe=gray!50,
 boxrule=0.35pt,
 arc=2pt,
 left=4pt,right=4pt,top=2pt,bottom=2pt,
 boxsep=1pt
}
\definecolor{PyBlue}{HTML}{1f77b4}
\definecolor{PyRed}{HTML}{d62728}

\newcommand{\Oc}{\mathcal{O}}
\newcommand{\bx}{\mathbf{x}}
\newcommand{\bw}{\mathbf{w}}

\newcommand{\bu}{\mathbf{u}}
\newcommand{\be}{\mathbf{e}}
\newcommand{\bb}{\mathbf{b}}
\newcommand{\bd}{\mathbf{d}}
\newcommand{\R}{\mathbb{R}}
\newcommand{\X}{\mathcal{X}}
\newcommand{\U}{\mathcal{U}}
\newcommand{\XPF}{{\mathcal{X}_{\rm PF}}}
\newcommand{\LS}{\operatorname{LS}}
\newcommand{\f}{\mathbf{f}}
\newcommand{\h}{\mathbf{h}}
\newcommand{\fPF}{{\mathbf{f}_{\rm PF}}}
\newcommand{\Farc}{\Phi}
\newcommand{\hatFarc}{\hat{\Phi}}
\newcommand{\Diag}{\operatorname{Diag}}

\newcommand{\SoftMax}{\operatorname{SoftMax}}
\newcommand{\Proj}{\operatorname{Proj}}
\newcommand{\dist}{\operatorname{dist}}
\newcommand{\E}{\mathbb{E}}

\title{SURF: Steering the Scalarization Weight to Uniformly Traverse the Pareto Front}

\author{%
 Liuyuan Jiang, Chentong Huang, Lisha Chen
 \\
 Department of Electrical and Computer Engineering\\
 University of Rochester, Rochester, NY 14627 \\
 \texttt{\{ljiang24, chuang80\}@ur.rochester.edu, lisha.chen@rochester.edu}
}

\begin{document}
\maketitle

\vspace{-0.1cm}
\begin{abstract}
\vspace{-0.2cm}

Scalarization is widely used in multi-objective optimization owing to its simplicity and scalability. In many applications, the goal is to generate  solutions that represent diverse user preferences, ideally with uniform coverage of the Pareto front (PF).
However, uniformly sampling scalarization weights usually induces non-uniform coverage of the PF. 
We explain this mismatch through the geometry of the scalarization path. As the weight varies, the corresponding solutions traverse the PF at a generally nonuniform speed. Integrating this speed gives cumulative arc length. After normalization by the total length, it defines a cumulative distribution function (CDF) on weights whose inverse selects points at equal arc-length intervals.
Building on this
insight, we propose \textbf{SURF} (\textbf{S}ampling \textbf{U}niformly along the
Pa\textbf{r}eto \textbf{F}ront). 
For structured problems, including bi-objective bandits, we derive analytical expressions for this CDF map and the resulting PF-aware weight sampling rule. For general problems, SURF alternates between CDF reconstruction and weight sampling. 
Theoretically, we show that under provable conditions for bi-objective LS, SURF's CDF reconstruction error converges linearly to an $\Oc(N^{-2})$ discretization floor.
Empirically, experiments on bandits, multi-objective-gymnasium, and multi-objective LLM alignment demonstrate that SURF efficiently achieves more uniform PF coverage than baselines.
The code is available at \url{https://github.com/Liuyuan999/MOO_Uniform_PF}.

\end{abstract}

\section{Introduction}
\label{sec:intro}

Multi-objective optimization (MOO) is increasingly central to modern machine learning and large language model (LLM) alignment, where one must balance competing criteria such as accuracy, robustness, fairness, and safety rather than optimize a single scalar reward~\citep{shi2024decoding,yang2024metaaligner,yang2024rewards,zhang2024libmoon}. 
\begin{wrapfigure}{r}{0.33\textwidth}
\vspace{-0.25cm}
\centering
\includegraphics[width=0.33\textwidth]{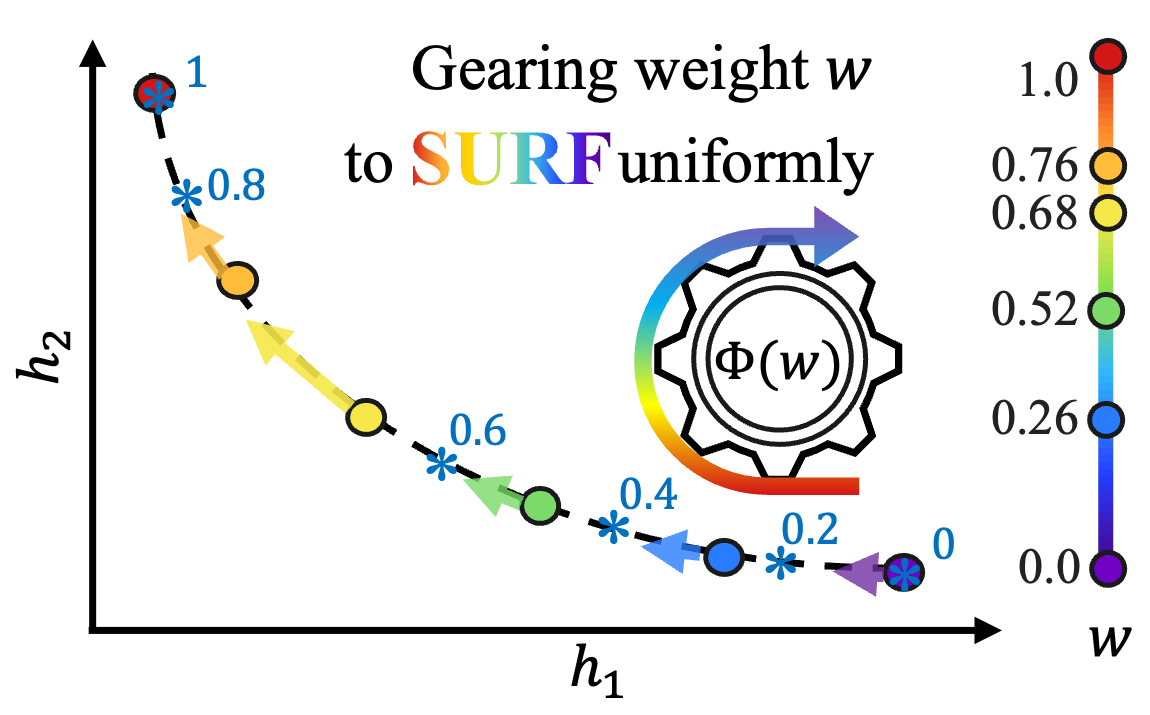}
\vspace{-0.7cm}
\caption{Multi-armed bandit LS results (details in Appendix~\ref{app:toy}). Uniform weights produce uneven points on PF (\textcolor{PyBlue}{stars}), while PF-aware sampling yields uniformly spaced colored circles.}
\label{fig:density_mismatch}
\vspace{0.3cm}
\end{wrapfigure}
Formally, let $\bu$ denote the decision variable in the feasible set $\U\subseteq\mathbb{R}^d$. 
MOO seeks to minimize $M$ objectives $h_m:\mathbb{R}^d \rightarrow \mathbb{R}$:

\vspace{-0.6cm}
\begin{align}
\min_{\bu \in \U} \left\{\h(\bu)
:= \begin{bmatrix} h_1(\bu), \cdots, h_M(\bu)\end{bmatrix} ^\top \right\}.
\label{eq:BOO}
\end{align}

\vspace{-0.3cm}
We consider optimizing the above vector-valued objective to Pareto optimality.
Intuitively, a solution $\bu^*$ is \emph{Pareto optimal} if no feasible point improves one objective without worsening another. The set of all Pareto-optimal solutions is the \emph{Pareto set} $\U_{\rm PF}$, and its image $\mathcal{F}_{\rm PF}$ under $\h$ is the Pareto front (PF). Formal definitions are provided in Section~\ref{sec:arc-length-geometry}.
The PF thus provides a compact description of trade-offs via $M$ objective values, typically with $M\ll d$.
In \emph{a priori} decision making, known preferences guide optimization toward a desired trade-off. In \emph{a posteriori} decision making, users choose after inspecting a representative set of PF solutions, motivating uniform coverage.
A common practice to solve MOO is via \emph{scalarization}. For example, in the bi-objective case ($M=2$), given $w\in[0,1]$, linear scalarization (LS) solves
\vspace{-0.1cm}
\begin{align}
 \bu_w^*
 \in
 \arg\min_{\bu \in \U}
 \big\{\LS_{\mathbf h}(\bu;w)
 :=
 w h_1(\bu) + (1-w) h_2(\bu)
 \big\}.
 \label{eq:LS h}\vspace{-0.2cm}
\end{align}
Under certain conditions (see~Section~\ref{sec:arc-length-geometry}), sweeping $w$ over $[0,1]$ traces out the entire PF, so the standard baseline uses evenly spaced weights~\cite{rame2023rewarded}. Yet, as Figure~\ref{fig:density_mismatch} shows, this baseline rarely yields evenly spaced PF points: solutions could clump in some regions of the front and leave others almost untouched. Existing remedies are either substantially more involved than scalarization~\citep{das1998normal, lovison2011singular, martin2016continuation, vieira2012multicriteria, wang2024newton, zhang2024gliding} or offer no coverage guarantee in limited-computation-budget regimes~\citep{barrett2008learning, chen2024ferero,navon2020learning, zhang2024gliding}. 
This motivates the question:

\begin{center}
\vspace{-0.2cm}
\emph{Can we \textbf{efficiently} find solutions with objectives \textbf{uniformly distributed along the PF}?}
\vspace{-0.15cm}
\end{center}

\begin{wrapfigure}{r}{0.65\textwidth}
\vspace{-0.3cm}
\centering
\includegraphics[width = 0.645\textwidth]{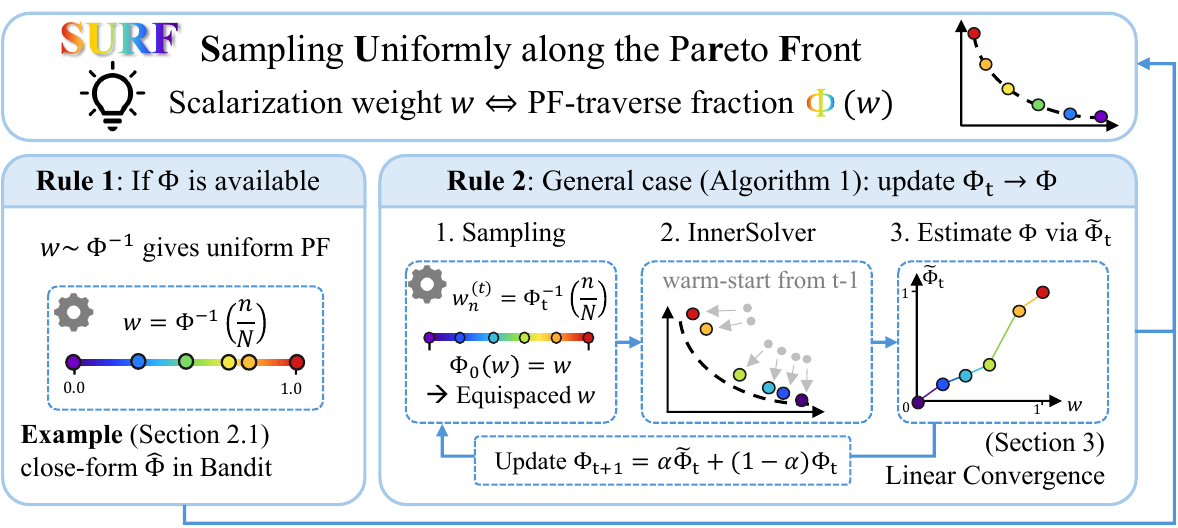}
\vspace{-0.6cm}
\caption{
SURF overview: bridging $w$ and PF points via CDF $\Farc$ yields two rules.
Rule~\ref{prin:pf_aware_ls_weighting} directly samples $\{w_n\}_{n=0}^N$ when $\Farc$ is known (Section~\ref{sec:arc-length-geometry}).
Rule~\ref{prin:surf_general} handles the general case by iteratively approximating $\Farc$ (Algorithm~\ref{alg:SURF}).
}
\vspace{-0.3cm}
\label{fig:SURFoverview}
\end{wrapfigure}

We answer this by proposing \textbf{S}ampling \textbf{U}niformly along the Pa\textbf{r}eto \textbf{F}ront (SURF; Figure~\ref{fig:SURFoverview}), a plug-in wrapper around scalarization that leaves the inner solver untouched with preserved efficiency, and only adjusts \emph{how the weights are sampled}. 
The key insight is the relation between the weight $w$ and the PF geometry. Picture the scalarization weight $w$ as a dial that drags a point along the PF: as $w$ slides from $0$ to $1$, the LS solution $\bu_w^*$ traces the entire front via $\fPF(w):=\h(\bu_w^*)$. The catch is that the dial turns the front at uneven speed, i.e., in some regions a small change in $w$ sweeps a long stretch of PF, while in others a large change barely budges the point (cf. Figure~\ref{fig:density_mismatch}).
To precisely characterize this relation, define the \emph{traversal speed} $v(w):=\bigl\|\tfrac{\partial}{\partial w}\fPF(w)\bigr\|$ that captures how fast the PF geometry changes at $w$, the \emph{PF arc length} $s(w):=\int_0^w v(p)\,dp$ that captures the traversal distance along the PF when turning the weight from $0$ to $w$, and the normalized arc-length cumulative distribution function (CDF) 
$\Farc(w):=s(w)/s(1)$ that specifies the fraction of the PF reached at $w$.
The Euclidean norm measures PF traversal in $(h_1,h_2)$, using the compact objective-space representation with $M\ll d$.
The cure mirrors the diagnosis: instead of turning the dial uniformly, SURF turns it through $\Farc^{-1}$. This yields the following PF-geometry-aware weight sampling rule:

\vspace{-0.1cm}
\begin{leftrulebox}
\begin{principle}
\label{prin:pf_aware_ls_weighting}
If $\Farc$ is available, SURF chooses $\{w_n\}_{n=0}^N$ given by $w_n=\Farc^{-1} \left(\frac{n}{N}\right)$.
\end{principle}
\end{leftrulebox}
A priori preferences between objectives can similarly be expressed as desired normalized arc-length positions $q_n\in[0,1]$ on the PF, giving $w_n=\Farc^{-1}(q_n)$.
Section~\ref{sec:MDP} carries this out in structured settings: for bi-objective reinforcement learning, 
we derive a closed-form $v(w)$ in terms of $\bu_w^*$, and an analytically available $\Farc$ in the bandit specialization. For general problems where $\Farc$ is not analytically available, Section~\ref{sec:main_result} introduces SURF (Algorithm~\ref{alg:SURF}), which alternates between solving the scalarized problem at the current set of weights and refining an estimate of $\Farc$ that drives the next set of refined weights. This yields the second joint estimation and sampling rule:
\vspace{-0.1cm}
\begin{leftrulebox} 
\begin{principle}
\label{prin:surf_general}
If $\Farc$ is not available, SURF estimates $\Farc$ and the weights $\{w_n\}_{n=0}^N$ jointly.
\end{principle}
\end{leftrulebox}
SURF estimates $s(w)$ by accumulating neighboring objective-space distances, leaving parameter-space optimization to the inner solver. Our contributions are summarized as follows:

\vspace{-0.2cm}
\begin{enumerate}[leftmargin=*, align=left]
\item[\textbf{C1) PF geometry characterization.}] In Section~\ref{sec:arc-length-geometry}, we characterize the scalarization-induced PF geometry through the traversal speed $v(w)$ and CDF $\Farc$, leading to the PF-aware weight sampling principle in Rule~\ref{prin:pf_aware_ls_weighting}, and use the bandit setting as a concrete instantiation of this principle with analytical expressions (Section~\ref{sec:MDP}).

\item[\textbf{C2) SURF algorithm with convergence guarantee.}] In Section~\ref{sec:main_result}, we propose SURF (Algorithm~\ref{alg:SURF}), which realizes Rule~\ref{prin:surf_general} by reconstructing the unknown CDF $\Farc$ from finitely many inexact, warm-started scalarized solves. Theorem~\ref{thm:cdf} shows that the CDF reconstruction error converges to an unavoidable $\Oc(N^{-2})$ finite-sampling floor.

\item[\textbf{C3) Empirical verification.}] In Section~\ref{sec:exp}, we empirically verify SURF on MO-Gymnasium benchmarks and LLM alignment tasks using different inner solvers, with additional Tchebycheff scalarization tests in Appendix~\ref{app:nonconvex_tchebycheff}.
SURF improves uniform PF coverage while maintaining or improving PF approximation quality over representative baselines under matched training budgets.
\end{enumerate}

\paragraph{Related Work.}
Many practical applications involving multiple criteria decision making or multi-task learning, such as multi-objective reinforcement learning and LLM alignment, require adaptation to or representation of different user preferences or objective trade-offs rather than a single preferred solution on the PF~\citep{reymond2022pareto,felten2024multi,lin2024policy,shi2024decoding,guo2024controllable}. A classical approach to solving MOO is scalarization, which converts the original vector-valued objective into a scalar-valued objective through scalarization. This includes LS, Tchebycheff scalarization~\citep{jahn1985scalarization,lin2024smooth}, $\epsilon$-constraint method~\citep{chankong2008multiobjective}, etc. Among these, LS is a widely used baseline~\citep{zhou2024beyond,yang2024learning,li2025gradient}. A separate line of work studies how to obtain diverse or well-distributed solutions along the PF. Representative approaches include adaptive scalarization methods~\citep{barrett2008learning,alegre2022optimistic}, which adjust scalarization weights to discover diverse trade-offs; preference-guided methods~\citep{das1998normal,lin2019pareto,mahapatra2020multi,chen2024pareto,zhang2024gliding}, which steer the search using prescribed reference directions or preference vectors; and continuation or equal-spacing methods~\citep{zhang2006minmax,pereyra2013equispaced,lovison2011singular}, which explicitly trace the PF and attempt to maintain approximately even spacing between neighboring solutions. These methods typically enforce uniformity indirectly through weight selection, reference-point design, or local spacing control, as presented in Figure~\ref{fig:uniformity_baselines}.
OLS targets a convex coverage set~\citep{barrett2008learning,alegre2022optimistic}, while NBI introduces auxiliary constrained subproblems~\citep{das1998normal}. SURF instead corrects sampling density through distribution matching while retaining the scalarized solver. The map from scalarization weight to Pareto point induces a pushforward measure on the front. If we sample scalarization weights from a given distribution, the resulting Pareto points follow a distribution on the PF~\citep{gonzalez2009digital,huang2010adaptive,weller2016mesh}. Uniform coverage can then be interpreted as matching this induced measure to a uniform arc‑length target. An extended discussion of related work is deferred to Appendix~\ref{app:related_work}.
\begin{figure}
    \centering
    \includegraphics[width=0.99\linewidth]{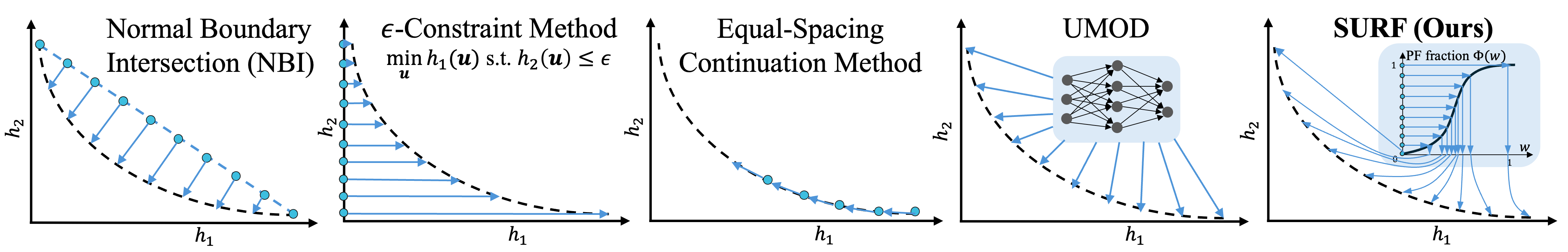}
    \caption{
    Comparison of approaches for obtaining diverse PF solutions, including NBI method~\citep{das1998normal}, $\epsilon$-constraint method~\citep{chankong2008multiobjective},  continuation/equal-spacing methods~\citep{pereyra2013equispaced,lovison2011singular}, and UMOD~\citep{zhang2024gliding}. 
    These methods promote coverage through auxiliary constraints, directions, path tracing, or learned mappings, whereas SURF directly matches the scalarization weights to the uniform PF distribution.
    }
    \label{fig:uniformity_baselines}
\end{figure}

\paragraph{Notation.} We denote $\mathbb R_+=[0,\infty)$ and $\mathbb R_-=(-\infty,0]$. 
For a set $\mathcal X$, $\operatorname{relint}(\mathcal X)$ is its relative interior and $\Proj_{\mathcal X}(\bx)$ is the Euclidean projection. 
For $n\in\mathbb N$, write $[n]=\{1,\ldots,n\}$, $\Delta_n=\{\bw\in\mathbb R_+^n:\sum_{i=1}^n w_i=1\}$, and $\Delta_n^\circ=\operatorname{relint}(\Delta_n)$. 
For $\mathbf p,\mathbf q\in\Delta_n$, we write $\mathrm{KL}(\mathbf p\|\mathbf q):=\sum_i p_i\log(p_i/q_i)$ for the Kullback-Leibler (KL) divergence, with the analogous integral form for continuous distributions.
For vector $\bx$, $\|\bx\|$ is the Euclidean norm, $\Diag(\bx)$ is its diagonal matrix, and $\SoftMax(\bx)_i={\exp(x_i)}/{\sum_j\exp(x_j)}$, with binary special case $\sigma(x)=(1+e^{-x})^{-1}$. 
For matrix $H$, $\|H\|$ is the spectral norm,
$\mathrm{vec}(H)$ is its row-wise vectorization, and $\ker(H)$ is its null space. 
Operations $1/(\cdot)$, $\log(\cdot)$, and $\exp(\cdot)$ are applied elementwise. 
Detailed problem-specific notation is summarized in Appendix~\ref{app:notation}.

\section{Pareto Front Geometry Analysis and Geometry-Aware Scalarization}
\label{sec:arc-length-geometry}
We use LS as a concrete example for its tractable geometry and full PF reachability under hidden convexity.
To make the dial-speed picture precise, we now formalize the regularity conditions under which $\bu_w^*$ is well-defined and the PF traversal speed is smooth.
We first provide basic definitions for MOO, then discuss additional assumptions for our problem setup. 
\begin{definition}[Pareto optimality and Pareto front]
\label{def:po_pf}
A feasible point $\bu^*\in\U$ is \emph{Pareto optimal} if there does not exist $\bu\in\U$ such that $h_m(\bu)\le h_m(\bu^*)$ for all $m\in[M]$ and $h_j(\bu)<h_j(\bu^*)$ for at least one $j\in[M]$. 
The collection of all Pareto-optimal solutions forms the \emph{Pareto set} (PS), denoted as $\U_{\rm PF}$, and its image in objective space forms the \emph{Pareto front} (PF), $\mathcal F_{\rm PF}:=\{\h(\bu^*):\bu^*\in\U_{\rm PF}\}$.    
\end{definition}

\begin{wrapfigure}{r}{0.36\textwidth}
 \centering
 \includegraphics[width=0.355\textwidth]{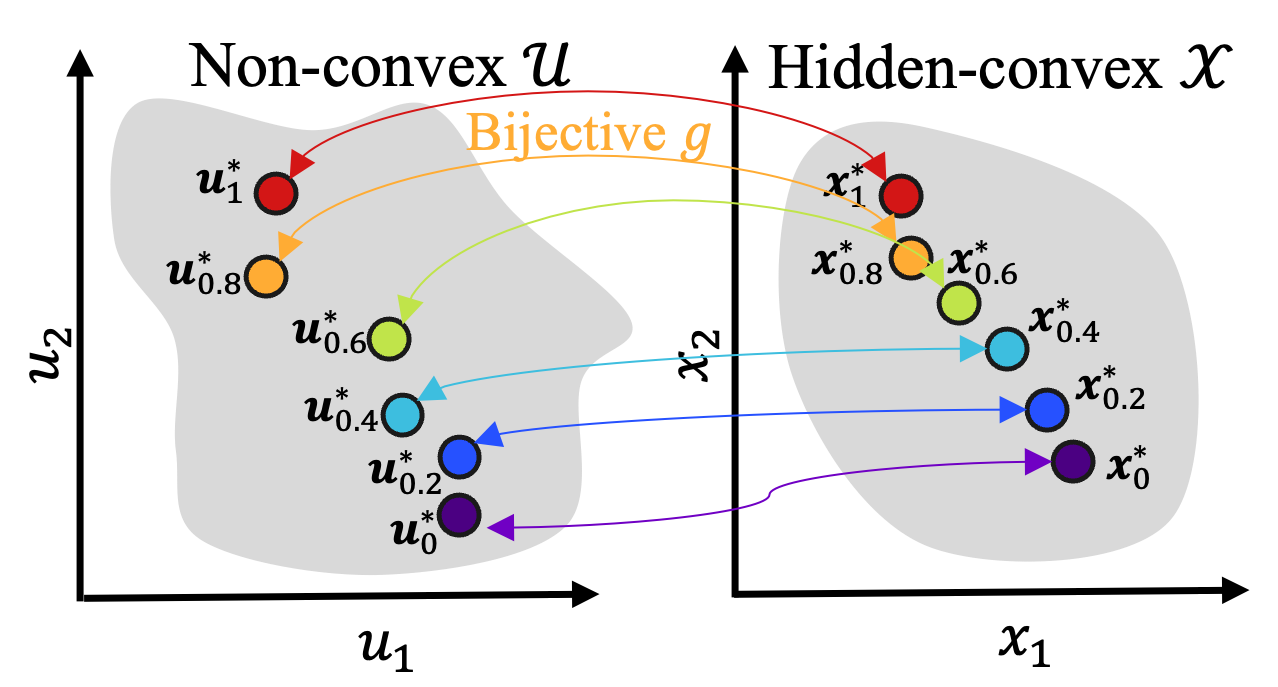}
 \vspace{-0.7cm}
 \caption{
 Original parameter space $\mathcal U$ and hidden space $\mathcal X=g(\mathcal U)$.
 Bijection $g$ maps non-convex structures in $\mathcal U$ to a convex space $\mathcal X$.%
 }
\label{fig: parameter_space}
\vspace{-0.5cm}
\end{wrapfigure}

Modern machine-learning objectives are highly non-convex in their original parameterization. Nevertheless, many problems admit an equivalent hidden representation where the objective becomes convex or strongly convex~\citep{petrulionyte2024functional,fatkhullin2025stochastic}. Based on this viewpoint, we impose the following assumption.

\begin{assumption}[Hidden-space regularity] 
\label{assump:hidden_regular} 
Assume there exists a bijection $g:\mathcal U\to\mathcal X$ such that $h_m(u)=f_m(g(u))$ for each $m\in[M]$. Each $f_m$ is twice continuously differentiable on an open set containing $\X$ and is $\mu$-strongly convex on $\mathcal X$. The hidden PS $\XPF =g(\U_{\rm PF})$ lies in a compact subset $\mathcal K\subset \mathcal X$ such that each $\nabla f_m$ is $l_{f,1}$-Lipschitz-continuous on $\mathcal K$. 
\end{assumption}
For local sensitivity along a regular solution branch, strong convexity can be replaced by a positive-definite reduced Lagrangian Hessian~\citep{fiacco1976sensitivity}.
This assumption is provable in a range of settings, including nonlinear least squares~\citep{nesterov2006cubic,drusvyatskiy2019efficiency}, power control and optimal doping profile problems~\citep{boyd2007tutorial}, system-level synthesis in control~\citep{anderson2019system}, revenue management and inventory control~\citep{chen2025efficient,chen2023network}, and reinforcement learning~\citep{zhang2020variational,schlaginhaufen2024towards}, which we revisit as a concrete example in Section~\ref{sec:MDP}.

Denote $\bx=g(\bu)$, $\X=g(\U)$, and $\f(\bx)
:=[f_1(\bx),\dots,f_M(\bx)]^\top$. Then MOO in \eqref{eq:BOO} admits the hidden-space form $\min_{\bx \in \X}~ \f(\bx)$.
Hidden strong convexity~\citep{fatkhullin2025stochastic} yields that for any weight $\bw\in \Delta_M$, there is $\bx^*_\bw=g(\bu^*_\bw)$ where
\begin{align}
\bx_{\bw}^*\in\arg\min_{\bx \in \X}\big\{\LS_{\mathbf f}(\bx;\bw)
:=\textstyle\sum_{m=1}^M w_m f_m(\bx)\big\}.
\label{eq:LS_multi}
\end{align}
Thus, the original problem~\eqref{eq:BOO} and the hidden formulation~$\min_{\bx \in \X}~ \f(\bx)$ are equivalent in the sense of PS and PF. Formally, $\X_{\rm PF} = g(\U_{\rm PF})$ and $\mathcal F_{\rm PF}=\{\f(\bx): \bx\in\mathcal X_{\rm PF}\}$.
Furthermore, the hidden (strong) convexity in Assumption~\ref{assump:hidden_regular} 
ensures unique scalarized solutions, full exploration of the Pareto front, and a well-defined Pareto-front path.
This is summarized in Proposition~\ref{prop:hidden_cvx_imply}.
\begin{proposition}[{\citep[Theorem~3.5,~Proposition~3.9]{ehrgott2005multicriteria},~\citep[Theorem 3.1.3]{miettinen_nonlinear_1998}}]
\label{prop:hidden_cvx_imply}
Under Assumption~\ref{assump:hidden_regular}, given $\mathbf{w}\in \Delta_M$, solution~\eqref{eq:LS_multi} is unique. {Every such minimizer $\bx_{\bw}^*$ given $\mathbf{w}\in \Delta_M$ is Pareto optimal for $\mathbf{f}$, and every Pareto-optimal point can be obtained by \eqref{eq:LS_multi}} with some $\mathbf{w}\in \Delta_M$.
\end{proposition}
Proposition~\ref{prop:hidden_cvx_imply} implies that, under Assumption~\ref{assump:hidden_regular}, \emph{steering the scalarization weight $\bw$ steers the solution along the PF}. Through the bijection $\bx=g(\bu)$, the unique global minimizer $\bx_{\bw}^*$ in the hidden convex formulation corresponds to the unique global minimizer $\bu_{\bw}^*=g^{-1}(\bx_{\bw}^*)$ in the original parameterization, with the same objective vector. Since every Pareto-optimal point is obtained for some $\bw\in\Delta_M$, sweeping the weights reaches the entire PF.

In the bi-objective case, writing $\bw=[w,1-w]^\top$ recovers~\eqref{eq:LS h}. The induced PF map is
\begin{align}
\fPF(\bw):=\h(\bu_\bw^*)=\f(\bx_\bw^*)\in\mathbb R^M,\quad \bw\in\Delta_M.
\end{align}
Below we provide a concrete example in the bi-objective case.

\subsection{Instantiation: Bi-Objective Reinforcement Learning and Bandits}
\label{sec:MDP}

In the bi-objective case, we identify $\bw=[w,1-w]^\top\in\Delta_2$ with the scalar weight $w\in[0,1]$, and write $\fPF(w):=\fPF([w,1-w]^\top)$, which is a one-dimensional path. To ensure a differentiable and smooth PF path, we impose a standard condition in PF-tracing analysis~\citep{schtze2005continuation,beltran2020pareto}.
\begin{assumption}[Constraint Regularity]
\label{assump:LICQ}
Assume $\X=\{\bx:\mathbf c_1(\bx)\le 0,\mathbf c_2(\bx)= 0\}$, where $\mathbf c_1(\bx)$ is twice differentiable, Lipschitz-smooth and convex and $\mathbf c_2(\bx)$ is linear. For $\min_{\bx\in\X}\LS_f(\bx;w)$ with any $w\in[0,1]$, assume that the active constraints are fixed, satisfy the Linear Independence Constraint Qualification (LICQ), and the Lagrangian Hessian is Lipschitz-continuous along $\{\bx_w^*:w\in[0,1]\}$.
\end{assumption}
Under Assumptions~\ref{assump:hidden_regular} and~\ref{assump:LICQ}, the path $w\mapsto\fPF(w)$ is differentiable, with its first coordinate nonincreasing and its second nondecreasing (cf. Appendix~\ref{app:fPFdifferentiable}). Its geometry is described by
\begin{align}
\underbrace{v(w):=\Big\|\frac{\partial}{\partial w}\fPF(w)\Big\|}_{\text{traversal speed}},\quad
\underbrace{s(w):=\int_0^w v(p)\,\mathrm dp}_{\text{cumulative arc length}},\quad
\underbrace{\Farc(w):=\frac{s(w)}{s(1)}}_{\text{normalized arc-length CDF}},
\quad w\in[0,1].
\label{eq:arc_cdf_def}
\end{align}
The speed measures the local motion of the PF point, and $s(w)$ accumulates the distance traveled from weight $0$ to $w$. When $s(1)>0$, $\Farc(w)$ is the fraction of the total PF arc length traversed by weight $w$. Viewed as a CDF on weights, it corresponds to density $v(w)/s(1)$, which induces uniform arc-length sampling on the PF. Uniform weight sampling generally gives a different distribution because $v(w)$ varies along the path. The nondegeneracy condition below ensures that $\Farc$ is strictly increasing and can be inverted.

To ensure non-vanishing PF traversal speed, we impose the following nondegeneracy assumption.
\begin{assumption}[Nonzero trade-off]
\label{assump:tangent_gap}
For each $w\in[0,1]$, let $J_w$ denote the Jacobian matrix of the active constraints of the scalarized problem $\min_{\bx\in\X}\LS_f(\bx;w)$ at $\bx_w^*$. Assume there exists some constant $c>0$ such that $\big\| \Proj_{\ker(J_w)}\big(\nabla f_1(\bx_w^*)-\nabla f_2(\bx_w^*)\big)\big\|\ge c$ for all $w\in[0,1]$.
\end{assumption}
Assumption~\ref{assump:tangent_gap} rules out zero trade-offs along feasible directions, which would prevent changes in the scalarization weight from moving the solution along the PF. Such nondegeneracy-type conditions are standard in trade-off analysis and PF continuation analyses~\citep{eskelinen2012trade,bolten2021tracing}. This condition ensures that $\Farc$ is invertible. More discussion is deferred to Appendix~\ref{appendix:v_bound_from_tangent_gap}. 
\begin{lemma}[Bounded PF traversal speed]
\label{lem:v_bound_from_tangent_gap}
Suppose Assumptions~\ref{assump:hidden_regular},~\ref{assump:LICQ},~\ref{assump:tangent_gap} hold. Then, there exist constants $0<v_{\min}\le v_{\max}<\infty$ such that $v(w)\in [v_{\min},v_{\max}]$ for all $w\in[0,1]$. Consequently, $s(1)=\int_0^1 v(w) dw\in [v_{\min},v_{\max}]$, and $s(w)$ and $\Farc(w)$ are strictly increasing on $[0,1]$.
\end{lemma}
The proof of Lemma~\ref{lem:v_bound_from_tangent_gap} is given in Appendix~\ref{appendix:v_bound_from_tangent_gap}. Therefore, given a budget of $(N+1)$ PF points, we can choose weights $w_n=\Farc^{-1}\left(\frac{n}{N}\right)$ for $n=0,\ldots,N$, so that the induced points $\{\f(\bx_{w_n}^*)\}_{n=0}^N$ are uniformly spaced along the PF, i.e., $s(w_{n+1})-s(w_n)=\frac{s(1)}{N}$. 
This enables Rule~\ref{prin:pf_aware_ls_weighting} when the PF geometry is explicit, covering analytic examples such as quadratic problems in Appendix~\ref{app:quadratic}, as well as practical multi-objective reinforcement learning (MORL) problems~\citep{roijers2013survey,zhou2024beyond}. 

In bi-objective RL, decision problems are commonly modeled as Markov Decision Processes (MDPs) $\mathcal M=(\mathcal S,\mathcal A,\mathcal P,\mathbf r,\gamma,\rho,\tau)$, where $\mathcal S$ and $\mathcal A$ are finite state and action spaces, $\mathcal P(s'|s,a)$ is the transition kernel, $\gamma\in(0,1)$ is the discount factor, $\rho\in\Delta_{|\mathcal S|}$ is the initial-state distribution, $\mathbf r(s,a)=[r_1(s,a),r_2(s,a)]^\top$ is the bi-objective reward, and $\tau$ is a regularization functional. We instantiate a KL-regularized MDP with $\tau(\pi(\cdot| s)) = \beta \mathrm{KL}( \pi(\cdot| s) \| \pi_{\rm ref}(\cdot| s))$ applied to the policy $\pi$ with $\beta>0$ and some reference policy $\pi_{\rm ref}(\cdot|s)$. This yields a generally non-convex instance of~\eqref{eq:BOO}, with $\bu=\pi\in\U=\Delta_{|\mathcal A|}^{|\mathcal S|}$ and $h_m(\bu) = -\mathbb E_{\pi} \big[ \sum_{t=0}^{\infty}\gamma^t\bigl(r_m(s_t,a_t)-\tau(\pi(\cdot|s_t))\bigr) \big]$, covering LLM alignment settings with competing rewards~\citep{dai2023safe,shi2024decoding} and RLHF/DPO-style objectives~\citep{rafailov2023direct,dai2023safe,xiong2025projection}. 

\emph{This KL-regularized MORL problem is generally nonconvex in policy space but strongly convex in normalized occupancy measures.} Let
\begin{equation}
\bx=\mathrm{vec}(\mu_\pi) \in \Delta_{|\mathcal{S}||\mathcal{A}|},\quad \text{where}\quad 
\mu_\pi(s,a):=(1-\gamma)\,\mathbb{E}_{\rho,\pi}\big[\textstyle\sum_{t=0}^{\infty}\gamma^t\mathbf{1}\{s_t=s,a_t=a\}\big].
\label{eq:occupancy_measure}
\end{equation}
Denote by $E$ the per-state summation matrix with $E_{(s,a),s'}=\mathbf{1}\{s'=s\}$, and by $P$ the transition matrix with $P_{(s,a),s'}=\mathcal P(s'|s,a)$. 
Then the occupancy space is a convex feasible set~\citep{zhang2020variational,fatkhullin2025stochastic}:
\begin{equation}
\X =
\{\mathrm{vec}(\mu_\pi):\pi\in\Delta_{|\mathcal A|}^{|\mathcal S|}\} = \big\{ \bx\in\mathbb R_+^{|\mathcal S||\mathcal A|}: (E-\gamma P)^\top \bx=(1-\gamma)\rho \big\} .
\label{eq:occupancy_measure_space}
\end{equation}
Although illustrated on finite tabular MDPs, this viewpoint extends to continuous state or action spaces, where the finite vector $\bx$ is replaced by an infinite-dimensional density~\citep{bhatt1996occupation,altman2021constrained}.
Moreover, when $E^\top \bx \geq d_{\min}>0$, a standard sufficient exploration condition, the policy-to-occupancy map~\eqref{eq:occupancy_measure} is bijective~\citep{schlaginhaufen2024towards,zhang2020variational}. Hence, we can write $h_m(\bu)=f_m(g(\bu))$ for the bijection $g$. Writing $R_m=\mathrm{vec}(r_m(\mathcal S,\mathcal A))$ and $R_0=\mathrm{vec}(\log \pi_{\rm ref}(\mathcal A|\mathcal S))$ gives the strongly convex formulation~\citep{schlaginhaufen2024towards}
\begin{equation}
 f_m(\bx) =
 \beta\left(\langle \bx,\log \bx\rangle
 -\langle E^\top \bx,\log(E^\top \bx)\rangle -\langle R_0,\bx\rangle \right)-\langle R_m,\bx\rangle,\quad m \in \{1,2\}.
 \label{eq: MDP}
\end{equation}
For any $w\in[0,1]$, the LS problem associated with \eqref{eq: MDP} admits a unique interior solution $\bx_w^*\in\Delta_{|\mathcal{S}|\times |\mathcal{A}|}^\circ$~\citep{muller2026optimal} that is Pareto optimal, and $w\mapsto \bx_w^*$ is Lipschitz-continuous~\citep[Theorem 2.16]{ito2008lagrange}. Hence $\{\bx_w^*:w\in[0,1]\}\subset\operatorname{relint}(\mathcal X)$ is compact, so the objective in \eqref{eq: MDP} has uniformly bounded and Lipschitz-continuous Hessian on this set. Therefore, the MDP instance satisfies Assumption~\ref{assump:hidden_regular}. Moreover, since the LS solution is interior, only the equality constraint in \eqref{eq:occupancy_measure_space} is active. Its Jacobian is $J=(E-\gamma P)^\top$, which has full row rank as detailed in Appendix~\ref{app:rank_condition}, and hence Assumption~\ref{assump:LICQ} holds. Additionally, Assumption~\ref{assump:tangent_gap} reduces to $\|\Proj_{\ker((E-\gamma P)^\top)}(R_1-R_2)\|\ge c$, a mild condition that typically holds unless the two objectives coincide along feasible directions.
The norm of the objective-vector derivative along the optimal occupancy path gives $v(w)$, whose integral and normalization yield $\Farc$ for Rule~\ref{prin:pf_aware_ls_weighting}.
\begin{proposition}[PF speed in bi-objective RL]
\label{prop: speed for negative conditional entropy}
Under the sufficient exploration condition~\citep{schlaginhaufen2024towards}, the bi-objective RL in \eqref{eq: MDP} with constraint in \eqref{eq:occupancy_measure_space} yields
\begin{align*}
v(w)
= & \sqrt{(1-w)^2+w^2} \Big|(R_1-R_2)^\top \big(\widetilde H_w^{-1}-\widetilde H_w^{-1}J^\top (J\widetilde H_w^{-1}J^\top)^{-1}J\widetilde H_w^{-1} \big)(R_1-R_2) \Big|,
\end{align*}
where $\widetilde H_w =\beta\big(\Diag(\tfrac{1}{\mathbf{x}^*_w})-E\Diag(\tfrac{1}{E^\top\mathbf{x}^*_w})E^\top\big)+cJ^\top J$ is the constraint-augmented Hessian for all $c>0$.
When $|\mathcal S|=1$, this reduces to the entropic bandit case where the LS solution admits the closed form $\bu_w^*=\bx_w^*=\SoftMax\left(R_0+\beta^{-1}(wR_1+(1-w)R_2)\right)\in\Delta_A^\circ$, 
and the speed becomes
\begin{align}
v(w) = & \beta^{-1}\sqrt{(1-w)^2+w^2} (R_1-R_2)^\top\big(\Diag(\bu_w^*)-\bu_w^* (\bu_w^*)^\top\big)(R_1-R_2).
\label{eq:bandit_speed}
\end{align}
\end{proposition}
The proof of Proposition~\ref{prop: speed for negative conditional entropy} is deferred to Appendix~\ref{appendix: proof of speed for negative conditional entropy}. 
Proposition~\ref{prop: speed for negative conditional entropy} shows that $\Farc$ can be established once the solutions $\{\bx_w^*\}_{w\in[0,1]}$ are known, while constructing $\Farc$ is meant to identify a finite set of weights $\{w_n\}_{n=0}^N$ to save the budget for finding the corresponding solutions. 
This creates a \textit{coupled dependence} between geometry estimation on the PF and scalarized optimization over $\bx_w^*$.
Fortunately, this issue disappears when $|\mathcal S|=1$, where $\bx_w^*$ and $v(w)$ are available in closed form and $\Farc(w)$ is analytically available. Examples include the best-arm identification (BAI) problem in the bandit setting~\citep{bubeck2009pure,kaufmann2016complexity}, 
where each arm $a\in\mathcal A$ has an unknown mean reward vector $[R_1(a),R_2(a)]^\top$. While BAI for scalarized rewards is well studied~\citep{auer2002finite,yahyaa2015thompson,cheng2025multi}, PF coverage remains less explored~\citep{garivier2024sequential}. Proposition~\ref{prop: speed for negative conditional entropy} then yields the following instantiation of the plug-in PF-aware weighting Rule~\ref{prin:pf_aware_ls_weighting}.
\renewcommand{\theprinciple}{1 example} 
\begin{principle}
\label{prin:bandit}
 For the bandit problem, estimate reward vectors $\hat R_1,\hat R_2$ to construct $\hat v(\cdot;\hat R_1,\hat R_2)$ by replacing $(R_1,R_2)$ with $(\hat R_1,\hat R_2)$ in \eqref{eq:bandit_speed} and similarly 
\vspace{-0.1cm}
\begin{align}
 \hat s(w;\hat R_1,\hat R_2) = \int_0^w \hat v(u;\hat R_1,\hat R_2) du, \quad \hatFarc(w;\hat R_1,\hat R_2) = \frac{\hat s(w;\hat R_1,\hat R_2)}{\hat s(1;\hat R_1,\hat R_2)}.
\end{align}
Then, apply Rule~\ref{prin:pf_aware_ls_weighting} with $\Farc$ replaced by $\hatFarc$ when choosing $\{w_n\}_{n=0}^N$.
\end{principle}
\addtocounter{principle}{-1} 
\renewcommand{\theprinciple}{\arabic{principle}}

\begin{wrapfigure}{r}{0.348\textwidth}
 \centering
 \vspace{-0.4cm}
 \includegraphics[width=0.315\textwidth]{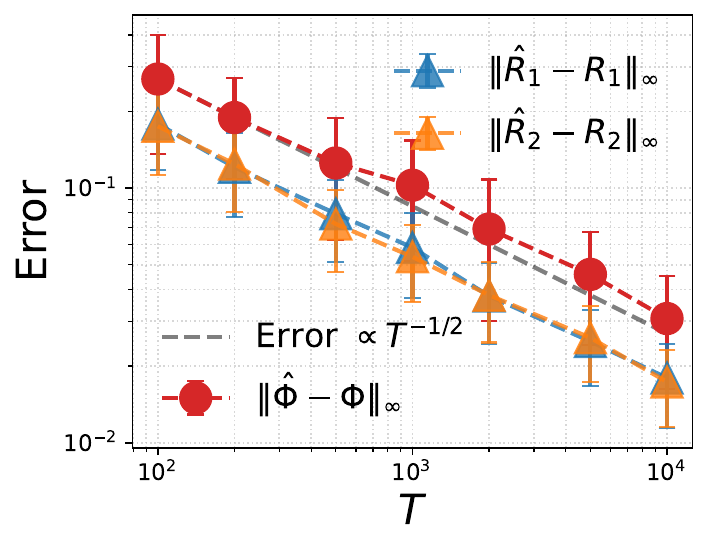}
 \vspace{-0.5cm}
 \caption{Error vs. number of pulls $T$ in bandit. Error bars show standard deviations across $100$ trials.
 }
\label{fig:error_conv}
\vspace{-0.3cm}
\end{wrapfigure}
This analytical construction makes the PF geometry explicit and provides a tractable testbed for the geometry behind SURF, as visualized in Figure~\ref{fig:density_mismatch}. Since the reward vector $(R_1, R_2)$ is unknown in practice, we estimate it from samples and plug the estimate into the closed-form $v$. The resulting $\hat\Phi$ is a consistent estimator of $\Phi$. The estimation error of $\hatFarc$ is driven solely by reward-estimation error and is therefore controlled by standard statistical bounds~\citep{auer2002finite,drugan2014scalarization,xu2023pareto}. Specifically, $\|\hatFarc-\Farc\|_\infty$ decays at rate $\Oc(T^{-1/2})$, where $T$ is the number of pulls used to estimate the reward vectors. This result is formalized in Appendix~\ref{appendix:bandit_R_and_Farc_err} and validated empirically in Figure~\ref{fig:error_conv}. This example also suggests a natural extension to scalarization-based MO-bandit frameworks~\citep{drugan2013designing,yahyaa2015thompson,cheng2025multi}: rather than drawing weights $w$ uniformly, one can sample $w$ according to Rule~\ref{prin:pf_aware_ls_weighting} with the estimated $\hatFarc$ in place of $\Farc$.

\section{SURF: Algorithm for Sampling Uniformly along the Pareto Front}
\label{sec:main_result}

We use LS throughout this section, with SURF refining coverage along the PF path induced by LS.
The previous section shows that Rule~\ref{prin:pf_aware_ls_weighting} achieves uniform PF coverage when $\Farc$ or its estimate (cf. Rule~\ref{prin:bandit}) is available. In general, however, $\Farc$ is not explicitly known. To apply Rule~\ref{prin:surf_general} in this setting, we propose SURF, an iterative refinement algorithm for \textbf{S}ampling \textbf{U}niformly along the Pa\textbf{r}eto \textbf{F}ront (SURF). 
At a high level, SURF refines the estimate of $\Farc$ while improving the scalarized solutions associated with the target weights $w_n=\Farc^{-1}(n/N)$.
Specifically, at outer iteration $t$, SURF uses the current CDF estimate $\Farc_t$ to allocate weights $w_n^{(t)} = \Farc_t^{-1}\left(\frac{n}{N}\right)$ with initialization $\Farc_0(w)=w$ to start with the standard uniform weighting baseline. If $\Farc_t=\Farc$, then $w_n^{(t)}$ would produce uniform PF coverage. For each sampled weight $w_n^{(t)}$, SURF then approximately solves the corresponding scalarized subproblem
\begin{align}
\bu_n^{(t)}
\approx
\arg\min_{\bu\in\U}\LS_h(\bu;w_n^{(t)})
\quad
\textcolor{gray}{\text{or equivalently }
\bx_n^{(t)}
\approx
\arg\min_{\bx\in\X}\LS_f(\bx;w_n^{(t)})}. \label{eq:inner_objective}
\end{align}
SURF is compatible with off-the-shelf solvers, e.g., Adam for loss minimization~\citep{kingma2014adam} and PPO for RL~\citep{schulman2017proximal}. The \textsc{InnerSolver} runs for $K$ steps, warm-started from the previous iterate $\bu_n^{(t-1)}$. The resulting PF estimates $\{\h(\bu_n^{(t)})\}_{n=0}^N$ \textcolor{gray}{or equivalently $\{\f(\bx_n^{(t)})\}_{n=0}^N$} are then used to build estimates for $s$ at the sampled weights $\{w_n^{(t)}\}_{n=0}^N$ via cumulative chord length
\begin{align}
\tilde s^{(t)}(w_0^{(t)})&:=0,~
\tilde s^{(t)}(w_{n+1}^{(t)}):=
\tilde s^{(t)}(w_n^{(t)})
+
\big\|
\h(\bu_{n+1}^{(t)})-\h(\bu_n^{(t)})
\big\|,
\quad n=0,\cdots,N-1.
\label{eq:sw_tilde}
\end{align}
Then, SURF interpolates these values to construct $\tilde s^{(t)}$ on $[0,1]$, for example, Piecewise Cubic Hermite Interpolating Polynomial (PCHIP)~\citep{fritsch1980monotone}. Compared with naive piecewise-linear interpolation, PCHIP remains simple to implement, is available in standard numerical libraries, preserves monotonicity, and provides more favorable Lipschitz-smoothness properties, as discussed in the next section. These features are useful both in practice and in the convergence analysis.
In this way, an empirical estimate of $\Farc$ is established via $\tilde \Farc_t(w):=\frac{\tilde s^{(t)}(w)}{\tilde s^{(t)}(1)}$ and 
SURF then updates the CDF estimate via
\begin{equation}
\Farc_{t+1}(w)=\alpha \tilde \Farc_t(w)+(1-\alpha)\Farc_t(w),
\label{eq:surf_update}
\end{equation}
where a damping factor $\alpha$ is especially useful for inexact \textsc{InnerSolver} to stabilize the update.

\begin{algorithm}[t]
\caption{SURF: \textbf{S}ampling \textbf{U}niformly along the Pa\textbf{r}eto \textbf{F}ront%
}
\label{alg:SURF}
\begin{algorithmic}[1]
\Require Number of segments $N$, initial CDF $\Farc_0(w)=w$, target quantiles $\{n/N\}_{n=0}^N$, damping factor $\alpha$,
a finite $K$-step inner solver \textsc{InnerSolver} for $\min_{\bu\in \U} \LS_h(\bu;w)$ with initial iterates $\{\bu_n^{(-1)}\}_{n=0}^N$
\textcolor{gray}{or for $\min_{\bx\in \X} \LS_f(\bx;w)$ with initial iterates $\{\bx_n^{(-1)}\}_{n=0}^N$}.
\For{$t = 0,1,2,\cdots, T$}
 \State Find $w_n^{(t)} = \Farc_t^{-1}(\frac{n}{N})$, for all $n=0,1,\cdots,N$ \label{algline: outerloop begin}
 \For{$n=0,1,\cdots,N$} \label{algline: innerloop begin}
 \State Update $\bu_n^{(t)} = \textsc{InnerSolver}\left(w_n^{(t)}; \bu_{\text{init}}=\bu_n^{(t-1)}\right)$ \textcolor{gray}{ or $\bx_n^{(t)}$ accordingly.}
 \State Save $\mathbf{h}(\bu_n^{(t)})$ \textcolor{gray}{ or $\mathbf{f}(\bx_n^{(t)})$ equivalently}. 
 \EndFor \label{algline: innerloop end}

 \State Build $\tilde{s}^{(t)}(w_n^{(t)})$ via \eqref{eq:sw_tilde}
 and $\tilde s^{(t)}:[0,1]\rightarrow \mathbb{R}_+$ via interpolation (e.g., Algorithm~\ref{alg:PCHIP}). 
 \State Construct $\tilde \Farc_{t}(w)=\frac{\tilde s^{(t)}(w)}{\tilde s^{(t)}(1)}$ and update $\Farc_{t+1}(w) = \alpha \tilde{\Farc}_{t}(w) + (1-\alpha)\Farc_t(w),\forall w\in [0,1]$.
 \label{algline: outerloop end}
\EndFor
\Ensure Refined CDF $\Farc_{T+1}$ and PF samples $\{(\bu_n^{(T)},\h(\bu_n^{(T)}))\}_{n=0}^N$.
\end{algorithmic}
\end{algorithm}

For a priori preferences among objectives, the uniform positions $\{\frac{n}{N}\}_{n=0}^N$ can be replaced by desired trade-off positions $\{q_n\}_{n=0}^N\in[0,1]^{N+1}$, yielding $w_n^{(t)}=\Farc_t^{-1}(q_n)$ for each $n=0,\cdots,N$.

\paragraph{Convergence Analysis.} 
We analyze Algorithm~\ref{alg:SURF} by separating two sources of error: the \emph{discretization error} from reconstructing $s$ using only $(N+1)$ PF samples, and the \emph{optimization error} from solving the scalarized subproblems \eqref{eq:inner_objective} only approximately.
We start with the exact-\textsc{InnerSolver} case that isolates the discretization effect: even when the scalarized solutions are exact, $\bu_n^{(t)}=\bu^*_{w_n^{(t)}}$ \textcolor{gray}{or equivalently $\bx_n^{(t)}=\bx^*_{w_n^{(t)}}$}, estimating the PF geometry from only $(N+1)$ sampled points introduces an inevitable finite-sampling floor. 
We then consider the optimization error induced by running \textsc{InnerSolver} for only $K$ steps. To quantify this perturbation, we characterize the convergence rate of \textsc{InnerSolver} via $\eta_K$, a $K$-step contraction factor such that, for any weight $w\in[0,1]$ and initialization $\tilde\bx_0$, its output $\tilde\bx_K$ satisfies $\|\tilde\bx_K-\bx_w^*\|\le \eta_K\|\tilde\bx_0-\bx_w^*\|$, where $\eta_K\downarrow0$ as $K\to\infty$. 
For linearly convergent \textsc{InnerSolver}, e.g., $K$-step projected gradient descent~\citep{fatkhullin2025stochastic}, one may take $\eta_K=\bigl(1-{\mu}/{l_{f,1}}\bigr)^K$, where $l_{f,1}$ is valid along the entire projected gradient descent trajectory; see Appendix~\ref{app:PGD_trajectory}.
For sublinearly convergent methods, hidden strong convexity converts objective-gap or stationarity guarantees into hidden-space distance bounds~\citep{fatkhullin2025stochastic}, so for sufficiently large $K$ the same characterization can be instantiated with $\eta_K=\Oc(K^{-p})$ for some $p>0$.
The following theorem shows the convergence of SURF to the finite-sampling floor, with an additional coupled term accounting for inner optimization error.
Let $C(N)$ be a scalar depending on $N$.
We write $\|\Farc_t-\Farc\|_\infty=\sup_{w\in[0,1]}|\Farc_t(w)-\Farc(w)|$ and define the convergence measure as
\begin{equation}
\label{eq:surf_error_measure}
\mathcal E_t \big(\Phi_t, \{w_n^{(t)}\}, \{\mathbf{x}_n^{(t)}\} \big) := 
\underbrace{\|\Farc_t-\Farc\|_\infty}_{\text{CDF estimation error}}
+C(N)\underbrace{\max_{0 \leq n\leq N}\|\bx_n^{(t)}-\bx_{w_n^{(t)}}^*\|}_{\text{inner optimization error}}.     
\end{equation}
\vspace{-0.5cm}
\begin{leftrulebox}
\begin{theorem}[Convergence of SURF]
\label{thm:cdf}
Suppose Assumptions~\ref{assump:hidden_regular},~\ref{assump:LICQ},~\ref{assump:tangent_gap} hold. Then, Algorithm~\ref{alg:SURF} generates $\{\Farc_t\}$ and $\{\tilde\Farc_t\}$ that are uniformly $l_{\Farc_{\rm iter},1}$-Lipschitz-smooth for all $t\geq 0$. Let $\kappa := \frac{4\sqrt{2}v_{\max}}{v_{\min}}\sqrt{l_{\Farc_{\rm iter},1}+l_{\Farc,1}}$, and $N=\Omega(\kappa^{1.5})$.
For general \textsc{InnerSolver} with budget specified in Remark~\ref{remark:Kchoice}, SURF with $\alpha \leq 0.5$ admits   $C(N)=\Theta(N)$ such that, for all $t\ge0$,
\begin{align}\label{eq:inexact_SURFconvergence}
& \mathcal E_t\big(\Phi_t, \{w_n^{(t)}\}, \{\mathbf{x}_n^{(t)}\}\big) 
\leq  \left(1-\frac{\alpha}{4}\right)^t 
\mathcal E_0\big(\Phi_0, \{w_n^{(0)}\}, \{\mathbf{x}_n^{(0)}\}\big)
+\underbrace{\Oc(N^{-2})}_{\text{discretization floor}}. 
\end{align}
\end{theorem}
\end{leftrulebox}
\begin{remark}[Inner-solver budget]
\label{remark:Kchoice}
If \textsc{InnerSolver} returns exact solutions, then the inner optimization error times $C(N)$ is zero.
Otherwise, consider an \textsc{InnerSolver} with $K$-step contraction factor $\eta_K$. We choose $K$ so that $\eta_K=\Oc\big({s(1)}/{N}\big)$, that is $K=\Omega\big(\log ({Nv_{\max}}{v_{\min}^{-2}}) \big)$ for a linearly convergent solver and $K=\Omega\big(\big( {N}/{s(1)}\big)^{1/p}\big)$ for a sublinearly convergent solver with $\eta_K=\Oc(K^{-p})$ for some $p>0$.
Then the inner-solver perturbs the CDF reconstruction by a sum of $N$ chord-length errors, so $C(N) = \Theta(N)$ is the natural amplification factor.
Alternatively, one may use a small outer stepsize $\alpha=\Oc(N^{-1})$, in which case $\eta_K=\Oc(1)$ is enough after a warmup that ensures $\max_n\|\bx_n^{(0)}-\bx_{w_n^{(0)}}^*\|=\Oc\left({v_{\min}^2}v_{\max}^{-1}N^{-1}\right)$, as detailed in Theorem~\ref{thm:warm_started_two_regimes} in Appendix~\ref{appendix:cdf_inexact}.
\end{remark}
The proof bounds CDF reconstruction through chord and interpolation errors, then couples CDF-driven grid motion with warm-started tracking of the scalarized optima to contract the combined error to a finite-sampling floor. The geometric factor $\kappa$ combines traversal-speed variation and CDF smoothness.
The proof of Theorem~\ref{thm:cdf} and the constants in $\Oc(\cdot)$ and $\Omega(\cdot)$ are made explicit in Appendix~\ref{app:cdf_refinement}. 
Theorem~\ref{thm:cdf} guarantees that SURF drives the CDF estimation error to an $\Oc(N^{-2})$ finite-sampling discretization floor, matching the standard second-order accuracy of polyline-based reconstruction of smooth curves~\citep{suli2003introduction}. In contrast, even with an exact \textsc{InnerSolver}, the first-step reconstruction $\tilde \Farc_0$ can carry an $\Oc(\kappa^3N^{-2})$ error, where the $\kappa$ factor comes from uneven clustering on the PF, as theoretically discussed in Lemma~\ref{lem:s_tilde_combined_bound} in Appendix~\ref{appendix:cdf_exact} and empirically verified in Appendix~\ref{app:toy}. In Appendix~\ref{app:SURV_CV}, we further relate CDF reconstruction error to a PF uniformity metric measuring the relative deviation of sampled PF segments from equal spacing, \textit{Coefficient of Variation} (CV)~\citep{everitt2002cambridge}, as
\vspace{-0.1cm}
\begin{leftrulebox}
\vspace{-0.1cm}
\begin{equation}\label{eq:cv_e_t}
\text{CV}:=\tfrac{\mathrm{std}\big(\{\|\h(\bu_{n+1}^{(t)})-\h(\bu_{n}^{(t)})\|\}_{n=0}^{N-1} \big)}{\mathrm{mean}\big(\{\|\h(\bu_{n+1}^{(t)})-\h(\bu_{n}^{(t)})\|\}_{n=0}^{N-1} \big)}
=\Oc\big(N^{-1}+N\|\Farc_t-\Farc\|_\infty\big)
=\Oc(N^{-1}),
\end{equation}
\end{leftrulebox}
at sufficiently large outer iteration $t=\Omega(\log N)$. Equation~\eqref{eq:cv_e_t} links CDF refinement accuracy to PF spacing uniformity, showing that SURF improves uniform PF coverage.

Remark~\ref{remark:Kchoice} illustrates that SURF does not require each scalarized subproblem to be solved to high precision at every outer iteration. To achieve $\|\Farc_T-\Farc\|_\infty=\Oc(N^{-2})$ with a linearly convergent inner solver~\cite{fatkhullin2025stochastic}, SURF requires $T=\Oc(\log N)$ outer iterations and per-iteration inner complexity $\Oc(N\log N)$, giving total complexity $\Oc(N(\log N)^2)$. At the same time, SURF yields $\max_n\|\bx_n^{(T)}-\bx_{w_n^{(T)}}^*\|=\Oc(N^{-3})$. In contrast, uniform weighting requires $\Oc(N\log(N^3))=\Oc(N\log N)$ total complexity for $N$ to reach the same inner accuracy, while providing no uniform PF-coverage guarantee. Alternatively, the small-stepsize regime $\alpha=\Oc(N^{-1})$ in Remark~\ref{remark:Kchoice} relaxes the required inner-solver contraction, making it especially suitable for sublinearly convergent \textsc{InnerSolver}.
In practice, with a moderate choice of $N=5$ and $\alpha=0.3$ in the LLM alignment task in Section~\ref{sec:exp}, SURF achieves efficiency comparable to the uniform scalarization baseline under the same training budget, while significantly improving PF uniformity.

\section{Experiments}
\label{sec:exp}
In this section, we empirically verify the effectiveness of SURF against \textbf{baselines} including \textit{uniform weighting} with $w_n=\frac{n}{N}$, \textit{Optimistic Linear Support (OLS)}~\citep{barrett2008learning} which aims to identify supported PF solutions, \textit{Uniform Multi-Objective optimization based on Decomposition (UMOD)}~\citep{zhang2024gliding}, which targets uniform PF coverage, and inference-time weighted model merging named \textit{reward soup (Soup)}~\citep{rame2023rewarded}. For budget fairness, all optimization-based methods are given the same total number of inner iterations.
We use \textbf{evaluation metrics} including \textit{Hypervolume} (HV)~\citep{zitzler1999multiobjective} and \textit{Inverted Generational Distance} (IGD)~\citep{bogoya2019averaged} for PF approximation quality, and
\textit{Gap Ratio}~$=\frac{\max_n \|\h(\bu_{n+1}^{(t)})-\h(\bu_{n}^{(t)})\|}{\min_n \|\h(\bu_{n+1}^{(t)})-\h(\bu_{n}^{(t)})\|}$~\citep{roijers2013survey} and CV~\citep{everitt2002cambridge} for spacing uniformity. Spacing is assessed alongside approximation quality, since evenly spaced points can lie off the PF, as Soup illustrates. Reference points, reference-front construction, and objective scales are specified in Appendix~\ref{app:metric_settings}.
We conduct experiments on MO Gymnasium benchmarks~\citep{felten_toolkit_2023} in Section~\ref{sec:mogym} and an LLM alignment task for Reddit summarization~\citep{volske2017tl} in Section~\ref{sec:llm}. 
Additional details for the experiments  and toy examples are provided in Appendix~\ref{app:exp_details}. 
Beyond LS, Appendix~\ref{app:nonconvex_tchebycheff} shows that SURF improves spacing on the nonconvex ZDT2 and quarter-circle fronts using weighted Tchebycheff scalarization.

\subsection{MO-Gymnasium Problems}
\label{sec:mogym}

\begin{table*}[t]
\centering
\renewcommand{\arraystretch}{1.15}
\resizebox{0.98\textwidth}{!}{%
\begin{tabular}{ll|cC{2.329cm}C{2.329cm}C{2.329cm}G{2.329cm}}
\toprule
\textbf{Tasks} & \textbf{Metric} & \textbf{Soup} & \textbf{Uniform-$w$} & \textbf{UMOD} & \textbf{OLS} & \textbf{Ours} \\
\midrule

\multirow{4}{*}{DST}
& HV ($\times 10^{-3}$) $\uparrow$
& 3.10 $\pm$ 0
& 1.83 $\pm$ 0
& 3.85 $\pm$ 0.00
& 3.69 $\pm$ 0
& \textbf{3.87 $\pm$ 0} \\

& IGD ($\times 10^{-2}$) $\downarrow$
& 0.68 $\pm$ 0
& 4.51 $\pm$ 0
& 0.23 $\pm$ 0.00
& 0.41 $\pm$ 0
& \textbf{0.21 $\pm$ 0} \\

& CV ($\times 10^{-1}$) $\downarrow$
& \textit{0.09 $\pm$ 0}
& 10.91 $\pm$ 0
& 3.17 $\pm$ 0.05
& 9.87 $\pm$ 0
& \textbf{0.81 $\pm$ 0} \\

& Gap Ratio $\downarrow$
& \textit{1.04 $\pm$ 0}
& 17.16 $\pm$ 0
& 2.41 $\pm$ 0.13
& 26.85 $\pm$ 0
& \textbf{1.45 $\pm$ 0} \\
\midrule

\multirow{4}{*}{Fishwood}
& HV ($\times 10^{-1}$) $\uparrow$ 
& 2.39 $\pm$ 0
& \textbf{2.40 $\pm$ 0}
& 2.39 $\pm$ 0.00
& 2.39 $\pm$ 0
& 2.39 $\pm$ 0 \\

& IGD ($\times 10^{-2}$) $\downarrow$
& 0.91 $\pm$ 0
& 0.91 $\pm$ 0
& 0.92 $\pm$ 0.00
& 0.96 $\pm$ 0
& \textbf{0.90 $\pm$ 0} \\

& CV ($\times 10^{-1}$) $\downarrow$
& 2.00 $\pm$ 0
& 1.94 $\pm$ 0
& 2.33 $\pm$ 0.17
& 3.57 $\pm$ 0
& \textbf{0.26 $\pm$ 0} \\

& Gap Ratio $\downarrow$
& 1.90 $\pm$ 0
& 1.75 $\pm$ 0
& 2.14 $\pm$ 0.14
& 2.49 $\pm$ 0
& \textbf{1.12 $\pm$ 0} \\
\midrule

\multirow{4}{*}{MO-Mountaincar}

& HV ($\times 10^{3}$) $\uparrow$

& 0.33 $\pm$ 0.20

& 0.89 $\pm$ 0.27

& 0.99 $\pm$ 0.17

& 0.35 $\pm$ 0.30

& \textbf{1.11 $\pm$ 0.23} \\

& IGD $\downarrow$

& 9.45 $\pm$ 1.02

& 6.41 $\pm$ 0.88

& 17.28 $\pm$ 10.56

& 23.77 $\pm$ 6.63

& \textbf{6.00 $\pm$ 0.72} \\

& CV ($\times 10^{-1}$) $\downarrow$

& 12.32 $\pm$ 5.72

& 7.25 $\pm$ 1.33

& 10.06 $\pm$ 1.57

& 6.54 $\pm$ 3.61

& \textbf{4.41 $\pm$ 2.03} \\

& Gap Ratio $\downarrow$

& 489.94 $\pm$ 401.29

& 74.96 $\pm$ 97.16

& 10.28 $\pm$ 3.89

& 9.47 $\pm$ 10.66

& \textbf{3.75 $\pm$ 3.75} \\

\bottomrule
\end{tabular}%
}
\caption{MO-Gymnasium~\citep{felten_toolkit_2023} results comparing PF approximation quality and uniform-coverage metrics. We use $N=15$ for DST, $N=10$ for Fishwood, and $N=5$ for MO-Mountaincar. Results are reported as mean$\pm$standard deviation over 8 realizations.
Zero standard deviation represents deterministic experiments without randomness.}
\label{tab:main_comparison}
\vspace{-0.3cm}
\end{table*}

Table~\ref{tab:main_comparison} summarizes the results on three MO-Gymnasium~\citep{felten_toolkit_2023} tasks. 
Appendix~\ref{app:modgym} provides details on the environment, condition checks of occupancy-measure regularity, and implementation details.

\begin{wrapfigure}{r}{0.34\textwidth}
 \centering
 \vspace{-0.3cm}
 \includegraphics[width=0.335\textwidth]{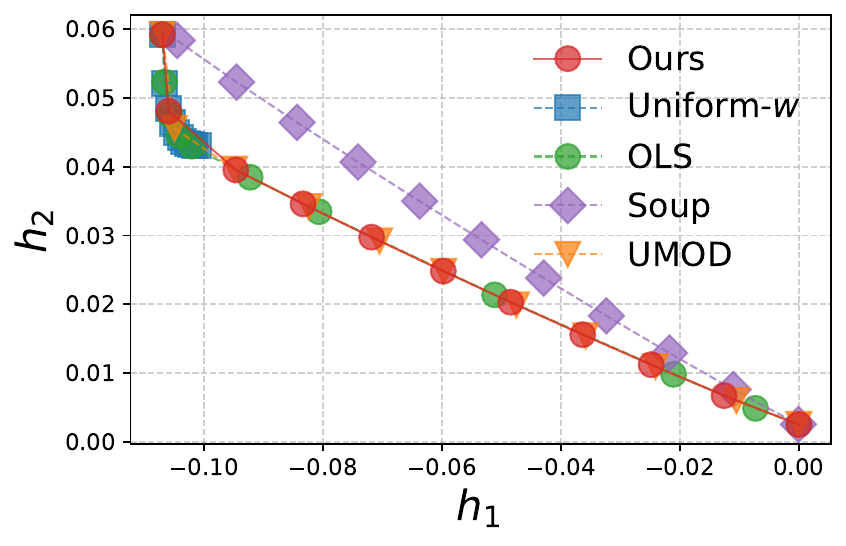}
 \vspace{-0.7cm}
 \caption{PFs obtained by SURF and baselines on DST.}
 \label{fig:dst_main_result}
\end{wrapfigure}

\textbf{DST.}
DST~\citep{vamplew2011empirical,yahyaa2015thompson} is a classical bi-objective RL benchmark, where the agent trades off treasure value and travel time. 
We evaluate SURF Algorithm~\ref{alg:SURF} with $30$ outer iterations and $\alpha=0.3$. 
Figure~\ref{fig:dst_main_result} shows that SURF produces more uniform PF coverage than the optimization-based baselines. 
Compared with uniform weighting, SURF improves HV by $2.11\times$ and reduces IGD, CV, and Gap Ratio by $21.5\times$, $13.5\times$, and $11.8\times$, respectively. 
SURF also outperforms UMOD, another baseline explicitly designed for uniform PF coverage, with slightly higher HV, lower IGD, and substantially better spacing, reducing the Gap Ratio from $2.41$ to $1.45$. 
Although the Soup method yields uniform spacing,  the solutions are suboptimal and not on the PF, as indicated by the visualization in Figure~\ref{fig:dst_main_result}.
In Appendix~\ref{app:dst_setup}, we additionally visualize the DST policies learned by different methods through their induced trajectories, and provide ablations on $N$, $\alpha$, and $K$, together with convergence-time analysis showing that UMOD is slower than SURF.

\textbf{Fishwood.} The Fishwood task trades off fishing and wood-collection rewards~\citep{roijers2018multi}. We apply SURF with $\alpha=0.3$ for $15$ outer iterations. Additional setup details are provided in Appendix~\ref{app:fishwood_setup}.
As shown in Table~\ref{tab:main_comparison}, all methods achieve nearly identical HV and IGD, indicating similar PF approximation quality. 
However, SURF substantially improves spacing uniformity, reducing CV from the best baseline value $0.194$ to $0.026$ and Gap Ratio from $1.75$ to $1.12$, corresponding to $7.5\times$ and $1.6\times$ reductions, respectively.

\textbf{MO-Mountaincar.}
The MO-Mountaincar problem is a multi-objective variant of the classical Mountaincar control task, where an underpowered car must build momentum by moving back and forth in a valley to reach the goal near the top of the right hill~\citep{Moore90efficient}. In this setting, the agent faces a trade-off between fast completion and action efficiency, induced by a per-step time penalty and an action penalty. Additional details are provided in Appendix~\ref{app:mountaincar}.
As shown in Table~\ref{tab:main_comparison}, SURF achieves the best overall performance, improving HV from the best baseline value $0.99$ to $1.11$ and reducing IGD from $6.41$ to $6.00$.
It also improves spacing uniformity, reducing CV from the best baseline value $0.654$ to $0.441$ and Gap Ratio from $9.47$ to $3.75$.
These results indicate that SURF improves both PF approximation quality and uniform coverage in this task.

\subsection{LLM Alignment on the Reddit Summarization Task} 
\label{sec:llm}

To test SURF beyond standard MORL benchmarks, we apply it to a stochastic LLM alignment problem on Reddit summarization~\citep{volske2017tl,stiennon2020learning}. 
We fine-tune \texttt{Qwen/Qwen2.5-0.5B-Instruct}~\citep{qwen2025qwen25technicalreport} with PPO~\citep{schulman2017proximal}, while keeping the pretrained backbone frozen and updating only parameter-efficient finetuning (PEFT) adapter parameters via LoRA~\citep{hu2022lora}. 
The two objectives are induced by a summarization-quality reward model and a factual-faithfulness reward model, together with a fixed KL penalty to the reference policy with $\beta = 0.05$. 
Each scalarized PPO job uses one 40GB A100 GPU, 4 CPU cores, and 16GB host memory, with both SURF and Uniform-$w$ receiving 12 PPO epochs per slot after a shared 10-epoch warm start.
Full training details are deferred to Appendix~\ref{app:llm_ft}.

\begin{wraptable}{r}{0.538\textwidth}
\vspace{-0.525cm}
\centering
\renewcommand{\arraystretch}{1.15}
\resizebox{0.52\textwidth}{!}{%
\begin{tabular}{l|C{2.1cm}C{2.1cm}G{2.1cm}}
\toprule
\textbf{Metric} 
& \textbf{Soup} 
& \textbf{Uniform-$w$} 
& \textbf{Ours} \\
\midrule

HV $\uparrow$
& \begin{tabular}{@{}c@{}}0.75 $\pm$ 0.30\\ (0.95 $\pm$ 0.30)\end{tabular}
& \begin{tabular}{@{}c@{}}3.76 $\pm$ 0.31\\ (4.00 $\pm$ 0.30)\end{tabular}
& \begin{tabular}{@{}c@{}}\textbf{4.03 $\pm$ 0.35}\\ (\textbf{4.24 $\pm$ 0.35})\end{tabular}
\\

CV $\downarrow$
& \begin{tabular}{@{}c@{}}1.69 $\pm$ 0.26\\ (1.69 $\pm$ 0.26)\end{tabular}
& \begin{tabular}{@{}c@{}}0.65 $\pm$ 0.12\\ (0.65 $\pm$ 0.12)\end{tabular}
& \begin{tabular}{@{}c@{}}\textbf{0.30 $\pm$ 0.17}\\ (\textbf{0.29 $\pm$ 0.17})\end{tabular}
\\

Gap Ratio $\downarrow$
& \begin{tabular}{@{}c@{}}161.20 $\pm$ 15.27\\ (159.83 $\pm$ 15.16)\end{tabular}
& \begin{tabular}{@{}c@{}}5.89 $\pm$ 1.15\\ (6.01 $\pm$ 1.21)\end{tabular}
& \begin{tabular}{@{}c@{}}\textbf{2.93 $\pm$ 5.35}\\ (\textbf{2.89 $\pm$ 5.19})\end{tabular}
\\

\bottomrule
\end{tabular}%
}
\caption{
LLM alignment results on Reddit summarization. 
Main entries are evaluated using KL-regularized objectives, while values in parentheses are evaluated using reward-only objectives without the KL term. 
All metrics are evaluated over $10$ batches with batch size $64$, reported as mean $\pm$ standard deviation.
}
\label{tab:llm_alignment}
\vspace{-0.35cm}
\end{wraptable}

Table~\ref{tab:llm_alignment} shows that SURF improves both performance and PF coverage over Uniform-$w$ under the same per-slot PPO budget. 
Compared with Uniform-$w$, SURF achieves $1.07\times$ higher HV, $2.2\times$ lower CV, and $2.0\times$ lower Gap Ratio. 
The relatively large standard deviation of Gap Ratio reflects the stochasticity of LLM evaluation: each estimate is computed on batches of only $64$ prompts, so a small number of unusually short or long objective-space segments can substantially affect the max/min spacing ratio. 
Soup is less effective in our setting, suggesting that endpoint-only interpolation between PEFT adapters may not reliably yield well-spaced PF coverage for these objectives, especially with a smaller backbone and limited training budget.
We omit UMOD and OLS because they require substantially heavier computation in this LLM setting, and IGD because constructing a sufficiently dense reference PF is beyond our available computational budget.

\section{Conclusion}
\label{sec:conclusion}
This paper introduced SURF, a scalarization wrapper for improving uniform coverage along the reachable Pareto front while retaining the chosen inner solver.
One of our key contributions is a geometric view of scalarization, where scalarization not only produces individual Pareto points but also induces a traversal of the PF with generally non-uniform speed. 
By characterizing this weight-to-PF map through the arc-length CDF $\Farc$, we obtain a principled correction from uniform weights to uniform PF coverage. 
For structured bandit settings, we derive explicit forms of this geometry, and for general problems, we propose an iterative CDF-refinement algorithm compatible with standard \textsc{InnerSolver}. We establish convergence to a discretization floor for bi-objective LS under the stated assumptions and inner-solver conditions.
Empirically, SURF improves PF coverage in the tested optimization, reinforcement learning, and LLM alignment settings,
suggesting that arc-length-aware scalarization is a practical tool for constructing more uniformly distributed Pareto front approximations.
Future directions include extending geometry-aware sampling to higher-dimensional Pareto fronts and larger-scale LLM alignment tasks.

{\small \bibliography{main}}

\newpage

\appendix
\startcontents[appendix]

\section*{Supplementary Material: Contents}
\label{app:contents}

\printcontents[appendix]{}{1}{%
 \setcounter{tocdepth}{3}%
}


\section{Additional Related Work}
\label{app:related_work}
\paragraph{MORL and LLM Alignment.}
Multi-objective reinforcement learning provides a natural application domain because one often seeks not a single policy, but a family of policies spanning different trade-offs among objectives or diverse preferences of users~\citep{reymond2022pareto,felten2024multi,lin2024policy}. In language-model alignment, the same need arises when balancing criteria such as helpfulness, faithfulness, harmlessness, or style control~\citep{shi2024decoding,yang2024metaaligner,guo2024controllable,shen2025simultaneous}. Recent work on multi-objective reinforcement learning from AI feedback and multi-objective decoding highlights both the importance of this setting and the practical appeal of scalarization-based~\citep{qiu2024traversing,zhou2024beyond,williams2024multi,wang2024conditional,li2025gradient} or interpolation-based pipelines~\citep{rame2023rewarded,shi2024decoding,li2025multi,lin2025parm}. Our goal is complementary: rather than replacing scalarization, we characterize and correct its geometric sampling bias so that repeated scalarized solves yield a uniformly representative set of Pareto solutions.

\begin{table}[t]
\centering
\begin{tabular*}{0.98\textwidth}{@{\extracolsep{\fill}} l p{0.68\textwidth}}
\toprule
Symbol & Description \\
\midrule
$\bu \in \U$
& Decision variable and feasible set in the original space \\

$\h(\bu) = [h_1(\bu),\cdots,h_M(\bu)]^\top$
& Objective vector in the original space \\

$\min_{\bu \in \U} \h(\bu)$
& Original multi-objective problem \\

$\U_{\mathrm{PF}}$
& Pareto set in the original space \\

$\bx = g(\bu) \in \X = g(\U)$
& Hidden-space variable and feasible space induced by $g$ \\

$\mathbf{f}(\bx) = [f_1(\bx),\cdots,f_M(\bx)]^\top$
& Objectives in the hidden space, with $\f(\bx)=\h(g^{-1}(\bx))$ \\

$\min_{\bx \in \X} \mathbf{f}(\bx)$
& Hidden-space multi-objective problem \\

$\X_{\mathrm{PF}}=g(\U_{\mathrm{PF}})$
& Pareto set in the hidden space \\

$\mathcal{F}_{\mathrm{PF}}=\h(\U_{\rm PF})=f(\X_{\rm PF})$
& Pareto front \\
\midrule

$\mathbf{w} \in \Delta_M$
& LS weight vector \\

$\LS_h(\bu;\bw)=\sum_{m=1}^M w_m h_m(\bu)$
& LS objective in the original space \\

$\bu_{\bw}^* \in \arg\min_{\bu\in\U}\LS_h(\bu;\bw)$
& LS optimizer in the original space \\

$\LS_f(\bx;\bw)=\sum_{m=1}^M w_m f_m(\bx)$
& LS objective in the hidden space \\

$\bx_{\bw}^* \in \arg\min_{\bx\in\X}\LS_f(\bx;\bw)$
& LS optimizer in the hidden space \\

$\fPF(\bw)=\f(\bx_\bw^*)=\h(\bu_\bw^*)$
& Multi-objective PF map induced by LS \\

\midrule
$w \in [0,1]$
& Scalar LS weight for $M=2$, with $\bw=[w,1-w]$ \\

$\LS_h(\bu;w),\bu_w^*,\LS_f(\bx;w),\bx_w^*, \fPF(w)$
& Shorthand for the $M=2$ case with $\bw=[w,1-w]$ \\

$v(w)=\left\|\frac{\partial}{\partial w}\fPF(w)\right\|$
& Bi-objective PF traversal speed \\

$s(w)=\int_0^w v(p) dp$
& Bi-objective PF arc length up to weight $w$ \\

$\Farc(w)=\frac{s(w)}{s(1)}$
& Normalized arc-length CDF \\
\bottomrule
\end{tabular*}
\caption{Summary of notation. The first block lists the original-space MOO problem and hidden-space reformulation through the bijection $\bx=g(\bu)$, the second block defines LS weights, scalarized objectives, optimizers, and PF maps, and the final block summarizes the induced PF geometry for $M=2$, where the scalar weight $w\in[0,1]$ induces the path $\fPF(w)$, whose traversal speed, arc length, and normalized arc-length CDF are given by $v$, $s$, and $\Farc$, respectively.}
\label{tab:notation_moo}
\end{table}

\paragraph{Gradient-Based MOO and Scalarization.}
Scalarization methods for reducing MOO to single-objective subproblems are widely used in practice~\cite{marler2010weighted}, including LS, Tchebycheff scalarization~\citep{jahn1985scalarization,lin2024smooth}, $\epsilon$-constraint method and its variants~\citep{chankong2008multiobjective,momma2022multi,mahapatra2020multi,ye2022pareto,chen2024pareto,roy2023optimization,chen2025foops},  and reference-point methods~\citep{wierzbicki1998reference}. Among them, LS is a commonly used baseline~\citep{zhou2024beyond,yang2024learning,zhang2024optimal,chen2025gradient_moo_review}. Additionally, gradient-based multi-objective methods such as MGDA~\citep{sener2018multi,liu2021conflict,liu2021stochastic,zhou2022_SMOO,fernando2022mitigating,chen2023three,xiao2023direction} typically return a single Pareto-stationary solution per run, rather than a preference-specified or uniformly distributed set of Pareto points.

\paragraph{PF Approximation.}
A variety of methods have been proposed to achieve some coverage of the PF. Classical Normal Boundary Intersection (NBI) method~\citep{das1998normal} generates well-distributed points via structured subproblems. Optimistic Linear Support (OLS) adaptively selects scalarization weights to construct the convex coverage set using repeated single-objective solves, making it a standard adaptive scalarization baseline in MORL~\citep{barrett2008learning,alegre2022optimistic}. Adaptive weighting strategies, such as min–max formulations with local arc-length control, promote uniform spacing but typically rely on sequential PF tracing and differential-geometry ingredients~\citep{zhang2006minmax}. Another line of work enforces equal-spacing constraints explicitly within an optimization framework, solved sequentially using off-the-shelf optimizers~\citep{gill2006snopt} or in parallel via second-order techniques~\citep{pereyra1999asynchronous,pereyra2009fast,pereyra2013equispaced}, which can be memory intensive and less attractive when only first-order or noisy oracle access is available. 
Another line of work employs continuation or PF-tracking methods, tracing the Pareto manifold via differential equations and predictor-corrector schemes~\citep{lovison2011singular,martin2016continuation,vieira2012multicriteria,bolten2021tracing}. Recent studies assess uniformity via the averaged Hausdorff distance~\citep{bogoya2019averaged}, and develop second-order methods to promote uniform PF spacing~\citep{wang2024newton,zhang2024gliding}, which are expensive due to their second-order nature and the need for additional operations, e.g., sorting or clustering.

\paragraph{Pushforward Measures and Distribution Matching.}
Histogram equalization provides a classical example of distribution matching by transforming an empirical distribution to a uniform target~\citep{gonzalez2009digital,billingsley2017probability}. In our setting, treating the scalarization weight $w$ as a random variable induces a pushforward measure on the PF through the map $\mathbf{f}_{\rm PF}(w)$. Uniform arc-length sampling can therefore be viewed as matching this induced measure to a uniform target on the front. Related perspectives arise in monotone calibration~\citep{zadrozny2002transforming}, adaptive moving-mesh methods~\citep{huang2010adaptive,xu2011convergence}, and optimal-transport-based mesh adaptation~\citep{budd2009moving,weller2016mesh}.

\section{Notation Table}
\label{app:notation}

Table~\ref{tab:notation_moo} summarizes the main problem-specific notation used throughout the paper, grouped by original-space quantities, hidden-space quantities, and PF geometry.

\section{Preliminaries and Geometry Analysis Details}
\label{app:prelim_theory}
\begin{figure}[t]
 \centering
 \includegraphics[width=0.99\linewidth]{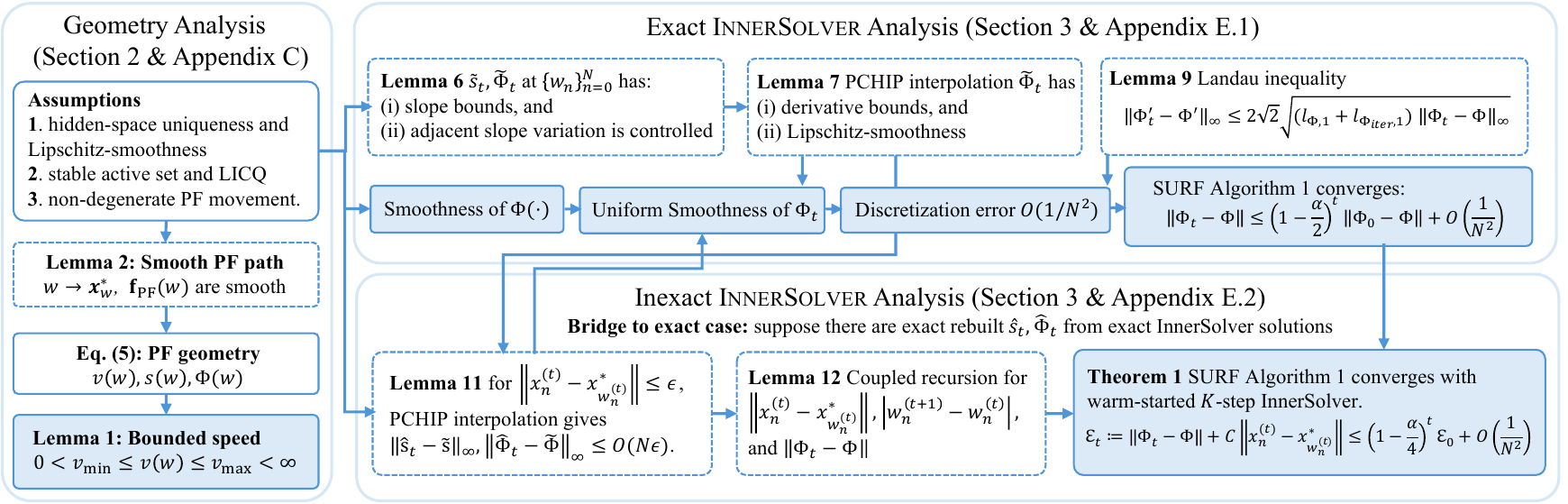}
\caption{
Theoretical analysis roadmap for SURF. The solid-line blocks summarize the main results presented in the main paper, while the dashed-line blocks indicate supporting intermediate results developed in the appendix. 
}
\label{fig:proof_roadmap}
\end{figure}
In this appendix, we develop the geometric ingredients shown in the left part of Figure~\ref{fig:proof_roadmap}. Specifically, Appendix~\ref{app:fPFdifferentiable} proves smoothness of the scalarization-induced PF path, Appendix~\ref{appendix:v_bound_from_tangent_gap} establishes bounds on the traversal speed, Appendix~\ref{app:multi_objective_extension} discusses the extension to the multiple-objective setting, and Appendix~\ref{app:quadratic} provides an analytical quadratic example.

\subsection{Smoothness of the PF}
\label{app:fPFdifferentiable}

The following lemma shows that LS induces a Lipschitz-smooth PF path, and that the PF point moves in a monotone direction as the scalarization weight varies.

\begin{lemma}[Smoothness and monotonicity of the scalarization path]
\label{lem:fPFdifferentiable}
Suppose Assumptions~\ref{assump:hidden_regular},~\ref{assump:LICQ} hold.
Then, for every $w\in[0,1]$, the scalarized problem \eqref{eq:LS_multi} has a unique minimizer $\bx_w^*$.
Moreover, the mapping $w\mapsto \bx_w^*$ is Lipschitz-continuous on $[0,1]$, continuously differentiable on $(0,1)$, and admits one-sided derivatives at $w=0$ and $w=1$ that are the one-sided limits of $\partial_w\bx_w^*$. The mapping
\begin{align}
w\mapsto \fPF(w)=\f(\bx_w^*)=[{f_{\rm PF}}_1(w),{f_{\rm PF}}_2(w)]^\top
\end{align}
is continuously differentiable on $(0,1)$ and admits one-sided derivatives at $w=0$ and $w=1$ that are the one-sided limits of $\partial_w\fPF(w)$.
Additionally, ${f_{\rm PF}}_1$ is nonincreasing and ${f_{\rm PF}}_2$ is nondecreasing on $[0,1]$. In particular,
\begin{align}
\frac{\partial}{\partial w}\fPF(w)\in \R_- \times \R_+,
\quad \forall w\in[0,1],
\end{align}
that is, $\frac{\partial}{\partial w}{f_{\rm PF}}_1(w)\le 0$ and $\frac{\partial}{\partial w}{f_{\rm PF}}_2(w)\ge 0$ for all $w\in [0,1]$.

Additionally, if Assumption~\ref{assump:tangent_gap} holds, $\fPF(\cdot)$, $\Farc(\cdot)$, and $\fPF(s^{-1}(\cdot))$ are respectively $l_{\fPF,1}$-, $l_{\Farc,1}$- and $l_{\tilde{\mathbf{f}}_{\rm PF},1}$-Lipschitz-smooth, and $v$ is $l_{v,0}$-Lipschitz-continuous.
\end{lemma}

\begin{proof}
Fix any $w\in[0,1]$. By Assumption~\ref{assump:hidden_regular}, the scalarized objective
\begin{align}
\LS_f(\bx;w)=w f_1(\bx)+(1-w)f_2(\bx)
\end{align}
is $\mu$-strongly convex on the convex set $\X$. Therefore, the minimizer
$\bx_w^*=\arg\min_{\bx\in\X}\LS_f(\bx;w)$ is unique.

We next study how this solution varies with $w$.
Under Assumptions~\ref{assump:hidden_regular},~\ref{assump:LICQ}, let $\boldsymbol{\nu}_w^*$ denote the corresponding optimal Lagrange multiplier vector for the scalarized problem \eqref{eq:LS_multi}. By the standard sensitivity result for nonlinear programming, the primal-dual solution mapping
$
w\mapsto (\bx_w^*,\boldsymbol{\nu}_w^*)
$
is locally continuously differentiable in $w$~\cite[Theorem 2.1]{fiacco1976sensitivity}.
Moreover, the feasible set $\X$ does not depend on $w$, and the scalarized objective
$\LS_f(\bx;w)$ is affine, hence Lipschitz-continuous, in $w$. Therefore,
$
w\mapsto (\bx_w^*,\boldsymbol{\nu}_w^*)
$
is also Lipschitz-continuous in $w$~\citep[Theorem 2.16]{ito2008lagrange}. In particular,
$w\mapsto \bx_w^*$ is Lipschitz-continuous on $[0,1]$ and continuously differentiable on $(0,1)$, with one-sided derivatives at $w=0$ and $w=1$.
Since $f_1$ and $f_2$ are continuously differentiable, the composition
$\fPF(w)=\f(\bx_w^*)$ is continuously differentiable on $(0,1)$, with one-sided derivatives at the endpoints. 

It remains to show the monotonicity of the PF coordinates. Fix $0\le w_1<w_2\le 1$. By optimality of $\bx_{w_1}^*$ for weight $w_1$,
\begin{align}
w_1 f_1(\bx_{w_1}^*)+(1-w_1)f_2(\bx_{w_1}^*)
\le
w_1 f_1(\bx_{w_2}^*)+(1-w_1)f_2(\bx_{w_2}^*).
\label{eq:opt_w1_pfmono_re_merge}
\end{align}
Similarly, by optimality of $\bx_{w_2}^*$ for weight $w_2$,
\begin{align}
w_2 f_1(\bx_{w_2}^*)+(1-w_2)f_2(\bx_{w_2}^*)
\le
w_2 f_1(\bx_{w_1}^*)+(1-w_2)f_2(\bx_{w_1}^*).
\label{eq:opt_w2_pfmono_re_merge}
\end{align}
Multiplying \eqref{eq:opt_w1_pfmono_re_merge} by $1-w_2$ and \eqref{eq:opt_w2_pfmono_re_merge} by $1-w_1$, then adding, gives $(w_2-w_1)\big(f_1(\bx_{w_2}^*)-f_1(\bx_{w_1}^*)\big)\le 0$.
Since $w_2-w_1>0$, ${f_{\rm PF}}_1$ is nonincreasing on $[0,1]$, as follows from
\begin{align}
{f_{\rm PF}}_1(w_2)=f_1(\bx_{w_2}^*)\le f_1(\bx_{w_1}^*)={f_{\rm PF}}_1(w_1).
\end{align}
Substituting this back into \eqref{eq:opt_w1_pfmono_re_merge}, we obtain
\begin{align}
(1-w_1)\big(f_2(\bx_{w_1}^*)-f_2(\bx_{w_2}^*)\big)
\le
w_1\big(f_1(\bx_{w_2}^*)-f_1(\bx_{w_1}^*)\big)\le 0,
\end{align}
hence ${f_{\rm PF}}_2(w_1)=f_2(\bx_{w_1}^*)\le f_2(\bx_{w_2}^*)={f_{\rm PF}}_2(w_2)$ and ${f_{\rm PF}}_2$ is nondecreasing on $[0,1]$.

Since $\fPF$ is continuously differentiable on $[0,1]$, the above monotonicity implies $\frac{\partial}{\partial w}{f_{\rm PF}}_1(w)\le 0$ and $\frac{\partial}{\partial w}{f_{\rm PF}}_2(w)\ge 0$ for all $w\in[0,1]$. Equivalently, $\frac{\partial}{\partial w}\fPF(w)\in \R_- \times \R_+$.
Let $\mathbf c(\bx)$ denote the fixed active constraint vector along the solution path, and let $\boldsymbol{\nu}_w^*$ be the associated Lagrange multiplier vector at $\bx_w^*$, so 
\begin{align}
\mathcal L(\bx,\boldsymbol{\nu};w):=wf_1(\bx)+(1-w)f_2(\bx)+\boldsymbol{\nu}^\top\mathbf c(\bx).
\end{align}
By Assumption~\ref{assump:LICQ}, the local KKT system is
\begin{align}
w\nabla f_1(\bx_w^*)+(1-w)\nabla f_2(\bx_w^*)+J_w^\top\boldsymbol{\nu}_w^*=0,\quad \mathbf c(\bx_w^*)=0,
\label{eq:local_kkt_smooth}
\end{align}
where $J_w:=\nabla\mathbf c(\bx_w^*)$. As discussed above, $w\mapsto(\bx_w^*,\boldsymbol{\nu}_w^*)$ is Lipschitz-continuous on $[0,1]$, continuously differentiable on $(0,1)$, and admits one-sided derivatives at the endpoints.

Differentiating \eqref{eq:local_kkt_smooth} with respect to $w$ gives
\begin{align}
\underbrace{\begin{bmatrix}\nabla_{\bx\bx}^2\mathcal L(\bx_w^*,\boldsymbol{\nu}_w^*;w)&J_w^\top\\J_w&0\end{bmatrix}}_{K(w):=}\begin{bmatrix}\frac{\partial}{\partial w}\bx_w^*\\\frac{\partial}{\partial w}\boldsymbol{\nu}_w^*\end{bmatrix}=-\begin{bmatrix}\nabla f_1(\bx_w^*)-\nabla f_2(\bx_w^*)\\0\end{bmatrix}.
\label{eq:diff_kkt_zprime_consistent}
\end{align}

By Assumption~\ref{assump:hidden_regular}, strong convexity on $\X$ implies that $\nabla_{\bx\bx}^2\mathcal L(\bx_w^*,\boldsymbol{\nu}_w^*;w)$ is positive definite on the active feasible tangent space $\ker(J_w)$; together with LICQ, this implies that $K(w)$ is nonsingular for every $w\in[0,1]$.
To use a full-space inverse, choose a sufficiently large fixed $c>0$ and define
\begin{align}
\widetilde \nabla_{\bx\bx}^2\mathcal L(\bx_w^*,\boldsymbol{\nu}_w^*;w):=\nabla_{\bx\bx}^2\mathcal L(\bx_w^*,\boldsymbol{\nu}_w^*;w)+cJ_w^\top J_w.
\end{align}
Since $\nabla_{\bx\bx}^2\mathcal L(\bx_w^*,\boldsymbol{\nu}_w^*;w)$ is uniformly positive definite on $\ker(J_w)$ and $J_w$ has full row rank by LICQ, this choice gives $\widetilde \nabla_{\bx\bx}^2\mathcal L(\bx_w^*,\boldsymbol{\nu}_w^*;w)\succ0$ uniformly along the compact path.
Moreover, $J_w\frac{\partial}{\partial w}\bx_w^*=0$, so replacing $\nabla_{\bx\bx}^2\mathcal L(\bx_w^*,\boldsymbol{\nu}_w^*;w)$ by $\widetilde \nabla_{\bx\bx}^2\mathcal L(\bx_w^*,\boldsymbol{\nu}_w^*;w)$ does not change the primal sensitivity equation.
Since $\widetilde \nabla_{\bx\bx}^2\mathcal L(\bx_w^*,\boldsymbol{\nu}_w^*;w)\succ0$ and $J_w$ has full row rank, the Schur-complement inverse formula applies to the augmented KKT system. On the compact path $\{\bx^*_w:w\in[0,1]\}$, the uniform restricted strong convexity and LICQ constants imply that $\widetilde \nabla_{\bx\bx}^2\mathcal L(\bx_w^*,\boldsymbol{\nu}_w^*;w)^{-1}$ and $\left(J_w\widetilde \nabla_{\bx\bx}^2\mathcal L(\bx_w^*,\boldsymbol{\nu}_w^*;w)^{-1}J_w^\top\right)^{-1}$ are uniformly bounded.
Moreover, by Assumption~\ref{assump:LICQ}, the Lagrangian Hessian is Lipschitz-continuous along the solution path; together with the fixed active set and LICQ, this implies that the Schur-complement factor mapping the right-hand side of \eqref{eq:diff_kkt_zprime_consistent} to $\frac{\partial}{\partial w}\bx_w^*$ is Lipschitz-continuous in $w$.
Since $w\mapsto\bx_w^*$ is Lipschitz-continuous and $\nabla f_m$ is Lipschitz-continuous on $\mathcal K$, the vector $\nabla f_1(\bx_w^*)-\nabla f_2(\bx_w^*)$ is Lipschitz-continuous in $w$.
Therefore, there exists $L_x<\infty$ such that
\begin{align}
\left\|\frac{\partial}{\partial w}\bx_{w_1}^*-\frac{\partial}{\partial w}\bx_{w_2}^*\right\|\le L_x|w_1-w_2|.
\label{eq:xprime_Lipschitz_consistent}
\end{align}

By the chain rule, $\frac{\partial}{\partial w}\fPF_m(w)
=
\nabla f_m(\bx_w^*)^\top \frac{\partial}{\partial w}\bx_w^*$, for $m\in\{1,2\}$.
Since $\nabla f_m$ is Lipschitz-continuous and bounded on the compact set containing
$\{\bx_w^*:w\in[0,1]\}$, while $w\mapsto \bx_w^*$ is Lipschitz-continuous and
$\frac{\partial}{\partial w}\bx_w^*$ is bounded and Lipschitz-continuous by \eqref{eq:xprime_Lipschitz_consistent}, we have
\begin{align}
& \left\|
\frac{\partial}{\partial w}\fPF(w_1)
-
\frac{\partial}{\partial w}\fPF(w_2)
\right\|
\nonumber\\
\leq &
\sum_{m=1}^2
\left|
\nabla f_m(\bx_{w_1}^*)^\top \frac{\partial}{\partial w}\bx_{w_1}^*
-
\nabla f_m(\bx_{w_2}^*)^\top \frac{\partial}{\partial w}\bx_{w_2}^*
\right|
\nonumber\\
\leq &
\sum_{m=1}^2
\left\|
\nabla f_m(\bx_{w_1}^*)-\nabla f_m(\bx_{w_2}^*)
\right\|
\left\|
\frac{\partial}{\partial w}\bx_{w_1}^*
\right\|
+
\sum_{m=1}^2
\left\|
\nabla f_m(\bx_{w_2}^*)
\right\|
\left\|
\frac{\partial}{\partial w}\bx_{w_1}^*
-
\frac{\partial}{\partial w}\bx_{w_2}^*
\right\|
\nonumber\\
\leq &
l_{\fPF,1}|w_1-w_2|
\end{align}
for some constant $l_{\fPF,1}<\infty$.
Now, since $v(w)=\left\|\frac{\partial}{\partial w}\fPF(w)\right\|$, the triangle inequality yields
\begin{align}
|v(w_1)-v(w_2)|
\le
\left\|
\frac{\partial}{\partial w}\fPF(w_1)
-
\frac{\partial}{\partial w}\fPF(w_2)
\right\|
\le
l_{\fPF,1}|w_1-w_2|.
\end{align}
Thus $v$ is Lipschitz-continuous on $[0,1]$ with modulus $l_{v,0}\le l_{\fPF,1}$.
For ${\fPF}_s(s):=\fPF(s^{-1}(s))$, by Lemma~\ref{lem:v_bound_from_tangent_gap},
\begin{align}
v(w)=s'(w)\ge v_{\min}>0,
\quad
\forall w\in[0,1].
\end{align}
Hence $s$ is strictly increasing, so $s^{-1}$ is well defined and Lipschitz-continuous with constant at most $1/v_{\min}$. By the chain rule, there is ${\fPF}_s'(s) = \frac{\frac{\partial}{\partial w}\fPF(w)}{v(w)} \big|_{w=s^{-1}(s)}$.
For any $w_1,w_2\in[0,1]$,
\begin{align}
0 \le &\left\| \frac{\frac{\partial}{\partial w}\fPF(w_1)}{v(w_1)} - \frac{\frac{\partial}{\partial w}\fPF(w_2)}{v(w_2)} \right\| \nonumber\\
\leq & \frac{ \left\| \frac{\partial}{\partial w}\fPF(w_1) - \frac{\partial}{\partial w}\fPF(w_2) \right\| }{v_{\min}} + \left\| \frac{\partial}{\partial w}\fPF(w_2) \right\| \left| \frac{1}{v(w_1)}-\frac{1}{v(w_2)} \right| \nonumber\\
\leq & \frac{l_{\fPF,1}}{v_{\min}}|w_1-w_2| + \frac{v_{\max}l_{v,0}}{v_{\min}^2}|w_1-w_2|.
\end{align}
Now let $s_1,s_2\in[0,s(1)]$, and write $w_i=s^{-1}(s_i)$. Then there exists
$l_{{\fPF}_s,1}<\infty$ such that
\begin{align}
\|{\fPF}_s'(s_1)-{\fPF}_s'(s_2)\|
\le
l_{{\fPF}_s,1}|s_1-s_2|.
\end{align}

Finally, by the fundamental theorem of calculus, $s$ is differentiable with $s'(w)=v(w)$. Since $\Farc=s/s(1)$, there is $\Farc'(w)=\frac{s'(w)}{s(1)}=\frac{v(w)}{s(1)}$.
Therefore, for any $w_1,w_2\in[0,1]$,
\begin{align}
|\Farc'(w_1)-\Farc'(w_2)|
=
\frac{|v(w_1)-v(w_2)|}{s(1)}
\le
\frac{l_{v,0}}{s(1)}|w_1-w_2|.
\end{align}
Thus $\Farc'$ is Lipschitz-continuous on $[0,1]$ with modulus $l_{\Farc,1}\le \frac{l_{v,0}}{s(1)}$. This completes the proof.
\end{proof}

Lemma~\ref{lem:fPFdifferentiable} shows that LS induces a Lipschitz-continuous and Lipschitz-smooth PF path.
A one-to-one parametrization by $w$ requires an additional nondegeneracy condition. This is established later in Lemma~\ref{lem:v_bound_from_tangent_gap} under Assumption~\ref{assump:tangent_gap}, which implies strictly positive PF speed.

\subsection{Proof of Lemma~\ref{lem:v_bound_from_tangent_gap}}
\label{appendix:v_bound_from_tangent_gap}

Before proving Lemma~\ref{lem:v_bound_from_tangent_gap}, we first clarify the geometric meaning of Assumption~\ref{assump:tangent_gap}.
\begin{remark}
\label{remark: vbound}
Under Assumption~\ref{assump:LICQ}, the active set is fixed for all $w\in[0,1]$ so $J_w$ is continuous with respect to $w$. Moreover, by KKT, there exists Lagrangian multiplier $\boldsymbol{\nu}_w^*$ such that $\Proj_{\ker(J_w)}\left( w \nabla f_1(\bx_w^*) + (1-w) \nabla f_2(\bx_w^*) + J_w^\top \boldsymbol{\nu}_w^* \right)= 0$.
Here, $\Proj_{\ker(J_w)}$ is a linear operation as $\ker(J_w)$ is a linear subspace. Then, the range of $J_w^\top$ is orthogonal to $\ker(J_w)$ and hence, for $w\in(0,1)$,
\begin{align*}
\Proj_{\ker(J_w)}\big(\nabla f_1(\bx_w^*)-\nabla f_2(\bx_w^*)\big)
= & \frac{1}{1-w}\Proj_{\ker(J_w)}\nabla f_1(\bx_w^*)
= -\frac{1}{w}\Proj_{\ker(J_w)}\nabla f_2(\bx_w^*).
\end{align*}
Therefore, Assumption~\ref{assump:tangent_gap} is equivalent to requiring $\|\Proj_{\ker(J_w)}\nabla f_1(\bx_w^*)\| \ge c(1-w)$, or $\|\Proj_{\ker(J_w)}\nabla f_2(\bx_w^*)\| \ge c w$, with endpoint cases interpreted by one-sided limits.

Geometrically, Assumption~\ref{assump:tangent_gap} rules out points where the two objectives are first-order identical along feasible directions, which would make the PF traversal speed vanish. Such degeneracy is non-generic and typically arises only from special symmetry or gradient alignment. Similar regularity conditions are standard in sensitivity and geometric analyses of PSs/fronts, where Lipschitz-smooth structure is derived from multiplier regularity, rank conditions, or Lipschitz-smooth path-following assumptions~\citep{hamada2020topology,bolten2021tracing}. Our condition serves the same purpose, but is tailored to arc-length analysis by directly enforcing a uniformly nonzero tangent trade-off.
\end{remark}

We now use this tangent-space separation to prove the uniform lower and upper bounds on the PF traversal speed.

\begin{proof}[Proof of Lemma~\ref{lem:v_bound_from_tangent_gap}]
Under Assumptions~\ref{assump:hidden_regular},~\ref{assump:LICQ}, Lemma~\ref{lem:fPFdifferentiable} gives that $w\mapsto \bx_w^*$ and $\fPF(w)$ are differentiable on $w\in[0,1]$. By Assumption~\ref{assump:hidden_regular}, the solution path $\{\bx_w^*:w\in[0,1]\}$ lies in the compact set $\mathcal K\subset \operatorname{relint}(\mathcal X)$, and $\nabla f_1,\nabla f_2$ are $l_{f,1}$-Lipschitz-continuous on $\mathcal K$. Together with the fixed active set and LICQ in Assumption~\ref{assump:LICQ}, standard sensitivity results imply that $w\mapsto \bx_w^*$ is continuously differentiable on $[0,1]$ up to one-sided derivatives at the endpoints. Therefore
\begin{align}
\frac{\partial}{\partial w}\fPF(w)
= & J_f(\bx_w^*) \frac{\partial \bx_w^*}{\partial w}
\end{align}
is continuous on $[0,1]$. Since $[0,1]$ is compact, $v_{\max}:=\sup_{w\in[0,1]}\left\|\frac{\partial}{\partial w}\fPF(w)\right\|<\infty$.

To derive the lower bound, fix any $w\in[0,1]$. Let $\boldsymbol{\nu}_w^*$ denote the optimal Lagrange multipliers of the scalarized problem $\min_{x\in\X}\LS_f(x;w)$, and let $\mathbf c$ denote the active constraint functions associated with the locally fixed active index set under Assumption~\ref{assump:LICQ}. The local KKT system at $\bx_w^*$ is
\begin{align}
\underbrace{w \nabla f_1(\bx_w^*) + (1-w)\nabla f_2(\bx_w^*)}_{\mathrm{I}(w)} + \underbrace{J_w^\top \boldsymbol{\nu}_w^*}_{\mathrm{II}(w)} = 0,\quad \underbrace{\mathbf c(\bx_w^*)}_{\mathrm{III}(w)} = 0,
\label{eq:local_kkt}
\end{align}
where $J_w := \nabla \mathbf c(\bx_w^*)$ is the Jacobian of the active constraints at $\bx_w^*$.

Next, differentiate the stationarity condition in \eqref{eq:local_kkt}. For the first term $\mathrm{I}(w)$, there is
\begin{align}
\frac{\partial}{\partial w} \mathrm{I}(w)
= & \nabla f_1(\bx_w^*)-\nabla f_2(\bx_w^*)+\big(w\nabla^2 f_1(\bx_w^*) + (1-w)\nabla^2 f_2(\bx_w^*)\big)\frac{\partial}{\partial w}\bx_w^*.
\end{align}
For the second term $\mathrm{II}(w)$, using $J_w=\nabla \mathbf c(\bx_w^*)$, there is
\begin{align}
\frac{\partial}{\partial w}\mathrm{II}(w)
= & \frac{\partial}{\partial w}\big(\nabla \mathbf c(\bx_w^*)^\top \boldsymbol{\nu}_w^*\big) =\nabla_{xx}^2 \big((\boldsymbol{\nu}_w^*)^\top \mathbf c(\bx_w^*)\big)\frac{\partial}{\partial w}\bx_w^*+J_w^\top \frac{\partial}{\partial w}\boldsymbol{\nu}_w^*.
\end{align}

Since the active index set is locally fixed, the constraint map $\mathbf c$ itself does not depend on $w$ explicitly, but only through $\bx_w^*$. Differentiating the feasibility condition $\mathbf c(\bx_w^*)=0$ with respect to $w$ therefore
\begin{align}
\frac{\partial}{\partial w}\mathrm{III}(w)
= & J_w \frac{\partial}{\partial w}\bx_w^*
= 0.
\end{align}
Hence, $\frac{\partial}{\partial w}\bx_w^* \in \ker(J_w)$. Combining the above yields
\begin{align}
&~~\mathcal H_w \frac{\partial}{\partial w}\bx_w^*+(\nabla f_1(\bx_w^*)-\nabla f_2(\bx_w^*))+J_w^\top \frac{\partial}{\partial w}\boldsymbol{\nu}_w^*=0,\quad J_w \frac{\partial}{\partial w}\bx_w^*=0,
\label{eq:diff_kkt_clean}\\
&\text{where}~
\mathcal H_w = \nabla_{xx}^2 \mathcal L_w(\bx_w^*,\boldsymbol{\nu}_w^*)
= w\nabla^2 f_1(\bx_w^*)+(1-w)\nabla^2 f_2(\bx_w^*)+\nabla_{xx}^2\big((\boldsymbol{\nu}_w^*)^\top \mathbf c(\bx_w^*)\big) .
\end{align}

Let $\Pi_w\in\mathbb R^{d\times d}$ denote the orthogonal projection matrix onto the tangent space $\ker(J_w)$. That is, for every $u\in\mathbb R^d$, $\Proj_{\ker(J_w)}(u)=\Pi_w u$. Since $\ker(J_w)$ is a linear subspace, such matrix $\Pi_w$ exists and is uniquely defined. Moreover, by the nature of projection, $\|\Pi_w u\| \le \|u\|$ for all $u\in\mathbb{R}^d$. Additionally, let $Q_w\in\mathbb R^{d\times r_w}$ be a matrix whose columns form an orthonormal basis of the tangent space $\ker(J_w)$, i.e., $\Pi_w = Q_w Q_w^\top$. Here, denote $r_w=\dim(\ker(J_w))$, there is $Q_w^\top Q_w = I_{r_w}$.

Since \eqref{eq:diff_kkt_clean} implies $J_w \frac{\partial}{\partial w}\bx_w^* = 0$, we have $\frac{\partial}{\partial w}\bx_w^* \in \ker(J_w)$. Therefore, there exists a vector $p_w\in\mathbb R^{r_w}$ such that $\frac{\partial}{\partial w}\bx_w^* = Q_w p_w$.

Premultiplying the first equation in \eqref{eq:diff_kkt_clean} by $Q_w^\top$ gives
\begin{align}
Q_w^\top \mathcal H_w \frac{\partial}{\partial w}\bx_w^*+Q_w^\top (\nabla f_1(\bx_w^*)-\nabla f_2(\bx_w^*))+Q_w^\top J_w^\top \frac{\partial}{\partial w}\boldsymbol{\nu}_w^*=0.
\end{align}
Since the columns of $Q_w$ span $\ker(J_w)$, we have $J_w Q_w = 0$, and hence $Q_w^\top J_w^\top = (J_w Q_w)^\top = 0$. Substituting $\frac{\partial}{\partial w}\bx_w^* = Q_w p_w$ yields
\begin{align}
Q_w^\top \mathcal H_w Q_w p_w
= & -Q_w^\top (\nabla f_1(\bx_w^*)-\nabla f_2(\bx_w^*)).
\end{align}
Denote the reduced Hessian $R_w := Q_w^\top \mathcal H_w Q_w$. Then $p_w = - R_w^{-1} Q_w^\top (\nabla f_1(\bx_w^*)-\nabla f_2(\bx_w^*))$, and
\begin{align}
\frac{\partial}{\partial w}\bx_w^*
= & -Q_w R_w^{-1} Q_w^\top (\nabla f_1(\bx_w^*)-\nabla f_2(\bx_w^*)).
\end{align}

The inverse above is well defined. Indeed, by Assumption~\ref{assump:hidden_regular}, both $f_1$ and $f_2$ are $\mu$-strongly convex, and by Assumption~\ref{assump:LICQ} the active inequality constraints are convex with nonnegative multipliers. Hence the Lagrangian Hessian satisfies $\mathcal H_w \succeq \mu I$. Therefore, for every $q\in\mathbb R^{r_w}$,
\begin{align}
q^\top R_w q
= & q^\top Q_w^\top \mathcal H_w Q_w q =(Q_w q)^\top \mathcal H_w (Q_w q) \ge \mu \|Q_w q\|^2 =\mu \|q\|^2.
\end{align}
So $R_w \succeq \mu I$, and in particular $R_w$ is positive definite and invertible.

Now let $g_w := \nabla f_1(\bx_w^*)-\nabla f_2(\bx_w^*)$. Then
\begin{align}
\left|\frac{\partial}{\partial w}\big(f_1(\bx_w^*)-f_2(\bx_w^*)\big)\right|
=\left|g_w^\top \frac{\partial}{\partial w}\bx_w^*\right|= \left|-g_w^\top Q_w R_w^{-1} Q_w^\top g_w\right| =(Q_w^\top g_w)^\top R_w^{-1}(Q_w^\top g_w).
\label{eq:diff_gap_reduced}
\end{align}

Since the path lies in the compact set $\mathcal K$ and $\nabla f_m$ is $l_{f,1}$-Lipschitz-continuous on $\mathcal K$, the objective Hessians are uniformly bounded along the path. Together with the fixed active set and Lipschitz-smooth active constraints, there exists $L_H<\infty$ such that $\lambda_{\max}(R_w)\le L_H$ for all $w\in[0,1]$. Hence, $R_w^{-1} \succeq \frac{1}{L_H} I$, and \eqref{eq:diff_gap_reduced} implies
\begin{align}
\left|\frac{\partial}{\partial w}\big(f_1(\bx_w^*)-f_2(\bx_w^*)\big)\right|
\ge & \frac{1}{L_H}\|Q_w^\top g_w\|^2.
\label{eq:gap_lb_mid}
\end{align}

Finally, since $Q_w Q_w^\top=\Pi_w$ is the orthogonal projection matrix onto $\ker(J_w)$,
\begin{align}
\|Q_w^\top g_w\|^2
= & g_w^\top Q_w Q_w^\top g_w = g_w^\top \Pi_w g_w =\|\Pi_w g_w\|^2 =\|\Proj_{\ker(J_w)}(g_w)\|^2.
\end{align}
By Assumption~\ref{assump:tangent_gap},  \eqref{eq:gap_lb_mid} gives $\left|\frac{\partial}{\partial w}\big(f_1(\bx_w^*)-f_2(\bx_w^*)\big)\right|
\ge \frac{c^2}{L_H}$.
Therefore, using $\sqrt{a^2+b^2}\ge |a-b|/\sqrt2$, there is
\begin{align}
\left\|\frac{\partial}{\partial w}\fPF(w)\right\|
= & \sqrt{\left(\frac{\partial}{\partial w}f_1(\bx_w^*)\right)^2+\left(\frac{\partial}{\partial w}f_2(\bx_w^*)\right)^2} \nonumber\\
\ge & \frac{1}{\sqrt2}\left|\frac{\partial}{\partial w}\big(f_1(\bx_w^*)-f_2(\bx_w^*)\big)\right| \geq \frac{c^2}{\sqrt2 L_H}.
\end{align}
In this way, the invertibility part follows directly from the fundamental theorem of calculus.
\end{proof}

\subsection{Extension to the Multi-Objective Setting}
\label{app:multi_objective_extension}

\paragraph{Smoothness of the PF.}
We next note that the Lipschitz-smoothness result extends naturally to the general $M$-objective setting, where the scalarization weight lies on the simplex $\Delta_M$. The following lemma establishes a Lipschitz-smooth PF map and a relative trade-off inequality along feasible pairwise directions of the simplex.

\begin{lemma}[Generalized version of Lemma~\ref{lem:fPFdifferentiable}]
\label{lem:fPFdifferentiable_multi}
Suppose Assumptions~\ref{assump:hidden_regular},~\ref{assump:LICQ} hold. Then, for every $\bw\in\Delta_M$, the scalarized problem \eqref{eq:LS_multi} has a unique minimizer $\bx_\bw^*$. Moreover, the mappings
\begin{align}
\bw\mapsto \bx_\bw^*, \quad \bw\mapsto \fPF(\bw)=\f(\bx_\bw^*)=[{f_{\rm PF}}_1(\bw),\dots,{f_{\rm PF}}_M(\bw)]^\top
\end{align}
are continuously differentiable on $\Delta_M$ in the relative sense, with one-sided directional derivatives on the boundary. In addition, $\bw\mapsto \bx_\bw^*$ is Lipschitz-continuous, and $\fPF$ is $l_{\fPF,1}$-Lipschitz-smooth on $\Delta_M$.
Additionally, fix any $i,j\in\{1,\dots,M\}$ with $i\neq j$, and any $\bw\in\Delta_M$ such that $\bw+t(\be_i-\be_j)\in\Delta_M$ for all sufficiently small $t\ge 0$. Then, for all such $t\ge 0$,
\begin{align}
{f_{\rm PF}}_i\bigl(\bw+t(\be_i-\be_j)\bigr)-{f_{\rm PF}}_i(\bw)\le {f_{\rm PF}}_j\bigl(\bw+t(\be_i-\be_j)\bigr)-{f_{\rm PF}}_j(\bw).
\end{align}
Consequently, $\bigl[D\fPF(\bw)(\be_i-\be_j)\bigr]_i\le \bigl[D\fPF(\bw)(\be_i-\be_j)\bigr]_j$.
\end{lemma}
\begin{proof}
The differentiability of $\bw\mapsto \bx_\bw^*$, $\bw\mapsto \fPF(\bw)=\f(\bx_\bw^*)=[{f_{\rm PF}}_1(\bw),\dots,{f_{\rm PF}}_M(\bw)]^\top$, and the Lipschitz-smoothness of $\fPF$ follows the same arguments as in the proof of Lemma~\ref{lem:fPFdifferentiable}.
It remains to show the directional trade-off inequality. Fix $i\neq j$, $\bw\in\Delta_M$, and $t>0$ such that $\tilde{\bw}:=\bw+t(\be_i-\be_j)\in\Delta_M$. Then $\tilde w_i=w_i+t$, $\tilde w_j=w_j-t$, and $\tilde w_k=w_k$ for $k\neq i,j$. By optimality of $\bx_\bw^*$ for weight $\bw$ and optimality of $\bx_{\tilde{\bw}}^*$ for weight $\tilde{\bw}$
\begin{align}
\sum_{m=1}^M w_m f_m(\bx_\bw^*)\le \sum_{m=1}^M w_m f_m(\bx_{\tilde{\bw}}^*),\quad 
\sum_{m=1}^M \tilde w_m f_m(\bx_{\tilde{\bw}}^*)\le \sum_{m=1}^M \tilde w_m f_m(\bx_\bw^*). \label{eq:multi_opt_wtilde_merge}
\end{align}
This yields $t\Big(f_i(\bx_{\tilde{\bw}}^*)-f_i(\bx_\bw^*)\Big)-t\Big(f_j(\bx_{\tilde{\bw}}^*)-f_j(\bx_\bw^*)\Big)\le 0$ and thus
\begin{align}
{f_{\rm PF}}_i(\tilde{\bw})-{f_{\rm PF}}_i(\bw)\le {f_{\rm PF}}_j(\tilde{\bw})-{f_{\rm PF}}_j(\bw).
\end{align}
Finally, dividing by $t$ and letting $t\downarrow 0$ gives
$\bigl[D\fPF(\bw)(\be_i-\be_j)\bigr]_i\le \bigl[D\fPF(\bw)(\be_i-\be_j)\bigr]_j$. This completes the proof.
\end{proof}
Lemma~\ref{lem:fPFdifferentiable_multi} shows that the scalarization-induced PF map remains smooth in the multiple-objective case. We next discuss how the bi-objective speed and CDF generalize to a surface-area density on the simplex.
\paragraph{Generalization of $v$, $s$, and $\Farc$.}
In the general $M$-objective setting, the weight $\bw$ lies on the $(M-1)$-dimensional simplex $\Delta_M$, so LS induces an $(M-1)$-dimensional surface on the PF rather than a curve. Let $D\fPF(\bw)$ denote the Jacobian of $\fPF$ with respect to local coordinates on $\Delta_M$. The natural geometric analogue of the bi-objective speed is the Jacobian-induced surface-area density
\begin{align}
v(\bw):=\sqrt{\det\!\big(D\fPF(\bw)^\top D\fPF(\bw)\big)}. \label{eq:density_v}
\end{align}
When $M=2$, $\Delta_2$ is one-dimensional and \eqref{eq:density_v} reduces to the usual PF traversal speed in \eqref{eq:arc_cdf_def}. For $M>2$, $v(\bw)$ plays the role of a surface-area density on the simplex, and its normalization induces uniform surface-area coverage of the PF.

This also gives a natural multivariate analogue of the arc-length CDF. Since $\bw\in\Delta_M$ has only $M-1$ degrees of freedom, we use the first $M-1$ coordinates and write $w_M=1-\sum_{m=1}^{M-1}w_m$. For $(w_1,\dots,w_{M-1})$ satisfying $w_m\ge 0$ and $\sum_{m=1}^{M-1}w_m\le 1$, define
\begin{align}
\Farc(w_1,\dots,w_{M-1}):=\frac{\int_{\{\mathbf{z}\in\Delta_M:0\le z_m\le w_m,\ m=1,\dots,M-1\}} v(\mathbf{z})\,d\mathbf{z}}{\int_{\Delta_M}v(\mathbf{z})\,d\mathbf{z}}.
\end{align}
Equivalently, $\Farc$ is the multivariate CDF associated with the surface-area density $v(\bw)$ under the coordinate chart $(w_1,\dots,w_{M-1})$. Thus, the bi-objective CDF refinement principle extends conceptually to the multi-objective case by replacing arc-length density with surface-area density on the simplex. Developing an efficient multi-dimensional refinement procedure, for example, through multivariate CDF estimation, transport maps, or resampling on $\Delta_M$, is beyond the scope of this paper and left for future work.

\subsection{Analytical Expressions for \texorpdfstring{$\Farc$}{Phi} in Quadratic Problems}
\label{app:quadratic}
We next give a simple quadratic example where the scalarization path and the corresponding arc-length CDF admit analytical expressions. This example illustrates the geometric mechanism behind SURF through integral and antiderivative representations.
\paragraph{$n$-Dimensional Quadratic Bi-Objective Problem.}
Consider the unconstrained bi-objective problem on $\bx\in \mathbb{R}^n$:
\begin{align}
\f(\bx)
=
\begin{bmatrix}
f_1(\bx)\\
f_2(\bx)
\end{bmatrix}
=
\begin{bmatrix}
(\bx-\bb_1)^\top Q_1 (\bx-\bb_1)\\
(\bx-\bb_2)^\top Q_2 (\bx-\bb_2)
\end{bmatrix},
\quad Q_1,Q_2\succ 0,\quad \bb_1\neq\bb_2.
\end{align}
For each scalarization weight $w\in[0,1]$, $\LS_f(\bx;w)=w f_1(\bx)+(1-w)f_2(\bx)$ is strongly convex and admits the unique minimizer
\begin{align}
\bx_w^*
=
\bigl(wQ_1+(1-w)Q_2\bigr)^{-1}
\bigl(wQ_1\bb_1+(1-w)Q_2\bb_2\bigr),
\label{eq:quad_nd_xw}
\end{align}
and the PS is $\XPF
=
\{\bx_w^*:w\in[0,1]\}$.
Let
\begin{align}
\Sigma(w):=wQ_1+(1-w)Q_2,
\quad
P(w):=\Sigma(w)^{-1}wQ_1,
\quad
\bd:=\bb_1-\bb_2.
\end{align}
Then \eqref{eq:quad_nd_xw} can be written as $\bx_w^* = P(w)\bb_1+\bigl(I-P(w)\bigr)\bb_2 = \bb_2+P(w)\bd$. Therefore, $\bx_w^*-\bb_1 = -\bigl(I-P(w)\bigr)\bd$ and $\bx_w^*-\bb_2 = P(w)\bd$, and the induced PF manifold is
\begin{align}
\f_{\rm PF}(w)
:=
\f(\bx_w^*)
=
\begin{bmatrix}
\bd^\top \bigl(I-P(w)\bigr)^\top Q_1 \bigl(I-P(w)\bigr)\bd\\
\bd^\top P(w)^\top Q_2 P(w)\bd
\end{bmatrix},
\quad w\in[0,1].
\end{align}

Differentiating $P(w)=\Sigma(w)^{-1}wQ_1$ gives $P'(w)
=
\Sigma(w)^{-1}\Bigl(Q_1\bigl(I-P(w)\bigr)+Q_2P(w)\Bigr)$ and
\begin{align}
\frac{\partial}{\partial w}\f_{\rm PF}(w)
=
\begin{bmatrix}
-2 \bd^\top \bigl(I-P(w)\bigr)^\top Q_1 P'(w)\bd\\
\phantom{-}2 \bd^\top P(w)^\top Q_2 P'(w)\bd
\end{bmatrix}.
\end{align}
In this way, the PF traversal speed satisfies
\begin{align}
v(w)
&:=
\left\|
\frac{\partial}{\partial w}\f_{\rm PF}(w)
\right\| =
2\left\|
\begin{bmatrix}
-\bd^\top \bigl(I-P(w)\bigr)^\top Q_1 \Sigma(w)^{-1}\bigl(Q_1(I-P(w))+Q_2P(w)\bigr)\bd\\
\phantom{-}\bd^\top P(w)^\top Q_2 \Sigma(w)^{-1}\bigl(Q_1(I-P(w))+Q_2P(w)\bigr)\bd
\end{bmatrix}
\right\| \nonumber\\
&\propto
\sqrt{(m_1(w)+m_3(w))^2+(m_2(w)+m_3(w))^2},
\label{eq:quad_nd_speed}
\end{align}
where
\begin{align*}
m_1(w)&:=\bd^\top (I-P(w))^\top Q_1\Sigma(w)^{-1}Q_1(I-P(w))\bd,\\
m_2(w)&:=\bd^\top P(w)^\top Q_2\Sigma(w)^{-1}Q_2P(w)\bd,\\
m_3(w)&:=%
\bd^\top P(w)^\top Q_2\Sigma(w)^{-1}Q_1(I-P(w))\bd .
\end{align*}
Thus, in general, $v(w)$ is non-constant, so uniform sampling of $w$ does not produce uniform PF coverage. Accordingly,
\begin{align}
\Farc(w)=
\frac{\int_0^w \sqrt{(m_1(u)+m_3(u))^2+(m_2(u)+m_3(u))^2} du}
{\int_0^1 \sqrt{(m_1(u)+m_3(u))^2+(m_2(u)+m_3(u))^2} du}.
\end{align}
Thus, Rule~\ref{prin:pf_aware_ls_weighting} applies directly in this example. $\Farc$ and its inverse can be computed numerically in a straightforward manner, and the PF-aware weights are then given by $w_n=\Farc^{-1}(n/N)$.

\paragraph{One-Dimensional Quadratic Bi-Objective Problem.}
In the one-dimensional case $x\in\mathbb{R}$, consider
\begin{align}
\f(x)
=
\begin{bmatrix}
f_1(x)\\
f_2(x)
\end{bmatrix}
=
\begin{bmatrix}
q_1(x-b_1)^2\\
q_2(x-b_2)^2
\end{bmatrix},
\qquad q_1,q_2>0,\quad b_1\neq b_2.
\end{align}
For each $w\in[0,1]$, the scalarized objective $\LS_f(x;w)=wf_1(x)+(1-w)f_2(x)$ is strongly convex and admits the unique minimizer $x_w^*=\frac{wq_1b_1+(1-w)q_2b_2}{wq_1+(1-w)q_2}$.
Thus, the PS is $\X_{\rm PF}=\{x_w^*:w\in[0,1]\}=[\min\{b_1,b_2\},\max\{b_1,b_2\}]$, and the induced PF is
\begin{align}
\f_{\rm PF}(w)
=
\f(x_w^*)
=
\frac{(b_2-b_1)^2}{\bigl(wq_1+(1-w)q_2\bigr)^2}
\begin{bmatrix}
q_1q_2^2(1-w)^2\\
q_2q_1^2w^2
\end{bmatrix}.
\end{align}
Differentiating yields the PF traversal speed
\begin{align}
v(w)
:=
\left\|\frac{d}{dw}\f_{\rm PF}(w)\right\|
=
\frac{2q_1^2q_2^2(b_2-b_1)^2\sqrt{(1-w)^2+w^2}}
{\bigl(wq_1+(1-w)q_2\bigr)^3}.
\label{eq:quadratic_vw}
\end{align}
Thus, uniformly spaced scalarization weights generally do not yield uniform PF coverage.

Since multiplicative constants cancel after normalization, $\Farc$ depends on $(q_1,q_2)$ only through $r:=q_2/q_1$.
To evaluate the arc-length integral, set $\zeta(w):=\frac{rw-(1-w)}{w+(1-w)r}$.
Define $H(\zeta):=\zeta\sqrt{1+\zeta^2}+\operatorname{asinh}(\zeta)$.
Since $\zeta(0)=-1/r$, integrating the traversal speed gives
\begin{align}
s(w)
=
\int_0^w v(u)\,du
=
\frac{q_1^2q_2^2(b_2-b_1)^2}{(q_1^2+q_2^2)^{3/2}}
\bigl[H(\zeta(w))+H(1/r)\bigr].
\end{align}
Using $\zeta(1)=r$ and normalizing by $s(1)$ yields $\Farc(w)=\frac{H(\zeta(w))+H(1/r)}{H(r)+H(1/r)}$.

\section{PF Geometry in Bi-Objective Reinforcement Learning}
\label{app:morl}
This appendix specializes the PF-geometry analysis to multi-objective reinforcement learning. We first verify the regularity of the discounted occupancy-measure constraints in Appendix~\ref{app:rank_condition}, then derive the PF speed formula used in the main paper in Appendix~\ref{appendix: proof of speed for negative conditional entropy}, and finally analyze the arc-length CDF estimation error in the singleton-state bandit case in Appendix~\ref{appendix:bandit_R_and_Farc_err}.

\subsection{Provable Regularity Condition for Discounted Occupancy Constraints}
\label{app:rank_condition}
\begin{lemma}[Provable verification of Assumption~\ref{assump:LICQ} for RL]
\label{lem:rank_E_minus_gamma_P}
Let $E_{(s,a),s'}=\mathbf 1\{s'=s\}$, $P_{(s,a),s'}=\mathcal P(s' \mid s,a)$, and suppose $\gamma\in(0,1)$. Then $E-\gamma P$ has full column rank. Equivalently, $(E-\gamma P)^\top$ has full row rank.
\end{lemma}

\begin{proof}
It suffices to show that
$\ker(E-\gamma P)=\{0\}$.
Let $y\in\mathbb R^{|\mathcal S|}$ satisfy
$
(E-\gamma P)y=0.
$
Then, for every $(s,a)\in\mathcal S\times\mathcal A$,
$
y(s)=\gamma\sum_{s'\in\mathcal S}\mathcal P(s'\mid s,a) y(s').
$
Choose $\bar s\in\arg\max_{s\in\mathcal S}|y(s)|$. Since
$\mathcal P(\cdot\mid \bar s,a)$ is a probability distribution, for any
$a\in\mathcal A$,
$
\|y\|_\infty
=
|y(\bar s)|
\le
\gamma\sum_{s'\in\mathcal S}\mathcal P(s'\mid \bar s,a) |y(s')|
\le
\gamma \|y\|_\infty .
$
Because $0<\gamma<1$, this implies $\|y\|_\infty=0$, hence $y=0$.
Therefore $\ker(E-\gamma P)=\{0\}$, so $E-\gamma P$ has full column rank.
\end{proof}

\subsection{Inner Optimization Trajectories
Stay in the Subset \texorpdfstring{$\cal K$}{K} in the Relative Interior of \texorpdfstring{$\cal X$}{X}}
\label{app:PGD_trajectory}
For the negative conditional-entropy formulation~\eqref{eq: MDP}, the gradient magnitude diverges as any element of $\bx$ goes to $0$. Fix some $\bx^{(0)}\in \operatorname{relint}(\X)=\{\bx\in\mathbb R_{++}^{|\mathcal S||\mathcal A|}:(E-\gamma P)^\top\bx=(1-\gamma)\rho\}$ as the initialization. Then, for each $w\in[0,1]$, the sublevel set $\{\bx\in\X:\LS_{\f}(\bx;w)\le \LS_{\f}(\bx^{(0)};w)\}$ remains strictly bounded away from the boundary of $\X$. Therefore, we can directly adopt
\begin{align}
\mathcal K:=\bigcup_{w\in[0,1]}\{\bx\in\X:\LS_{\f}(\bx;w)\le \LS_{\f}(\bx^{(0)};w)\} \label{eq:level_set}
\end{align}
Since $[0,1]$ is compact and $\LS_{\f}(\bx;w)$ is continuous in $(\bx,w)$, it follows that $\mathcal K$ is a compact subset of $\operatorname{relint}(\X)$. Moreover, for any fixed $w\in[0,1]$, projected gradient descent {on the policy or occupancy simplex} with stepsize $l_{f,1}^{-1}$ on the $\mu$-strongly convex and {locally} $l_{f,1}$-smooth objective  at $\bx^{(t)}$ satisfies
\begin{align}
\LS_{\f}(\bx^{(t+1)};w)-\LS_{\f}(\bx_w^*;w)
\le
\big(1-\mu l_{f,1}^{-1}\big)
\big(\LS_{\f}(\bx^{(t)};w)-\LS_{\f}(\bx_w^*;w)\big), \label{eq:PGD_contraction}
\end{align}
where $0<1-\mu l_{f,1}^{-1}<1$~\citep{fatkhullin2025stochastic}. Hence $\LS_{\f}(\bx^{(t+1)};w)\le \LS_{\f}(\bx^{(t)};w)\le \LS_{\f}(\bx^{(0)};w)$, so $\bx^{(t)}\in \mathcal K$ for all $t$. By the bijection $h_m(\bu)=f_m(g(\bu))$ and the equivalence between policy-space and occupancy-space optimization~\citep{schlaginhaufen2024towards}, the same sublevel-set invariance also holds for the corresponding policy-space update~\citep{fatkhullin2025stochastic}.

For general problems under Assumption~\ref{assump:hidden_regular}, fix $\bx^{(0)}\in {\cal K}$, and we can similarly obtain $\cup_{w\in[0,1]}\{\bx\in\X:\LS_{\f}(\bx;w)\le \LS_{\f}(\bx^{(0)};w)\} \subseteq {\cal K}$. Since $\LS_{\f}(\bx;w)$ is continuous in $(\bx,w)$ and $[0,1]$ is compact, this level-set union is closed. Moreover, $\LS_{\f}(\cdot;w)$ is $\mu$-strongly convex for every $w\in[0,1]$, so its sublevel sets are bounded. By compactness of $[0,1]$, the bound is uniform over $w$. Hence $\cup_{w\in[0,1]}\{\bx\in\X:\LS_{\f}(\bx;w)\le \LS_{\f}(\bx^{(0)};w)\}$ is compact. Therefore, we can obtain $l_{f,1}=\sup_{\bx\in\mathcal K,\;w\in[0,1]}\|\nabla^2\LS_{\f}(\bx;w)\|_2<\infty$. Following a similar analysis to~\eqref{eq:PGD_contraction}, we can show that $\bx^{(t)}$ stays in the level-set union that belongs to $\cal K$, where each $f_m$ remains $l_{f,1}$-smooth on $\cal K$, for all $t$ and all $w\in [0,1]$.

\subsection{Proof of Proposition~\ref{prop: speed for negative conditional entropy}: PF Speed in Bi-Objective RL}
\label{appendix: proof of speed for negative conditional entropy}

\begin{proof}
The unique minimizer satisfies $\bx_w^*\in\Delta_{S\times A}^{\circ}$~\citep{muller2026optimal}, so only the equality constraint is active. Hence the Lagrangian for $\min_{\bx\in\X}\LS_f(\bx;w)$ is
\begin{align}
\mathcal L(\bx,\boldsymbol{\nu};w)=\LS_f(\bx;w)+\left\langle\boldsymbol{\nu},(E-\gamma P)^\top \bx-(1-\gamma)\rho\right\rangle .
\end{align}
The KKT conditions are
\begin{align}
0
=& \nabla_{\boldsymbol{\nu}}\mathcal L(\bx_w^*,\boldsymbol{\nu}_w^*;w)
=(E-\gamma P)^\top \bx_w^*-(1-\gamma)\rho,
\label{eq:KKT_nu_rl}\\
0
=& \nabla_{\bx}\mathcal L(\bx_w^*,\boldsymbol{\nu}_w^*;w) \nonumber\\
=& \underbrace{\beta\log\bx_w^*-\beta E\log(E^\top\bx_w^*)-\beta R_0}_{=\nabla_{\bx}\bar\tau(\bx_w^*)}
-\bigl(wR_1+(1-w)R_2\bigr)+(E-\gamma P)\boldsymbol{\nu}_w^* .
\label{eq:KKT_x_rl}
\end{align}

Since $H_w$ is positive definite on $\ker(J)$ and $J=(E-\gamma P)^\top$ has full row rank by Lemma~\ref{lem:rank_E_minus_gamma_P}, standard sensitivity results imply that $w\mapsto(\bx_w^*,\boldsymbol{\nu}_w^*)$ is Lipschitz-continuous in $w$~\citep[Theorem 2.16]{ito2008lagrange} and locally continuously differentiable~\cite[Theorem 2.1]{fiacco1976sensitivity}.

Differentiating \eqref{eq:KKT_nu_rl}-\eqref{eq:KKT_x_rl} with respect to $w$ yields
\begin{align}
J\nabla_w\bx_w^*
=& 0, \quad
H_w\nabla_w\bx_w^*+J^\top\nabla_w\boldsymbol{\nu}_w^*
= R_1-R_2,
\label{eq:diff_kkt_stat_rl}
\end{align}
where $J=(E-\gamma P)^\top$ and $H_w=\beta\bigl(\Diag(1/\bx_w^*)-E\Diag(1/(E^\top\bx_w^*))E^\top\bigr)$.
Although $H_w$ is singular on the full space, it is positive definite on the feasible tangent space $\ker(J)$; we therefore use the augmented Hessian $\widetilde H_w:=H_w+cJ^\top J$ with any fixed $c>0$.

For any $z$, the conditional-entropy Hessian is positive semidefinite:
\begin{align}
z^\top H_w z
=& \beta\sum_s\left(\sum_a\frac{z(s,a)^2}{\bx_w^*(s,a)}-\frac{\left(\sum_a z(s,a)\right)^2}{\sum_a\bx_w^*(s,a)}\right)
\ge 0 .
\end{align}
Equality in the Cauchy-Schwarz bound requires $z(s,a)=\lambda_s\bx_w^*(s,a)$ for all $a$, i.e., only a state-marginal perturbation; by sufficient exploration and the bijective policy-occupancy correspondence, no nonzero such direction lies in $\ker(J)$, so $H_w$ is positive definite on $\ker(J)$. For every $z\ne0$,
\begin{align}
z^\top \widetilde H_w z
=& z^\top H_w z+c\|Jz\|^2
>0,
\end{align}
where the case $Jz\ne0$ follows from $c\|Jz\|^2>0$, and the case $Jz=0$ follows from positive definiteness of $H_w$ on $\ker(J)$. 

Thus $\widetilde H_w\succ0$, and since $J$ has full row rank, $\widetilde S_w:=J\widetilde H_w^{-1}J^\top$ is positive definite and invertible.

Because $J\nabla_w\bx_w^*=0$, replacing $H_w$ by $\widetilde H_w$ does not change the primal sensitivity equation:
\begin{align}
\widetilde H_w\nabla_w\bx_w^*+J^\top\nabla_w\boldsymbol{\nu}_w^*
=& H_w\nabla_w\bx_w^*+cJ^\top J\nabla_w\bx_w^*+J^\top\nabla_w\boldsymbol{\nu}_w^* \nonumber\\
=& H_w\nabla_w\bx_w^*+J^\top\nabla_w\boldsymbol{\nu}_w^* \nonumber\\
=& R_1-R_2.
\end{align}
Therefore, the same $\nabla_w\bx_w^*$ solves the augmented KKT system
\begin{align}
\begin{bmatrix}\widetilde H_w&J^\top\\J&0\end{bmatrix}
\begin{bmatrix}\nabla_w\bx_w^*\\\nabla_w\boldsymbol{\nu}_w^*\end{bmatrix}
=
\begin{bmatrix}R_1-R_2\\0\end{bmatrix}.
\label{eq:aug_KKT_condition_on_derive}
\end{align}

Since $\widetilde H_w\succ0$ and $\widetilde S_w=J\widetilde H_w^{-1}J^\top\succ0$, the standard Schur-complement inverse formula gives
\begin{align}
\begin{bmatrix}\widetilde H_w&J^\top\\J&0\end{bmatrix}^{-1}
=&
\begin{bmatrix}
\widetilde H_w^{-1}-\widetilde H_w^{-1}J^\top\widetilde S_w^{-1}J\widetilde H_w^{-1}
&
\widetilde H_w^{-1}J^\top\widetilde S_w^{-1}
\\
\widetilde S_w^{-1}J\widetilde H_w^{-1}
&
-\widetilde S_w^{-1}
\end{bmatrix},\\
\Rightarrow
\quad \nabla_w\bx_w^*
=&
\left(\widetilde H_w^{-1}-\widetilde H_w^{-1}J^\top\widetilde S_w^{-1}J\widetilde H_w^{-1}\right)(R_1-R_2).
\label{eq:nabla_xw_rl}
\end{align}

Next, for the bi-objective PF map $\fPF(w)=[{f_{\rm PF}}_1(w),{f_{\rm PF}}_2(w)]^\top$, we have
\begin{align}
\frac{\partial}{\partial w}{f_{\rm PF}}_1(w)
=& \left\langle\nabla_{\bx} f_1(\bx_w^*),\nabla_w\bx_w^*\right\rangle = \left\langle\nabla_{\bx}\bar\tau(\bx_w^*)-R_1,\nabla_w\bx_w^*\right\rangle \nonumber\\
=& \left\langle(1-w)(R_2-R_1)-(E-\gamma P)\boldsymbol{\nu}_w^*,\nabla_w\bx_w^*\right\rangle \nonumber\\
=& (1-w)\left\langle R_2-R_1,\nabla_w\bx_w^*\right\rangle,
\label{eq:df1dw_rl}
\end{align}
where the third equality uses \eqref{eq:KKT_x_rl}, and the last uses \eqref{eq:diff_kkt_stat_rl}. Similar analysis for $\frac{\partial}{\partial w}{f_{\rm PF}}_2(w)$ gives
\begin{align}
v(w)
=& \left\|\frac{\partial}{\partial w}\fPF(w)\right\| =\sqrt{\left(\frac{\partial}{\partial w}{f_{\rm PF}}_1(w)\right)^2+\left(\frac{\partial}{\partial w}{f_{\rm PF}}_2(w)\right)^2} \nonumber\\
=& \sqrt{(1-w)^2+w^2}\left|(R_1-R_2)^\top\nabla_w\bx_w^*\right|.
\end{align}
Substituting \eqref{eq:nabla_xw_rl} yields the PF speed for bi-objective RL.

Now consider the case $|\mathcal S|=1$. Then the discounted occupancy measure coincides with the action distribution, so the feasible set reduces to $\Delta_A=\{\bx\in\mathbb R_+^A:\langle\mathbf 1,\bx\rangle=1\}$.
Since $\bx\mapsto\beta\langle\bx,\log\bx\rangle$ is strictly convex on $\Delta_A^\circ$ and the remaining terms are linear, $\LS_f(\cdot;w)$ is strictly convex on $\Delta_A$ and has a unique minimizer. 
The Lagrangian is
\begin{align*}
\mathcal L(\bx,\nu;w)=\LS_f(\bx;w)+\nu(\langle\mathbf 1,\bx\rangle-1).
\end{align*}
Stationarity gives $\beta\log\bx-\beta R_0-\bigl(wR_1+(1-w)R_2\bigr)+\nu\mathbf 1=0$, so $\bx$ is proportional to $\exp\left(R_0+\beta^{-1}\bigl(wR_1+(1-w)R_2\bigr)\right)$. Enforcing $\langle\mathbf 1,\bx\rangle=1$ yields
\begin{align}
\bu_w^*=\bx_w^*=\SoftMax\left(R_0+\beta^{-1}\bigl(wR_1+(1-w)R_2\bigr)\right)\in\Delta_A^\circ .
\end{align}

Differentiating $\bx_w^*=\SoftMax(z(w))$ with $z(w)=R_0+\beta^{-1}\bigl(wR_1+(1-w)R_2\bigr)$ gives $z'(w)=\beta^{-1}(R_1-R_2)$. Using the standard SoftMax Jacobian $D\SoftMax(z)=\Diag(\bx_w^*)-\bx_w^*(\bx_w^*)^\top$, 
\begin{align}
\frac{\partial}{\partial w}\bx_w^*
=
\beta^{-1}\Bigl(\Diag(\bx_w^*)-\bx_w^*(\bx_w^*)^\top\Bigr)(R_1-R_2).
\end{align}
Substituting this into $v(w)=\sqrt{(1-w)^2+w^2}\left|(R_1-R_2)^\top\frac{\partial}{\partial w}\bx_w^*\right|$ and using symmetry of $\Diag(\bx_w^*)-\bx_w^*(\bx_w^*)^\top$ gives \eqref{eq:bandit_speed}.
\end{proof}

\subsection{CDF Estimation Error in Bandit Problems}
\label{appendix:bandit_R_and_Farc_err}

\begin{proposition}[Arc-length CDF estimation error in bandits]
\label{prop:bandit_R_and_Farc_err}
Consider MDP with singleton state space $|\mathcal{S}|=1$ and $|\mathcal{A}|=A\ge 2$ actions. Suppose the bounded speed condition in Lemma~\ref{lem:v_bound_from_tangent_gap} holds.
Given a budget of $T$ pulls, suppose that each pull of arm $a$ returns a bounded bi-objective reward vector
$[r_1(a),r_2(a)]$ with $\E[r_m(a)]=R_m(a)$ for $m\in\{1,2\}$.
We estimate $R_m(a)$ by the empirical mean $\hat R_m(a)$.
Then, with probability at least $1-\delta$, there is
\begin{align}
\|\hatFarc-\Farc\|_\infty
 = \Oc\left(s(1)^{-1}\|\hat R_m-R_m\|_\infty \right) = 
\Oc \left(s(1)^{-1}\sqrt{A\log(A/\delta)T^{-1}}\right).
\end{align}
\end{proposition}
\begin{proof}
We first control the reward estimation error.
Under a total budget of $T$ pulls, suppose each arm is sampled at least
$\lfloor T/A\rfloor$ times. Since each reward component lies in
$[-R_{\max},R_{\max}]$, Hoeffding's inequality yields that for any $\epsilon>0$ and
each $m\in\{1,2\}$,
$
\Pr \left(
\max_{a\in[A]} |\hat R_m(a)-R_m(a)| \ge \epsilon
\right)
\le
2A \exp \left(
-\frac{\lfloor T/A\rfloor \epsilon^2}{2R_{\max}^2}
\right).
$
Hence, by choosing $\epsilon_R
:=
2R_{\max}\sqrt{\frac{A\log(4A/\delta)}{2T}}$,
we obtain, with probability at least $1-\delta$,
\begin{align}
\|\hat R_m-R_m\|_\infty \le \epsilon_R,
\quad m\in\{1,2\}.
\label{eq:bandit_R_uniform_bound}
\end{align}

We now bound $\|\hatFarc-\Farc\|_\infty$. For any $w\in[0,1]$,
\begin{align*}
\big|\hatFarc(w)-\Farc(w)\big|
= & 
\left|\frac{\hat s(w)}{\hat s(1)}-\frac{s(w)}{s(1)}\right| \leq 
\frac{|\hat s(w)-s(w)|}{s(1)}
+
\frac{s(w)}{s(1)}\frac{| \hat s(1)-s(1) |}{\hat s(1)}.
\end{align*}
Since $0\le s(w)\le s(1)$, we have
\begin{align*}
|\hat s(w)-s(w)|
\leq &
\int_0^w |\hat v_u(u)-v_u(u)| du
\le
\sup_{u\in[0,1]} |\hat v_u(u)-v_u(u)|,
\\
|\hat s(1)-s(1)|
\leq &
\int_0^1 |\hat v_u(u)-v_u(u)| du
\le
\sup_{u\in[0,1]} |\hat v_u(u)-v_u(u)|.
\end{align*}
Therefore,
\begin{align}
\big|\hatFarc(w)-\Farc(w)\big|
\le
\frac{1}{s(1)}\sup_{u\in[0,1]} |\hat v_u(u)-v_u(u)|
+
\frac{1}{\hat s(1)}\sup_{u\in[0,1]} |\hat v_u(u)-v_u(u)|.
\label{eq:Farc_diff_pre_event}
\end{align}

It remains to control the speed perturbation.
Since $|R_m(a)|\le R_{\max}$ for all $a\in[A]$ and $m\in\{1,2\}$,
the logits $\beta^{-1}(wR_1+(1-w)R_2)$ remain uniformly bounded in $\ell_\infty$
for all $w\in[0,1]$. Hence the softmax vector $\bx_w^*$ stays in a compact subset
of $\Delta_A^\circ$, uniformly over $w\in[0,1]$, and the map
$(R_1,R_2)\mapsto v(w)$ is uniformly Lipschitz-continuous over $w\in[0,1]$.
Therefore, there exists a constant $L_R<\infty$, depending only on
$(A,\beta,R_{\max})$, such that
\begin{align}
\sup_{w\in[0,1]} |\hat v(w)-v(w)|
\le
L_R\big(\|\hat R_1-R_1\|_\infty+\|\hat R_2-R_2\|_\infty\big).
\label{eq:bandit_speed_lipschitz}
\end{align}

On $\left\{ \sup_{u\in[0,1]} |\hat v_u(u)-v_u(u)| \le \frac{s(1)}{2} \right\}$, we have
$
\hat s(1)
\ge
s(1)-|\hat s(1)-s(1)|
\ge
s(1)-\sup_{u\in[0,1]} |\hat v_u(u)-v_u(u)|
\ge
\frac{s(1)}{2}.
$
Substituting this into \eqref{eq:Farc_diff_pre_event} gives, for all $w\in[0,1]$,
$
\big|\hatFarc(w)-\Farc(w)\big|
\le
\frac{1}{s(1)}\sup_{u} |\hat v_u(u)-v_u(u)|
+
\frac{2}{s(1)}\sup_{u} |\hat v_u(u)-v_u(u)|
=
\frac{3}{s(1)}\sup_{u\in[0,1]} |\hat v_u(u)-v_u(u)|.
$
Hence, on $\left\{ \sup_{u\in[0,1]} |\hat v_u(u)-v_u(u)| \le \frac{s(1)}{2} \right\}$,
\begin{align}
\|\hatFarc-\Farc\|_\infty
\le
\frac{3}{s(1)}\sup_{u\in[0,1]} |\hat v_u(u)-v_u(u)|.
\label{eq:Farc_bound_on_Ev}
\end{align}

Moreover, by \eqref{eq:bandit_speed_lipschitz}, the event $\left\{ \sup_{u\in[0,1]} |\hat v_u(u)-v_u(u)| \le \frac{s(1)}{2} \right\}$ is implied by
$
L_R\big(\|\hat R_1-R_1\|_\infty+\|\hat R_2-R_2\|_\infty\big)\le \frac{s(1)}{2}.
$
In particular, if
$
2L_R\epsilon_R\le \frac{s(1)}{2},
$
then \eqref{eq:bandit_R_uniform_bound} implies $\left\{ \sup_{u\in[0,1]} |\hat v_u(u)-v_u(u)| \le \frac{s(1)}{2} \right\}$.
Equivalently, it suffices that
$
\epsilon_R \le \frac{s(1)}{4L_R},
$
which holds whenever $T$ is large enough, namely
$
T
\gtrsim
\frac{A R_{\max}^2 L_R^2}{s(1)^2}\log \frac{A}{\delta}.
$

Finally, on the event \eqref{eq:bandit_R_uniform_bound}, combining
\eqref{eq:Farc_bound_on_Ev} and \eqref{eq:bandit_speed_lipschitz} yields
$
\|\hatFarc-\Farc\|_\infty
\le
\frac{3L_R}{s(1)}
\big(\|\hat R_1-R_1\|_\infty+\|\hat R_2-R_2\|_\infty\big)
\le
\frac{6L_R}{s(1)}\epsilon_R.
$
Therefore, with probability at least $1-\delta$,
$
\|\hatFarc-\Farc\|_\infty
=
O \left(
s(1)^{-1}\sqrt{\frac{A\log(A/\delta)}{T}}
\right),
$
where the hidden constant depends only on $(A,\beta,R_{\max})$ through $L_R$.
This completes the proof.
\end{proof}

\section{Convergence Analysis of SURF}
\label{app:cdf_refinement}

Building on the geometric results summarized in Figure~\ref{fig:proof_roadmap}, this appendix proves the convergence guarantees for SURF. We first analyze the exact-\textsc{InnerSolver} setting and then extend the argument to the practical warm-started finite-step \textsc{InnerSolver} used in Algorithm~\ref{alg:SURF}.

Algorithm~\ref{alg:SURF} can be implemented in the parameter space $\U$ and returns iterates
$\bu_n^{(t)}$, \textcolor{gray}{or in the hidden space $\X$ to return $\bx_{n}^{(t)}$}. Throughout this section, we work with their images in the convex hidden space $\bx_n^{(t)} := g(\bu_n^{(t)}), \bx_w^* := g(\bu_w^*)$ under Assumption~\ref{assump:hidden_regular}. Since $h_m(\bu)=f_m(g(\bu))$, we equivalently have $\mathbf h(\bu_n^{(t)})=\f(\bx_n^{(t)}),\mathbf h(\bu_w^*)=\f(\bx_w^*)$.
Hence, all approximation and chord-length bounds below are stated in the $\bx$-space.

\subsection{Proof of the First Part of Theorem~\ref{thm:cdf}: Discretization Error under Exact \textsc{InnerSolver}}
\label{appendix:Lipschitz-smoothness}
In this subsection, we first consider the exact-\textsc{InnerSolver} setting. We begin by establishing Lipschitz-smoothness properties of $\tilde\Farc_t$ and $\Farc_t$ through Lemmas~\ref{lem:discrete_growth_control_tildeF_biobj}-\ref{lem:bounds_all_t_with_interp}. We then use these properties to control the discretization error in Lemma~\ref{lem:s_tilde_combined_bound}, which completes the proof of the first part of Theorem~\ref{thm:cdf}.

\paragraph{Smoothness of $\Farc_t$ under Exact \textsc{InnerSolver} Solutions.}
To prove the Lipschitz-smoothness of $\Farc_t$, the key issue is that the update in Algorithm~\ref{alg:SURF} depends on an interpolated CDF constructed from finitely many sampled Pareto points. Thus, to control the outer-loop dynamics, we must show that this interpolated CDF remains monotone and sufficiently Lipschitz-smooth. To do so, we first consider the exact \textsc{InnerSolver} solutions. In this case, $\tilde s^{(t)}$ and $\tilde \Farc_t$ are built from exact sampled solutions $\bu_{n}^{(t)}=\bu^*_{w_n^{(t)}}$ \textcolor{gray}{or equivalently $\bx_{n}^{(t)}=\bx^*_{w_n^{(t)}}$}.

We establish this in three steps. First, we show that the rebuilt CDF increases at a controlled rate across the sampled weights, and that these local growth rates do not change too abruptly between neighboring intervals in Lemma~\ref{lem:discrete_growth_control_tildeF_biobj}. Next, we show that the PCHIP interpolation rule lifts these discrete controls into global derivative and Lipschitz-smoothness bounds for the reconstructed CDF in Lemma~\ref{lem:pchip_derivative_bounds}. Finally, we prove by induction that these properties are preserved by the SURF outer update throughout all iterations in Lemma~\ref{lem:bounds_all_t_with_interp}.

We begin with the discrete CDF values at the sampled weights. Before turning to the rebuilt CDF itself, we first present the following geometric lemma, which relates PF arc-length increments to chord lengths along the PF.
\begin{lemma}[Arc-chord bound]
\label{lemma:arc_chord_sqrt2_pf}
Suppose Assumptions~\ref{assump:hidden_regular}, \ref{assump:LICQ} hold.
Moreover, for any $0\le w_1<w_2\le 1$, there is $s(w_2)-s(w_1)
 \le \sqrt{2} 
\left\|\fPF(w_2)
-\fPF(w_1)\right\|$.
\end{lemma}
\begin{proof}
Denote $\tilde{\f}_{\mathrm{PF}}(s):=\fPF(s^{-1}(s))$ for $s\in[0,s(1)]$ and write
$
\fPF_s(s)
=
\begin{bmatrix}
{f_{\mathrm{PF}}}_{s,1}(s)\\
{f_{\mathrm{PF}}}_{s,2}(s)
\end{bmatrix}$.

Under Assumptions~\ref{assump:hidden_regular} and~\ref{assump:LICQ}, Lemma~\ref{lem:fPFdifferentiable} shows that ${f_{\mathrm{PF}}}_{s,1}$ is monotone nonincreasing and ${f_{\mathrm{PF}}}_{s,2}$ is monotone nondecreasing. Thus, along the front, the first coordinate is nonincreasing and the second coordinate is nondecreasing as $w$ increases. Since $s(\cdot)$ is strictly increasing, the same monotonicity carries over to the PF parametrization.
Therefore, on the interval $[s(w_1),s(w_2)]$,
\begin{align}
\int_{s(w_1)}^{s(w_2)} \big|{f_{\mathrm{PF}}}_{s,1}'(s)\big| ds
= & 
\big|{f_{\mathrm{PF}}}_{s,1}(s(w_2)) - {f_{\mathrm{PF}}}_{s,1}(s(w_1))\big|, \\
\int_{s(w_1)}^{s(w_2)} \big|{f_{\mathrm{PF}}}_{s,2}'(s)\big| ds
= & 
\big|{f_{\mathrm{PF}}}_{s,2}(s(w_2)) - {f_{\mathrm{PF}}}_{s,2}(s(w_1))\big|.
\end{align}

Using $\|u\|\le \|u\|_1$ and then $\|u\|_1\le \sqrt{2}\|u\|$ for all $u\in\mathbb R^2$, we obtain
\begin{align*}
s(w_2)-s(w_1)
= & 
\int_{s(w_1)}^{s(w_2)} \|\f_{\mathrm{PF},s}'(s)\| ds \\
\leq &
\int_{s(w_1)}^{s(w_2)}
\left(
|{f_{\mathrm{PF}}}_{s,1}'(s)| + |{f_{\mathrm{PF}}}_{s,2}'(s)|
\right)ds \\
= & 
\big|{f_{\mathrm{PF}}}_{s,1}(s(w_2)) - {f_{\mathrm{PF}}}_{s,1}(s(w_1))\big|
+
\big|{f_{\mathrm{PF}}}_{s,2}(s(w_2)) - {f_{\mathrm{PF}}}_{s,2}(s(w_1))\big| \\
= & 
\left\|
\f_{\mathrm{PF},s}(s(w_2))
-
\f_{\mathrm{PF},s}(s(w_1))
\right\|_1 \\
\leq &
\sqrt{2} 
\left\|
\fPF_s(s(w_2))
-
\fPF_s(s(w_1))
\right\|.
\end{align*}
Finally, since $\fPF_s(s(w_i))=\fPF(w_i)$ for $i\in\{1,2\}$, we conclude the proof.
\end{proof}

This geometric relation ensures that the cumulative chord-length construction provides a faithful approximation to the true arc-length growth along the front. In particular, it allows us to control how fast the rebuilt CDF increases across sampled weights, and how much this local growth can vary from one interval to the next. We formalize this in the next lemma.

\begin{lemma}[Discrete growth control of the rebuilt CDF]
\label{lem:discrete_growth_control_tildeF_biobj}
Suppose Assumptions~\ref{assump:hidden_regular},~\ref{assump:LICQ},~\ref{assump:tangent_gap} hold. 
For any sampled weights $\{w_n^{(t)}\}_{n=0}^N$, let $\tilde \Farc_t(w_n^{(t)})=\frac{\tilde s^{(t)}(w_n^{(t)})}{\tilde s^{(t)}(1)}$ where $\tilde s^{(t)}(w_n^{(t)})$ is generated via \eqref{eq:sw_tilde} with exact \textsc{InnerSolver} solutions. Then:

\begin{enumerate}[label=(\roman*), leftmargin=*, align=left]
 \item for every $n\in\{0,\cdots,N-1\}$,
 \begin{align}
 \frac{1}{\sqrt2}\frac{v_{\min}}{v_{\max}}
 \le
 \frac{\tilde \Farc_t(w_{n+1}^{(t)})-\tilde \Farc_t(w_n^{(t)})}{w_{n+1}^{(t)}-w_n^{(t)}}
 \le
 \sqrt2 \frac{v_{\max}}{v_{\min}};
 \label{eq:tildeF_secant_bounds_biobj_merged}
 \end{align}

 \item for every $n\in\{0,\cdots,N-2\}$, there exists $l_{\tilde \Farc_t,n,1}
 \le
 \frac{1}{\tilde s^{(t)}(1)}
 \Bigl(l_{v,0}+\frac12 l_{\fPF,1}\Bigr)$
 such that
 \begin{align}
 \left|
 \frac{\tilde \Farc_t(w_{n+2}^{(t)})-\tilde \Farc_t(w_{n+1}^{(t)})}{w_{n+2}^{(t)}-w_{n+1}^{(t)}}
 -
 \frac{\tilde \Farc_t(w_{n+1}^{(t)})-\tilde \Farc_t(w_n^{(t)})}{w_{n+1}^{(t)}-w_n^{(t)}}
 \right|
 \le
 l_{\tilde \Farc_t,n,1}\bigl(w_{n+2}^{(t)}-w_n^{(t)}\bigr).
 \label{eq:tildeF_secant_variation_biobj_merged}
 \end{align}
\end{enumerate}
\end{lemma}

\begin{proof}
Since $\tilde \Farc_t=\tilde s^{(t)}/\tilde s^{(t)}(1)$, we have
\begin{align}
\frac{\tilde \Farc_t(w_{n+1}^{(t)})-\tilde \Farc_t(w_n^{(t)})}{w_{n+1}^{(t)}-w_n^{(t)}}
=
\frac{1}{\tilde s^{(t)}(1)}
\cdot
\frac{\|\fPF(w_{n+1}^{(t)})-\fPF(w_n^{(t)})\|}{w_{n+1}^{(t)}-w_n^{(t)}}.
\label{eq:tildeF_secant_as_chord_merged}
\end{align}
Thus it suffices to study the chord secant slopes of $\fPF$.

We first prove \eqref{eq:tildeF_secant_bounds_biobj_merged}.
For any $0\le a<b\le 1$, Lemma~\ref{lemma:arc_chord_sqrt2_pf} gives
\begin{align}
\|\fPF(b)-\fPF(a)\|
\ge
\frac{1}{\sqrt2}\bigl(s(b)-s(a)\bigr).
\label{eq:chord_ge_arc_over_sqrt2_merged}
\end{align}
Using the bounded speed condition in Lemma~\ref{lem:v_bound_from_tangent_gap}, we obtain
\begin{align}
\|\fPF(w_{n+1}^{(t)})-\fPF(w_n^{(t)})\|
&\ge
\frac{1}{\sqrt2}\int_{w_n^{(t)}}^{w_{n+1}^{(t)}}\|\fPF'(w)\| dw
\ge
\frac{1}{\sqrt2}v_{\min}(w_{n+1}^{(t)}-w_n^{(t)}),
\label{eq:seg_lower_merged}
\\
\|\fPF(w_{n+1}^{(t)})-\fPF(w_n^{(t)})\|
\leq &
\int_{w_n^{(t)}}^{w_{n+1}^{(t)}}\|\fPF'(w)\| dw
\le
v_{\max}(w_{n+1}^{(t)}-w_n^{(t)}).
\label{eq:seg_upper_merged}
\end{align}
Hence, on each interval $[w_n^{(t)},w_{n+1}^{(t)}]$, the secant slope of $\tilde s^{(t)}$ lies in $\left[\frac{1}{\sqrt2}v_{\min}, v_{\max}\right]$.

Applying Lemma~\ref{lemma:arc_chord_sqrt2_pf} to each sampled segment gives
$
\tilde s^{(t)}(1)=\sum_{n=0}^{N-1}\|\fPF(w_{n+1}^{(t)})-\fPF(w_n^{(t)})\|
\in
\left[\frac{1}{\sqrt2}s(1), s(1)\right],
$
while
$s(1)=\int_0^1 v(u) du\in [v_{\min},v_{\max}]$.
Therefore the secant slope of $\tilde \Farc_t$ on each interval $[w_n^{(t)},w_{n+1}^{(t)}]$ lies in
$\left[\frac{\frac{1}{\sqrt2}v_{\min}}{v_{\max}},\frac{v_{\max}}{\frac{1}{\sqrt2}v_{\min}}\right]=\left[\frac{1}{\sqrt2}\frac{v_{\min}}{v_{\max}},\sqrt2\frac{v_{\max}}{v_{\min}}\right]$, which proves \eqref{eq:tildeF_secant_bounds_biobj_merged}.

Next we prove \eqref{eq:tildeF_secant_variation_biobj_merged}. By the integral mean value theorem, there exist
$
\xi_n\in(w_n^{(t)},w_{n+1}^{(t)}), \quad \xi_{n+1}\in(w_{n+1}^{(t)},w_{n+2}^{(t)})
$ such that
\begin{align}
\frac{1}{w_{n+1}^{(t)}-w_n^{(t)}}\int_{w_n^{(t)}}^{w_{n+1}^{(t)}} v(u) du = v(\xi_n). \label{eq:ivt_speed_merged}
\end{align}
Hence, by Lemma~\ref{lem:fPFdifferentiable},
\begin{align}
    0 \le & \left| \frac{1}{w_{n+2}^{(t)}-w_{n+1}^{(t)}}\int_{w_{n+1}^{(t)}}^{w_{n+2}^{(t)}} v(u) du - \frac{1}{w_{n+1}^{(t)}-w_n^{(t)}}\int_{w_n^{(t)}}^{w_{n+1}^{(t)}} v(u) du \right| \nonumber\\
    = &  |v(\xi_{n+1})-v(\xi_n)| \le l_{v,0}|\xi_{n+1}-\xi_n| \le l_{v,0}\bigl(w_{n+2}^{(t)}-w_n^{(t)}\bigr).
    \label{eq:avg_speed_jump_merged}
\end{align}

Next, we bound the chord-arc gap on each interval using Lipschitzness of $\fPF'$.
Using $|\|x\|-\|y\||\le \|x-y\|$ and Jensen/triangle inequalities,
\begin{align}
0\le &
\frac{1}{w_{n+1}^{(t)}-w_n^{(t)}}\int_{w_n^{(t)}}^{w_{n+1}^{(t)}}\|\fPF'(u)\| du
-
\left\|
\frac{1}{w_{n+1}^{(t)}-w_n^{(t)}}\int_{w_n^{(t)}}^{w_{n+1}^{(t)}}\fPF'(u) du
\right\|
\nonumber\\
\le &
\frac{1}{w_{n+1}^{(t)}-w_n^{(t)}}\int_{w_n^{(t)}}^{w_{n+1}^{(t)}}
\left\|
\fPF'(u)-\frac{1}{w_{n+1}^{(t)}-w_n^{(t)}}\int_{w_n^{(t)}}^{w_{n+1}^{(t)}}\fPF'(r) dr
\right\|du.
\label{eq:norm_gap_basic_merged}
\end{align}
For any fixed $u\in[w_n^{(t)},w_{n+1}^{(t)}]$, by Lemma~\ref{lem:fPFdifferentiable} ,
\begin{align}
0 \le &\left\|
\fPF'(u)-\frac{1}{w_{n+1}^{(t)}-w_n^{(t)}}\int_{w_n^{(t)}}^{w_{n+1}^{(t)}}\fPF'(r) dr
\right\|
\nonumber\\
\le & 
\frac{1}{w_{n+1}^{(t)}-w_n^{(t)}}\int_{w_n^{(t)}}^{w_{n+1}^{(t)}}
\|\fPF'(u)-\fPF'(r)\| dr
\nonumber\\
\le & 
\frac{1}{w_{n+1}^{(t)}-w_n^{(t)}}\int_{w_n^{(t)}}^{w_{n+1}^{(t)}} l_{\fPF,1}|u-r| dr
\le
\frac{l_{\fPF,1}}{2}(w_{n+1}^{(t)}-w_n^{(t)}).
\label{eq:avg_tangent_dev_merged}
\end{align}
Averaging over $u$ and combining with \eqref{eq:norm_gap_basic_merged} yields
\begin{align}
\left|
\frac{\|\fPF(w_{n+1}^{(t)})-\fPF(w_n^{(t)})\|}{w_{n+1}^{(t)}-w_n^{(t)}}
-
\frac{1}{w_{n+1}^{(t)}-w_n^{(t)}}\int_{w_n^{(t)}}^{w_{n+1}^{(t)}} v(u) du
\right|
\le
\frac{l_{\fPF,1}}{2}(w_{n+1}^{(t)}-w_n^{(t)}).
\label{eq:chord_arc_gap_merged}
\end{align}

Combining \eqref{eq:avg_speed_jump_merged} and \eqref{eq:chord_arc_gap_merged} on the two consecutive intervals
$[w_n^{(t)},w_{n+1}^{(t)}]$ and $[w_{n+1}^{(t)},w_{n+2}^{(t)}]$ by the triangle inequality gives
\begin{align}
0 \le &\left| \frac{\|\fPF(w_{n+2}^{(t)})-\fPF(w_{n+1}^{(t)})\|}{w_{n+2}^{(t)}-w_{n+1}^{(t)}} - \frac{\|\fPF(w_{n+1}^{(t)})-\fPF(w_n^{(t)})\|}{w_{n+1}^{(t)}-w_n^{(t)}} \right| \nonumber\\
\le & 
l_{v,0}\bigl(w_{n+2}^{(t)}-w_n^{(t)}\bigr)
+\frac{l_{\fPF,1}}{2}
\Bigl((w_{n+2}^{(t)}-w_{n+1}^{(t)})+(w_{n+1}^{(t)}-w_n^{(t)})\Bigr)
\nonumber\\
= & 
\Bigl(l_{v,0}+\frac12 l_{\fPF,1}\Bigr)
\bigl(w_{n+2}^{(t)}-w_n^{(t)}\bigr).
\label{eq:chord_slope_jump_merged}
\end{align}
Multiplying \eqref{eq:chord_slope_jump_merged} by $1/\tilde s^{(t)}(1)$ and using \eqref{eq:tildeF_secant_as_chord_merged}
yields \eqref{eq:tildeF_secant_variation_biobj_merged}.
\end{proof}

Lemma~\ref{lem:discrete_growth_control_tildeF_biobj} provides the two discrete properties needed for the interpolation analysis. The bound in~\eqref{eq:tildeF_secant_bounds_biobj_merged} shows that the rebuilt CDF increases at a controlled rate at the sampled weights, so it is neither too flat nor too steep. The bound in~\eqref{eq:tildeF_secant_variation_biobj_merged} further shows that these local growth rates do not change too abruptly between neighboring intervals. Together, these properties ensure that the PCHIP reconstruction preserves a stable global shape and admits uniform derivative and Lipschitz-smoothness bounds.

\begin{lemma}[Derivative bounds for the PCHIP rebuild]
\label{lem:pchip_derivative_bounds}
Suppose Assumptions~\ref{assump:hidden_regular},~\ref{assump:LICQ},~\ref{assump:tangent_gap} hold. Consider a CDF $\Farc_t$ satisfying $|\Farc_t'|\in \left[ \frac{1}{2\sqrt2}\frac{v_{\min}}{v_{\max}}, \frac{5}{2\sqrt2}\frac{v_{\max}}{v_{\min}}\right]$. Let
$\{w_n^{(t)}\}_{n=0}^N$ be the sampled weights generated via $w_n^{(t)} = \Farc_t^{-1}(n/N)$, let $\tilde s^{(t)}(w_n^{(t)})$ be generated via \eqref{eq:sw_tilde} with exact \textsc{InnerSolver} solutions, let $\tilde s^{(t)}$ be the PCHIP interpolant built via Algorithm~\ref{alg:PCHIP}, and define $\tilde \Farc_t(w) =\frac{\tilde s^{(t)}(w)}{\tilde s^{(t)}(1)}$. 
Denote some $l_{\tilde \Farc_t,n,1} \leq \frac{1}{\tilde s^{(t)}(1)}\Bigl(l_{v,0}+\frac12 l_{\fPF,1}\Bigr)$.
Then, 
\begin{enumerate}[label=(\roman*), leftmargin=*, align=left]
 \item 
 \label{claim:grid}
 sampled weights $w_{n+1}^{(t)}-w_n^{(t)}\in\left[ \frac{2\sqrt2}{5}\frac{v_{\min}}{v_{\max}}\frac{1}{N}, 2\sqrt2\frac{v_{\max}}{v_{\min}}\frac{1}{N}
 \right]$ for all $n \in \{0,1,\cdots,N-1\}$.
 \item 
 \label{claim:secant}
 $\tilde \Farc_t$ is strictly increasing with
 $
 \tilde \Farc'_t \in \left[\frac{1}{\sqrt2}\frac{v_{\min}}{v_{\max}}-\frac{4\sqrt2\frac{v_{\max}}{v_{\min}}l_{\tilde \Farc_t,n,1}}{N}, 
 \sqrt2\frac{v_{\max}}{v_{\min}}+\frac{4\sqrt2\frac{v_{\max}}{v_{\min}}l_{\tilde \Farc_t,n,1}}{N}\right].
 $
 If
 $
 N \ge 16 \frac{v_{\max}^2}{v_{\min}^2} \frac{1}{\tilde s^{(t)}(1)}
 \Bigl(l_{v,0}+\frac12 l_{\fPF,1}\Bigr)
 $, then
 $
 \tilde \Farc_t'\in \left[ \frac{1}{2\sqrt2}\frac{v_{\min}}{v_{\max}}, \frac{5}{2\sqrt2}\frac{v_{\max}}{v_{\min}}\right].
 $
 \item \label{claim:smooth}
 $\tilde \Farc_t$ is Lipschitz-smooth with modulus
 $
 l_{\Farc_{\rm iter},1} = \Oc\left(\frac{v_{\max}^2}{s(1) v_{\min}^2} \Bigl(l_{v,0}+\frac12 l_{\fPF,1}\Bigr)\right).
 $
\end{enumerate}
\end{lemma}

\begin{proof}
Since $\Farc_t$ is strictly increasing and $\Farc_t'(w)\in \left[\frac{1}{2\sqrt2}\frac{v_{\min}}{v_{\max}}, 
\frac{5}{2\sqrt2}\frac{v_{\max}}{v_{\min}}\right]$, $w\in[0,1]$,
the inverse $\Farc_t^{-1}$ is well-defined on $[0,1]$ and satisfies
$
(\Farc_t^{-1})'(q)=\frac{1}{\Farc_t'(\Farc_t^{-1}(q))}
\in
\left[
\frac{2\sqrt2}{5}\frac{v_{\min}}{v_{\max}},
2\sqrt2\frac{v_{\max}}{v_{\min}}
\right].
$
Now, since $w_n^{(t)}=\Farc_t^{-1}(n/N)$ and $w_{n+1}^{(t)}-w_n^{(t)} = \Farc_t^{-1}\Big(\frac{n+1}{N}\Big)-\Farc_t^{-1}\Big(\frac{n}{N}\Big)$,
by the mean value theorem, there exists $\xi_n\in(n/N,(n+1)/N)$ such that
$w_{n+1}^{(t)}-w_n^{(t)}= (\Farc_t^{-1})'(\xi_n)\frac{1}{N}\in \left[ \frac{2\sqrt2}{5}\frac{v_{\min}}{v_{\max}}\frac{1}{N}, 2\sqrt2\frac{v_{\max}}{v_{\min}}\frac{1}{N} \right]$. This proves claim~\ref{claim:grid}.

Suppose Assumptions~\ref{assump:hidden_regular},~\ref{assump:LICQ},~\ref{assump:tangent_gap} hold. Lemma~\ref{lem:discrete_growth_control_tildeF_biobj}.(i) yields 
\begin{align}
s_n:=\frac{\tilde \Farc_t(w_{n+1}^{(t)})-\tilde \Farc_t(w_n^{(t)})}{w_{n+1}^{(t)}-w_n^{(t)}}\in\left[\frac{1}{\sqrt2}\frac{v_{\min}}{v_{\max}},\sqrt2\frac{v_{\max}}{v_{\min}}\right],\quad n=0,\cdots,N-1,
\end{align}
and Lemma~\ref{lem:discrete_growth_control_tildeF_biobj}.(ii) ensures that the adjacent secants satisfy
\begin{align}
|s_{n+1}-s_n|\le l_{\tilde \Farc_t,n,1}(w_{n+2}^{(t)}-w_n^{(t)}), \label{eq: secant Lipschitz-smooth}
\quad n=0,\cdots,N-2,
\end{align}
where $l_{\tilde \Farc_t,n,1}\leq \frac{1}{\tilde s^{(t)}(1)}\Bigl(l_{v,0}+\frac12 l_{\fPF,1}\Bigr)$ for all $n \in\{0,\cdots,N-2\}$.

Let $d_n:=\tilde \Farc'_t(w_n^{(t)})$, and $c_n:=3s_n-d_n-d_{n+1}$ for $n=0,\cdots,N-1$.
According to Algorithm~\ref{alg:PCHIP}, on each interval $[w_n^{(t)},w_{n+1}^{(t)}]$, writing $\theta=\frac{w-w_n^{(t)}}{w_{n+1}^{(t)}-w_n^{(t)}}\in[0,1]$, there is
\begin{align}
\tilde \Farc'_t(w)
=
(1-\theta)^2 d_n
+
2\theta(1-\theta)c_n
+
\theta^2 d_{n+1}.
\end{align}
Here, the coefficients $(1-\theta)^2$, $2\theta(1-\theta)$, $\theta^2$ are nonnegative and sum to one, so $\tilde \Farc'_t(w)$ is a convex combination of $d_n,c_n,d_{n+1}$. Therefore it suffices to bound these three quantities uniformly.

For interior sampled weights, i.e., $n=1,\cdots,N-1$, the PCHIP slope rule gives
\begin{align}
d_n\in[\min(s_{n-1},s_n), \max(s_{n-1},s_n)]\subseteq\left[\frac{1}{\sqrt2}\frac{v_{\min}}{v_{\max}},\sqrt2\frac{v_{\max}}{v_{\min}}\right].
\end{align}

It remains to control the endpoints $d_0,d_N$ and the interval control values $c_n$.
By the endpoint PCHIP formula, for some $\lambda\in(0,1)$,
one has $d_0=s_0+\lambda(s_0-s_1)$. In this way, \eqref{eq: secant Lipschitz-smooth} yields
\begin{align}
|d_0-s_0|
\le
\lambda|s_1-s_0|
\le
l_{\tilde \Farc_t,0,1}(w_2^{(t)}-w_0^{(t)})
\le
\frac{4 \sqrt{2}\frac{v_{\max}}{v_{\min}}l_{\tilde \Farc_t,n,1}}{N}.
\end{align}
Hence the endpoints satisfy
\begin{align}
d_0\in\Big[\frac{1}{\sqrt2}\frac{v_{\min}}{v_{\max}}-\frac{4 \sqrt{2}\frac{v_{\max}}{v_{\min}}l_{\tilde \Farc_t,n,1}}{N}, \sqrt2\frac{v_{\max}}{v_{\min}}+\frac{4 \sqrt{2}\frac{v_{\max}}{v_{\min}}l_{\tilde \Farc_t,n,1}}{N}\Big],\\
d_N\in\Big[\frac{1}{\sqrt2}\frac{v_{\min}}{v_{\max}}-\frac{4 \sqrt{2}\frac{v_{\max}}{v_{\min}}l_{\tilde \Farc_t,n,1}}{N}, \sqrt2\frac{v_{\max}}{v_{\min}}+\frac{4 \sqrt{2}\frac{v_{\max}}{v_{\min}}l_{\tilde \Farc_t,n,1}}{N}\Big].
\end{align}

Next consider $c_n=3s_n-d_n-d_{n+1}$. For interior intervals $n=1,\cdots,N-2$, we obtain
\begin{align}
|c_n-s_n|
\le
|s_n-d_n|+|s_n-d_{n+1}|
\le
|s_n-s_{n-1}|+|s_{n+1}-s_n|
\le
\frac{4 \sqrt{2}\frac{v_{\max}}{v_{\min}}l_{\tilde \Farc_t,n,1}}{N}.
\end{align}
Since $s_n\in\left[\frac{1}{\sqrt2}\frac{v_{\min}}{v_{\max}},\sqrt2\frac{v_{\max}}{v_{\min}}\right]$, it follows that $\forall n\in\{1,\cdots,N-2\}$,
\begin{align}
c_n\in
\Big[\frac{1}{\sqrt2}\frac{v_{\min}}{v_{\max}}-\frac{4 \sqrt{2}\frac{v_{\max}}{v_{\min}}l_{\tilde \Farc_t,n,1}}{N}, \sqrt2\frac{v_{\max}}{v_{\min}}+\frac{4 \sqrt{2}\frac{v_{\max}}{v_{\min}}l_{\tilde \Farc_t,n,1}}{N}\Big].
\end{align}

For the left endpoint interval, similarly we have $|c_0-s_0|
\le 
|d_0-s_0|+|d_1-s_0|
\le
\frac{4 \sqrt{2}\frac{v_{\max}}{v_{\min}}l_{\tilde \Farc_t,n,1}}{N}$. Hence,
\begin{align*}
c_0\in &
\Big[\frac{1}{\sqrt2}\frac{v_{\min}}{v_{\max}}-\frac{4 \sqrt{2}\frac{v_{\max}}{v_{\min}}l_{\tilde \Farc_t,n,1}}{N}, \sqrt2\frac{v_{\max}}{v_{\min}}+\frac{4 \sqrt{2}\frac{v_{\max}}{v_{\min}}l_{\tilde \Farc_t,n,1}}{N}\Big], \quad \text{and}\\
c_{N-1}\in&
\Big[\frac{1}{\sqrt2}\frac{v_{\min}}{v_{\max}}-\frac{4 \sqrt{2}\frac{v_{\max}}{v_{\min}}l_{\tilde \Farc_t,n,1}}{N}, \sqrt2\frac{v_{\max}}{v_{\min}}+\frac{4 \sqrt{2}\frac{v_{\max}}{v_{\min}}l_{\tilde \Farc_t,n,1}}{N}\Big].
\end{align*}

Combining the above bounds, we have
\begin{align}
d_n, c_n \in\Big[\frac{1}{\sqrt2}\frac{v_{\min}}{v_{\max}}-\frac{4 \sqrt{2}\frac{v_{\max}}{v_{\min}}l_{\tilde \Farc_t,n,1}}{N}, \sqrt2\frac{v_{\max}}{v_{\min}}+\frac{4 \sqrt{2}\frac{v_{\max}}{v_{\min}}l_{\tilde \Farc_t,n,1}}{N}\Big],
\end{align}
for all relevant $n$. Since $\tilde \Farc'_t(w)$ is a convex combination of $d_n,c_n,d_{n+1}$ on each interval,
it follows:
\begin{align}
\frac{1}{\sqrt2}\frac{v_{\min}}{v_{\max}}-\frac{4 \sqrt{2}\frac{v_{\max}}{v_{\min}}l_{\tilde \Farc_t,n,1}}{N}
\le
\tilde \Farc'_t(w)
\le
\sqrt2\frac{v_{\max}}{v_{\min}}+\frac{4 \sqrt{2}\frac{v_{\max}}{v_{\min}}l_{\tilde \Farc_t,n,1}}{N},
\quad \forall w\in[0,1].
\end{align}

Finally, if $N \ge 16 \frac{v_{\max}^2}{v_{\min}^2} \frac{1}{\tilde s^{(t)}(1)}\Bigl(l_{v,0}+\frac12 l_{\fPF,1}\Bigr)$,
then $l_{\tilde \Farc_t,n,1}\le \frac{1}{\tilde s^{(t)}(1)}\Bigl(l_{v,0}+\frac12 l_{\fPF,1}\Bigr)$ implies $\frac{4 \sqrt{2}\frac{v_{\max}}{v_{\min}}l_{\tilde \Farc_t,n,1}}{N}\le \frac{\frac{1}{\sqrt2}\frac{v_{\min}}{v_{\max}}}{2}$. Therefore, it proves claim~\ref{claim:secant}:
\begin{align}
\frac{1}{2\sqrt2}\frac{v_{\min}}{v_{\max}}
\le \tilde \Farc_t'(w)\le \frac{5}{2\sqrt2}\frac{v_{\max}}{v_{\min}},
\quad \forall w\in[0,1].
\end{align}

To prove claim~\ref{claim:smooth}, it remains to bound $\tilde \Farc''_t$ uniformly on each interval.
For $w\in[w_n^{(t)},w_{n+1}^{(t)}]$,
\begin{align}
\tilde \Farc_t''(w)
= 
\frac{1}{w_{n+1}^{(t)}-w_n^{(t)}}
\Big[
(6d_n+6d_{n+1}-12s_n)\theta
+
(6s_n-4d_n-2d_{n+1})
\Big].
\end{align}
Since $\tilde \Farc_t''$ is affine in $\theta$, its maximum absolute value on $[w_n^{(t)},w_{n+1}^{(t)}]$
is attained at an endpoint. Hence
\begin{align}
|\tilde \Farc_t''(w)|
\leq \frac{1}{w_{n+1}^{(t)}-w_n^{(t)}}
\max\Big\{|4(d_n-s_n)+2(d_{n+1}-s_n)|, |2(d_n-s_n)+4(d_{n+1}-s_n)|\Big\}.
\label{eq:pchip_second_derivative_basic_min}
\end{align}

Using $s_n\in\left[\frac{1}{\sqrt2}\frac{v_{\min}}{v_{\max}},\sqrt2\frac{v_{\max}}{v_{\min}}\right]$ and
claim~\ref{claim:grid}, 
the exact PCHIP (harmonic-mean) formula in Algorithm~\ref{alg:PCHIP} yields
\begin{align}
|d_n-s_n|
\le &
\frac{5\frac{v_{\max}^2}{v_{\min}^2}\left(2+5\frac{v_{\max}^2}{v_{\min}^2}\right)}
{\left(2+5\frac{v_{\max}^2}{v_{\min}^2}\right)+\left(1+10\frac{v_{\max}^2}{v_{\min}^2}\right)\frac12\frac{v_{\min}^2}{v_{\max}^2}}
 |s_n-s_{n-1}|,\quad \text{and}\\
|d_{n+1}-s_n|
\le &
\frac{2+5\frac{v_{\max}^2}{v_{\min}^2}}
{\left(1+10\frac{v_{\max}^2}{v_{\min}^2}\right)\frac12\frac{v_{\min}^2}{v_{\max}^2}
+\left(2+5\frac{v_{\max}^2}{v_{\min}^2}\right)}
 |s_{n+1}-s_n|.
\end{align}
By \eqref{eq: secant Lipschitz-smooth},
$
|s_n-s_{n-1}|\le l_{\tilde \Farc_t,n,1}(w_n^{(t)}-w_{n-1}),
\quad
|s_{n+1}-s_n|\le l_{\tilde \Farc_t,n,1}(w_{n+1}^{(t)}-w_n^{(t)}).
$
Substituting these into \eqref{eq:pchip_second_derivative_basic_min}, and again using
claim~\ref{claim:grid} to bound $\frac{w_n^{(t)}-w_{n-1}}{w_{n+1}^{(t)}-w_n^{(t)}}$,
we obtain
\begin{align}
\sup_{w\in[w_n^{(t)},w_{n+1}^{(t)}]}|\tilde \Farc''_t(w)|
\le
C_{\mathrm{pchip}} l_{\tilde \Farc_t,n,1},
\quad n\in\{1,\cdots,N-2\},
\end{align}
where $C_{\mathrm{pchip}} = \Oc\left(\frac{v_{\max}^2}{v_{\min}^2}\right)$.
For the first and last intervals, the endpoint PCHIP formula gives
$
\sup_{w\in[w_0^{(t)},w_1^{(t)}]}|\tilde \Farc''_t(w)|, 
\sup_{w\in[w_{N-1}^{(t)},w_N^{(t)}]}|\tilde \Farc''_t(w)|
\le
6 l_{\tilde \Farc_t,n,1}.
$

Combining the interior and endpoint bounds, $\tilde \Farc_t''$ is uniformly bounded on every
open interval, and since the PCHIP interpolant is $C^1$, it follows that $\tilde \Farc_t'$ is globally Lipschitz-continuous on $[0,1]$. Therefore $\tilde \Farc_t$ is
Lipschitz-smooth on $[0,1]$ with modulus
\begin{align}
l_{\tilde \Farc_t,1}
:=
\max\{6,C_{\mathrm{pchip}}\}\max_{0\le n\le N-2}l_{\tilde \Farc_t,n,1}.
\end{align}
Using
$
l_{\tilde \Farc_t,n,1}\le \frac{1}{\tilde s^{(t)}(1)}\Bigl(l_{v,0}+\frac12 l_{\fPF,1}\Bigr),
$
we may take
\begin{align}
l_{\Farc_{\rm iter},1}
=
\max\{6,C_{\mathrm{pchip}}\}
\frac{1}{\tilde s^{(t)}(1)}\Bigl(l_{v,0}+\frac12 l_{\fPF,1}\Bigr).
\end{align}
This proves claim~\ref{claim:smooth}.
\end{proof}

\begin{algorithm}[t]
\caption{Piecewise Cubic Hermite Interpolating Polynomial~\citep{fritsch1980monotone} for $\{(w_n^{(t)},\tilde s^{(t)}(w_n^{(t)}))\}_{n=0}^N$}
\label{alg:PCHIP}
\begin{algorithmic}[1] 
\Require Samples $\{(w_n^{(t)},\tilde s^{(t)}(w_n^{(t)}))\}_{n=0}^N$ at strictly increasing weights. Hermite basis functions $H$, given by $H_{00}(\theta)=2\theta^3-3\theta^2+1$, $H_{10}(\theta)=\theta^3-2\theta^2+\theta$, $H_{01}(\theta)=-2\theta^3+3\theta^2$, and $H_{11}(\theta)=\theta^3-\theta^2$.

\If{$N=1$}
 \State $m_0, m_1 = \dfrac{\tilde s^{(t)}(w_{1}^{(t)})-\tilde s^{(t)}(w_0^{(t)})}{w_{1}^{(t)}-w_0^{(t)}}$
\Else
 \For{$n=0,\cdots,N-1$}
 \State $\delta_n = \dfrac{\tilde s^{(t)}(w_{n+1}^{(t)})-\tilde s^{(t)}(w_n^{(t)})}{w_{n+1}^{(t)}-w_n^{(t)}}$
 \EndFor
 \For{$n=1,\cdots,N-1$}
 \If{$\delta_{n-1}\delta_n \le 0$}
 \State $m_n = 0$
 \Else
 \State $\omega_1 = 2(w_{n+1}^{(t)}-w_n^{(t)})+(w_n^{(t)}-w_{n-1}^{(t)})$
 \State $\omega_2 = (w_{n+1}^{(t)}-w_n^{(t)})+2(w_n^{(t)}-w_{n-1}^{(t)})$
 \State $m_n = (\omega_1+\omega_2)/\big(\tfrac{\omega_1}{\delta_{n-1}}+\tfrac{\omega_2}{\delta_n}\big)$
 \Comment{weighted harmonic mean}
 \EndIf
 \EndFor

 \State $m_0 = \dfrac{\big(2(w_1^{(t)}-w_0^{(t)})+(w_2^{(t)}-w_1^{(t)})\big)\delta_0-(w_1^{(t)}-w_0^{(t)})\delta_1}{(w_1^{(t)}-w_0^{(t)})+(w_2^{(t)}-w_1^{(t)})}$ \Comment{Endpoint}
 \If{$m_0\delta_0 \le 0$}
 \State $m_0 = 0$
 \ElsIf{$|m_0| > 3|\delta_0|$}
 \State $m_0 = 3\delta_0$
 \EndIf

 \State $m_N = \dfrac{\big(2(w_N^{(t)}-w_{N-1}^{(t)})+(w_{N-1}^{(t)}-w_{N-2}^{(t)})\big)\delta_{N-1}-(w_N^{(t)}-w_{N-1}^{(t)})\delta_{N-2}}{(w_N^{(t)}-w_{N-1}^{(t)})+(w_{N-1}^{(t)}-w_{N-2}^{(t)})}$ \Comment{Endpoint}
 \If{$m_N\delta_{N-1} \le 0$}
 \State $m_N = 0$
 \ElsIf{$|m_N| > 3|\delta_{N-1}|$}
 \State $m_N = 3\delta_{N-1}$
 \EndIf
\EndIf

\For{$n=0,\cdots,N-1$} 
 \State Build the normalized coordinate $\theta = \frac{w-w_n^{(t)}}{w_{n+1}^{(t)}-w_n^{(t)}}$ for all $w\in[w_n^{(t)},w_{n+1}^{(t)}]$ and
 $
 \tilde s^{(t)}(w)
 =
 H_{00}(\theta) \tilde s^{(t)}(w_n^{(t)})
 +
 H_{10}(\theta) (w_{n+1}^{(t)}-w_n^{(t)})m_n
 +
 H_{01}(\theta) \tilde s^{(t)}(w_{n+1}^{(t)})
 +
 H_{11}(\theta) (w_{n+1}^{(t)}-w_n^{(t)})m_{n+1}
 $
\EndFor

\State \Return $\tilde s^{(t)}(\cdot)$
\end{algorithmic}
\end{algorithm}
The previous lemma analyzes a single interpolation step assuming the input CDF already satisfies suitable derivative bounds. To complete the argument, it remains to show that these bounds are preserved throughout the SURF iterations. The next lemma establishes this by induction on the outer-loop index $t$. This is an extension to the statement that $\Farc_t$ is uniformly $l_{\Farc_{\rm iter},1}$-Lipschitz-smooth in Theorem~\ref{thm:cdf}.
\begin{lemma}
\label{lem:bounds_all_t_with_interp}
Suppose Assumptions~\ref{assump:hidden_regular},~\ref{assump:LICQ},~\ref{assump:tangent_gap} hold. Consequently, Algorithm~\ref{alg:SURF} with exact \textsc{InnerSolver} solutions and the interpolation rule in Algorithm~\ref{alg:PCHIP} generates iterations that satisfy
\begin{enumerate}[label=(\roman*), leftmargin=*, align=left]

 \item 
 \label{claim:grid_t}
 the sampled weight $w_{n+1}^{(t)}-w_n^{(t)}\in\left[ \frac{2\sqrt2}{5}\frac{v_{\min}}{v_{\max}}\frac{1}{N}, 2\sqrt2\frac{v_{\max}}{v_{\min}}\frac{1}{N}
 \right]$
 for all $n \in \{0,1,\cdots,N-1\}$;

 \item 
 \label{claim:secant_t}
 $\Farc_t$ is strictly increasing with
 $
 \Farc_t' \in \left[\frac{1}{\sqrt2}\frac{v_{\min}}{v_{\max}}-\frac{2\sqrt2\frac{v_{\max}}{v_{\min}}l_{\Farc_{\rm iter},1}}{N},
 \sqrt2\frac{v_{\max}}{v_{\min}}+\frac{2\sqrt2\frac{v_{\max}}{v_{\min}}l_{\Farc_{\rm iter},1}}{N}\right].
 $
 If $N$ is chosen such that
 $
 N \ge 16 \frac{v_{\max}^2}{v_{\min}^2} \frac{1}{\tilde s^{(t)}(1)}
 \Bigl(l_{v,0}+\frac12 l_{\fPF,1}\Bigr),
 $
 then
 $
 \Farc_t'\in \left[ \frac{1}{2\sqrt2}\frac{v_{\min}}{v_{\max}}, \frac{5}{2\sqrt2}\frac{v_{\max}}{v_{\min}}\right];
 $

 \item 
 \label{claim:smooth_t}
 $\Farc_t$ is Lipschitz-smooth with modulus
 $
 l_{\Farc_{\rm iter},1} = \Oc\left(\frac{v_{\max}^2}{s(1) v_{\min}^2} \Bigl(l_{v,0}+\frac12 l_{\fPF,1}\Bigr)\right).
 $
\end{enumerate}
\end{lemma}

\begin{proof}
We use induction on $t\geq 0$.

For the base case, $\Farc_0(w)=w$. Therefore,
$
w_{n+1}^{(0)}-w_n^{(0)} = \frac{1}{N},
$
so claim~\ref{claim:grid_t} holds. Additionally, $\Farc_0$ is differentiable with
$
\Farc_0'(w)=1,
$
hence $\Farc_0$ satisfies the derivative bounds in claim~\ref{claim:secant_t}. Moreover, $\Farc_0'$ is Lipschitz-continuous with modulus
$
l_{\Farc_0,1}=0\le l_{\Farc_{\rm iter},1},
$
which verifies claim~\ref{claim:smooth_t}. Hence, all claims hold for $t=0$.

Now assume the claims hold for some $t\ge 0$. We prove that they also hold at $t+1$.

By Lemma~\ref{lem:pchip_derivative_bounds}, we know that $\tilde \Farc_t$ satisfies the same derivative and Lipschitz-smoothness properties. The update rule in Algorithm~\ref{alg:SURF},
$
\Farc_{t+1}=(1-\alpha)\Farc_t+\alpha \tilde \Farc_t,
$
gives
\begin{align}
 \Farc_{t+1}'(w)
 = &
 (1-\alpha)\Farc_t'(w)+\alpha \tilde \Farc_t'(w) \nonumber \\
 &\in
 \left[\frac{1}{\sqrt2}\frac{v_{\min}}{v_{\max}}-\frac{2\sqrt2\frac{v_{\max}}{v_{\min}}l_{\Farc_{\rm iter},1}}{N},
 \sqrt2\frac{v_{\max}}{v_{\min}}+\frac{2\sqrt2\frac{v_{\max}}{v_{\min}}l_{\Farc_{\rm iter},1}}{N}\right],
\end{align}
since both $\Farc_t'$ and $\tilde \Farc_t'$ lie in this interval. This proves claim~\ref{claim:secant_t}.

Additionally, $\Farc_t'$ and $\tilde \Farc_t'$ are Lipschitz-continuous with moduli at most $l_{\Farc_{\rm iter},1}$. Therefore, for all $u,w\in[0,1]$,
\begin{align*}
|\Farc_{t+1}'(w)-\Farc_{t+1}'(u)|
\leq & (1-\alpha)|\Farc_t'(w)-\Farc_t'(u)|+\alpha |\tilde \Farc_t'(w)-\tilde \Farc_t'(u)| \\
\leq & \big((1-\alpha)l_{\Farc_{\rm iter},1}+\alpha l_{\Farc_{\rm iter},1}\big)|w-u| = l_{\Farc_{\rm iter},1}|w-u|.
\end{align*}
Thus $\Farc_{t+1}'$ is Lipschitz-continuous with constant at most $l_{\Farc_{\rm iter},1}$, proving claim~\ref{claim:smooth_t}.

Next, since $\Farc_{t+1}$ satisfies the derivative bounds in claim~\ref{claim:secant_t}, Lemma~\ref{lem:pchip_derivative_bounds}(\ref{claim:grid}) directly gives
$
w_{n+1}^{(t+1)}-w_n^{(t+1)}
\in
\left[ \frac{2\sqrt2}{5}\frac{v_{\min}}{v_{\max}}\frac{1}{N}, 2\sqrt2\frac{v_{\max}}{v_{\min}}\frac{1}{N}
\right],
$
which proves claim~\ref{claim:grid_t}. The induction is complete.
\end{proof}

\paragraph{Discretization Error under Exact \textsc{InnerSolver} Solutions.}

The Lipschitz-smoothness results established earlier are essential for applying Landau inequality~\citep{hardy1952inequalities}, summarized in Lemma~\ref{lem:landau_inequality}.

\begin{lemma}[{Landau inequality~\citep{hardy1952inequalities}}]
\label{lem:landau_inequality}
Suppose $\Farc,\Farc_t:[0,1]\to\mathbb{R}$ are respectively $l_{\Farc,1}$- and $l_{\Farc_{\rm iter},1}$-Lipschitz-smooth on $[0,1]$. Then,
\begin{align}
\|\Farc'-\Farc_t'\|_\infty
\le
2\sqrt{2}\sqrt{(l_{\Farc,1}+l_{\Farc_{\rm iter},1})\|\Farc-\Farc_t\|_\infty}.
\label{eq:derivative_from_function_difference}
\end{align}
Moreover, if $\Farc,\Farc_t$ are both second-order Lipschitz-continuous, then $\|\tilde\Farc_t''-\Farc''\|_\infty = \Oc(\|\tilde\Farc_t-\Farc\|_\infty^{1/3})$.
\end{lemma}

This tool allows us to convert function-value error in the CDF approximation into derivative error, which is the key step in controlling the discretization error induced by rebuilding the polyline approximation $\tilde s^{(t)}$ from finitely many sampled PF points. We first bound the uniform discrepancy between $s$ and its rebuilt approximation $\tilde s$ in terms of the sampled polyline error $\max_n (s(w_n)-\tilde s(w_n))$.
In this way, we can proceed to the proof of the convergence of SURF using exact-\textsc{InnerSolver} solutions in Theorem~
\ref{thm:cdf}.

\begin{lemma}[Discretization error]
\label{lem:s_tilde_combined_bound}
Suppose Assumptions~\ref{assump:hidden_regular},~\ref{assump:LICQ},
\ref{assump:tangent_gap} hold. Let $\kappa := 4\sqrt2\frac{v_{\max}}{v_{\min}} \sqrt{l_{\Farc_{\rm iter},1}+l_{\Farc,1}}$. Then
\begin{align}
\|s-\tilde s^{(t)}\|_\infty\le&\frac{\tilde s^{(t)}(1) (l_{\tilde\Farc_t\circ\Farc_t^{-1},1}+ l_{\Farc\circ\Farc_t^{-1},1})}{8} \frac{1}{N^2} + \max_{0\leq n\leq N} (s(w_n^{(t)})-\tilde s (w_n^{(t)})).
\label{eq:s_tilde_combined_bound}
\end{align}
where $l_{\tilde\Farc_t\circ\Farc_t^{-1},1}\le \left(\frac{s(1)}{v_{\min}}\right)^2\|\tilde\Farc_t''-\Farc_t''\|_\infty+2\sqrt2 l_{\Farc_{\rm iter},1}\left(\frac{s(1)}{v_{\min}}\right)^3\sqrt{2 l_{\Farc_{\rm iter},1}\|\tilde\Farc_t-\Farc_t\|_\infty}$, and $l_{\Farc\circ\Farc_t^{-1},1}\le \left(\frac{s(1)}{v_{\min}}\right)^2\|\Farc_t''-\Farc''\|_\infty+2\sqrt2 l_{\Farc_{\rm iter},1}\left(\frac{s(1)}{v_{\min}}\right)^3\sqrt{(l_{\Farc_{\rm iter},1}+l_{\Farc,1})\|\Farc_t-\Farc\|_\infty}$.
Define $e_t(q):=\Farc(\Farc_t^{-1}(q))-q$. Then
\begin{align}
& \max_{0\le n\le N} (s(w_n^{(t)})-\tilde s (w_n^{(t)}))
\label{eq:nodal_chord_arc_bound_min}\\
\leq & 
\tfrac{l_{\tilde{\f}_{\rm PF},1}^2s(1)^3}{6}
\left(
\frac1{N^2}
+
3 \max_{0\le n\le N-1}
\left|e_t\left(\tfrac{n+1}{N}\right)-e_t\left(\tfrac{n}{N}\right)\right|^2 
+
N \max_{0\le n\le N-1}
\left|e_t\left(\tfrac{n+1}{N}\right)-e_t\left(\tfrac{n}{N}\right)\right|^3
\right).\nonumber
\end{align}
Moreover, $\|\tilde \Farc_t-\Farc\|_\infty \le \frac{2\sqrt2}{s(1)}\|\tilde s^{(t)}-s\|_\infty$ and $\max_{0\le n\le N-1}
\left|
e_t\left(\tfrac{n+1}{N}\right)
-
e_t\left(\tfrac nN\right)
\right|
\le \frac{\kappa}{N}\|\Farc_t-\Farc\|_\infty^{1/2}$.
\end{lemma}

\begin{proof}
Fix $w\in[0,1]$ and choose $n\in\{0,\dots,N-1\}$ such that
$w\in[w_n^{(t)},w_{n+1}^{(t)}]$. 
Let $q=\Farc_t(w)$ and $q_n=\Farc_t(w_n^{(t)})$. Since the knots are generated by
$w_n^{(t)}=\Farc_t^{-1}(n/N)$, we have $q_n=n/N$. Let $I_n^q$ denote the linear
interpolant in the $q$-coordinate, pulled back to the $w$-coordinate. Then
\begin{align}
|s(w)-\tilde s^{(t)}(w)|
\le &
|s(w)-I_n^q s(w)|
+
|I_n^q s(w)-I_n^q\tilde s^{(t)}(w)|
+
|I_n^q\tilde s^{(t)}(w)-\tilde s^{(t)}(w)|
\label{eq:pchip_interp_decomp} \\
\le &
|s(w)-I_n^q s(w)|
+
|I_n^q\tilde s^{(t)}(w)-\tilde s^{(t)}(w)|
+
\max_{0\le n\le N}|s(w_n^{(t)})-\tilde s^{(t)}(w_n^{(t)})|. \nonumber
\end{align}
Here, the second inequality follows because
$I_n^q s(w)-I_n^q\tilde s^{(t)}(w)$ is a convex combination of the endpoint
errors.
Since $s(w)=s(1)\Farc(w)$, we have
$s(\Farc_t^{-1}(q))=s(1)\Farc(\Farc_t^{-1}(q))$. Hence, the standard linear
interpolation analysis~\citep{suli2003introduction} gives
\begin{align}
 |s(w)-I_n^q s(w)| \le \frac{s(1)l_{\Farc\circ\Farc_t^{-1},1}}{8N^2}, \quad\text{and}\quad |\tilde s^{(t)}(w)-I_n^q\tilde s^{(t)}(w)| \le \frac{s(1)l_{\tilde\Farc_t\circ\Farc_t^{-1},1}}{8N^2}. 
\end{align}
In this way, $\frac{\partial}{\partial q}\tilde\Farc_t(\Farc_t^{-1}(q)) = \frac{\tilde\Farc_t'(\Farc_t^{-1}(q))}{\Farc_t'(\Farc_t^{-1}(q))}$ and $\frac{\partial}{\partial q}\Farc(\Farc_t^{-1}(q)) = \frac{\Farc'(\Farc_t^{-1}(q))}{\Farc_t'(\Farc_t^{-1}(q))}$ gives
\begin{align*}
 l_{\tilde\Farc_t\circ\Farc_t^{-1},1}\le & \left(\frac{s(1)}{v_{\min}}\right)^2\|\tilde\Farc_t''-\Farc_t''\|_\infty+2\sqrt2 l_{\Farc_{\rm iter},1}\left(\frac{s(1)}{v_{\min}}\right)^3\sqrt{2 l_{\Farc_{\rm iter},1}\|\tilde\Farc_t-\Farc_t\|_\infty},\quad \text{and}\\
 l_{\Farc\circ\Farc_t^{-1},1}\le & \left(\frac{s(1)}{v_{\min}}\right)^2\|\Farc_t''-\Farc''\|_\infty+2\sqrt2 l_{\Farc_{\rm iter},1}\left(\frac{s(1)}{v_{\min}}\right)^3\sqrt{(l_{\Farc_{\rm iter},1}+l_{\Farc,1})\|\Farc_t-\Farc\|_\infty}.
\end{align*}
By uniform Lipschitz-smoothness of $\tilde \Farc_t$, $\Farc_t$, and $\Farc$ (see~ Lemma~\ref{lem:fPFdifferentiable},~\ref{lem:pchip_derivative_bounds},~\ref{lem:bounds_all_t_with_interp}), we have $\| \tilde\Farc_t''-\Farc_t''\|_\infty \leq 2l_{\Farc_{\rm iter},1}$ and $\| \Farc_t''-\Farc''\|_\infty \leq l_{\Farc_{\rm iter},1}+l_{\Farc,1}$.
Plugging these in \eqref{eq:pchip_interp_decomp} proves the first inequality.

Fix any $n\in\{0,\dots,N-1\}$ and write
$a=s(w_n^{(t)})$, $b=s(w_{n+1}^{(t)})$.
Then $\tilde{\f}_{\mathrm{PF}}(b)-\tilde{\f}_{\mathrm{PF}}(a)=\int_a^b \tilde{\f}_{\mathrm{PF}}'(r) dr$.
Using $\|\tilde{\f}_{\mathrm{PF}}'(a)\|=1$, we obtain
\begin{align}
\|\tilde{\f}_{\mathrm{PF}}(b)-\tilde{\f}_{\mathrm{PF}}(a)\|
\ge&
\left\langle
\tilde{\f}_{\mathrm{PF}}(b)-\tilde{\f}_{\mathrm{PF}}(a),
\tilde{\f}_{\mathrm{PF}}'(a)
\right\rangle
=
\int_a^b
\left\langle
\tilde{\f}_{\mathrm{PF}}'(r),
\tilde{\f}_{\mathrm{PF}}'(a)
\right\rangle dr\nonumber\\
= &
\int_a^b
\left(
1-\frac12
\|\tilde{\f}_{\mathrm{PF}}'(r)-\tilde{\f}_{\mathrm{PF}}'(a)\|^2
\right)dr
\geq
(b-a)-\frac{l_{\tilde{\f}_{\mathrm{PF}},1}^2}{6}(b-a)^3.
\label{eq:inline_chord_arc_upper}
\end{align}
Also, $\|\tilde{\f}_{\mathrm{PF}}(b)-\tilde{\f}_{\mathrm{PF}}(a)\| \le \int_a^b \|\tilde{\f}_{\mathrm{PF}}'(r)\|dr = b-a$.
Combining this with \eqref{eq:inline_chord_arc_upper} 
\begin{align}
0
\le&
\bigl(s(w_{n+1}^{(t)})-s(w_n^{(t)})\bigr)
-
\left\|
\tilde{\f}_{\mathrm{PF}}(s(w_{n+1}^{(t)}))
-
\tilde{\f}_{\mathrm{PF}}(s(w_n^{(t)}))
\right\| \le 
\frac{l_{\tilde{\f}_{\mathrm{PF}},1}^2}{6}
\bigl(s(w_{n+1}^{(t)})-s(w_n^{(t)})\bigr)^3.
\label{eq:inline_segment_gap}
\end{align}

By definition of $\tilde s^{(t)}$ in \eqref{eq:sw_tilde}, each summand is nonnegative by \eqref{eq:inline_segment_gap}. Hence
$
0\le s(w_0^{(t)})-\tilde s^{(t)}(w_0^{(t)})\le \cdots \le s(w_N^{(t)})-\tilde s^{(t)}(w_N^{(t)}),
$
so it suffices to bound the terminal error. Summing
\eqref{eq:inline_segment_gap} over $n=0,\dots,N-1$ gives
\begin{align}
0\le
s(w_N^{(t)})-\tilde s^{(t)}(w_N^{(t)})
\le
\frac{l_{\tilde{\f}_{\mathrm{PF}},1}^2}{6}
\sum_{n=0}^{N-1}
\bigl(s(w_{n+1}^{(t)})-s(w_n^{(t)})\bigr)^3.
\label{eq:polyline_terminal_basic_inline}
\end{align}

Moreover, by definition of $e_t(q)=\Farc(\Farc_t^{-1}(q))-q$, and since
$w_n^{(t)}=\Farc_t^{-1}(n/N)$,
\begin{align}
s(w_{n+1}^{(t)})-s(w_n^{(t)})
= &
s(1)\left(\Farc(w_{n+1}^{(t)})-\Farc(w_n^{(t)})\right)
= 
s(1)\left(
\frac1N
+
e_t\left(\frac{n+1}{N}\right)
-
e_t\left(\frac{n}{N}\right)
\right).
\label{eq:sw_increment_et_inline}
\end{align}
Substituting \eqref{eq:sw_increment_et_inline} into
\eqref{eq:polyline_terminal_basic_inline} and expanding the cube gives
\begin{align}
s(w_N^{(t)})-\tilde s^{(t)}(w_N^{(t)})
\le&
\frac{l_{\tilde{\f}_{\mathrm{PF}},1}^2 s(1)^3}{6}
\Bigg(
\frac1{N^2}
+
3 \max_{0\le n\le N-1}
\left|e_t\left(\tfrac{n+1}{N}\right)-e_t\left(\tfrac{n}{N}\right)\right|^2
\nonumber\\
&
+
N \max_{0\le n\le N-1}
\left|e_t\left(\tfrac{n+1}{N}\right)-e_t\left(\tfrac{n}{N}\right)\right|^3
\Bigg),
\label{eq:polyline_after_cancel_inline}
\end{align}
where we used
$
\sum_{n=0}^{N-1}
\left(
e_t\left(\tfrac{n+1}{N}\right)-e_t\left(\tfrac{n}{N}\right)
\right)
=
e_t(1)-e_t(0)=0.
$
Together with monotonicity of the prefix errors, this proves
\eqref{eq:nodal_chord_arc_bound_min}. It remains to bound $\|\tilde \Farc_t-\Farc\|_\infty$.
By definition,
\begin{align}
\|\tilde \Farc_t-\Farc\|_\infty = & \left\| \frac{1}{\tilde s^{(t)}(1)} \tilde s^{(t)} - \frac{1}{s(1)} s \right\|_\infty \notag\\
\leq & \left\| \frac{1}{\tilde s^{(t)}(1)}\bigl(\tilde s^{(t)}-s\bigr) \right\|_\infty + \left| \frac{1}{\tilde s^{(t)}(1)}-\frac{1}{s(1)} \right| \|s\|_\infty. \label{eq:Ftilde_split}
\end{align}
Since $\|s\|_\infty=s(1)$ and $\left|
\frac{1}{\tilde s^{(t)}(1)}-\frac{1}{s(1)} \right| = \frac{|s(1)-\tilde s^{(t)}(1)|}{\tilde s^{(t)}(1)s(1)}$, we obtain
\begin{align}
\|\tilde \Farc_t-\Farc\|_\infty
&\le
\frac{\|\tilde s^{(t)}-s\|_\infty}{\tilde s^{(t)}(1)}
+
\frac{|s(1)-\tilde s^{(t)}(1)|}{\tilde s^{(t)}(1)}
\le
\frac{2\|\tilde s^{(t)}-s\|_\infty}{\tilde s^{(t)}(1)},
\label{eq:Ftilde_by_sgap_1}
\end{align}
where the last inequality uses $|s(1)-\tilde s^{(t)}(1)|\le\|\tilde s^{(t)}-s\|_\infty$. By Lemma~\ref{lemma:arc_chord_sqrt2_pf}, there is $\tilde s^{(t)}(1)\ge \frac{1}{\sqrt2}s(1)$. Substituting into \eqref{eq:Ftilde_by_sgap_1} gives $\|\tilde \Farc_t-\Farc\|_\infty \le \frac{2\sqrt2}{s(1)}\|\tilde s^{(t)}-s\|_\infty$. Finally, for each $n$,
\begin{align}
\left|e_t\left(\tfrac{n+1}{N}\right)-e_t\left(\tfrac{n}{N}\right)\right|
= &
\left|
\Farc(w_{n+1}^{(t)})-\Farc(w_n^{(t)})-\frac1N
\right|
\nonumber\\
= &
\left|
\int_{n/N}^{(n+1)/N}
\frac{
\Farc'(\Farc_t^{-1}(q))
-
\Farc_t'(\Farc_t^{-1}(q))
}{
\Farc_t'(\Farc_t^{-1}(q))
}
dq
\right|.
\label{eq:et_diff_integral_inline}
\end{align}
By Lemma~\ref{lem:bounds_all_t_with_interp},
$
\Farc_t'(w)\ge \frac{1}{2\sqrt2}\frac{v_{\min}}{v_{\max}}$, $\forall w\in[0,1].
$
Therefore,
$
\left|e_t\left(\tfrac{n+1}{N}\right)-e_t\left(\tfrac{n}{N}\right)\right|
\le
\frac{1}{N}
\frac{2\sqrt2 v_{\max}}{v_{\min}}
\|\Farc'-\Farc_t'\|_\infty.
$
By Lemma~\ref{lem:landau_inequality},
$
\|\Farc'-\Farc_t'\|_\infty
\le
2\sqrt2
\sqrt{
(l_{\Farc_{\rm iter},1}+l_{\Farc,1})
\|\Farc_t-\Farc\|_\infty
}.
$
This finishes the proof.
\end{proof}
\paragraph{Convergence of Algorithm~\ref{alg:SURF} with Exact \textsc{InnerSolver} Solutions.} The earlier analysis therefore directly renders the convergence of SURF with exact \textsc{InnerSolver} solutions, which is a special case of Theorem~\ref{thm:cdf}.
\label{appendix:cdf_exact}

\begin{theorem}
\label{thm:CDF_exact_inner_solver}
Suppose \textsc{InnerSolver} is exact. Suppose Assumptions~\ref{assump:hidden_regular},~\ref{assump:LICQ},~\ref{assump:tangent_gap} hold. 
Denote $\kappa := 4\sqrt2\frac{v_{\max}}{v_{\min}} \sqrt{l_{\Farc_{\rm iter},1}+l_{\Farc,1}}$. 
Let $\alpha\in(0,1]$ and choose $N \ge 2 \sqrt{\frac{ \sqrt2 l_{\tilde{\f}_{\mathrm{PF}},1}^2 s(1)^2}{3} \left(3\kappa^2+\kappa^3\right)} = \Theta(\kappa^{3/2})$.
Then Algorithm~\ref{alg:SURF} satisfies the  convergence bound for all $t\ge 0$,
\begin{align}
    \|\Farc_t-\Farc\|_\infty \leq  & \left(1-\frac{\alpha}{2}\right)^t\|\Farc_0-\Farc\|_\infty +\Oc\left(\frac{1}{N^2}\right). \label{eq:exact_convergence}
\end{align}
\end{theorem}
\begin{remark}
 In this case, we can choose $\alpha = 1$ to maximize the convergence.
\end{remark}

\begin{proof}%
By the $\Farc_{t+1}$ update rule and the triangle inequality,
\begin{align}
\|\Farc_{t+1}-\Farc\|_\infty
\le
(1-\alpha)\|\Farc_t-\Farc\|_\infty
+
\alpha\|\tilde\Farc_t-\Farc\|_\infty .
\label{eq:outer_pre_basin}
\end{align}
By Lemma~\ref{lem:s_tilde_combined_bound}, there is
\begin{align} 
 & \|\tilde \Farc_t-\Farc\|_\infty \nonumber \\
 \le & \frac{\sqrt2}{4N^2}\left(l_{\tilde\Farc_t\circ\Farc_t^{-1},1}+l_{\Farc\circ\Farc_t^{-1},1}\right) + \frac{\sqrt2 l_{\tilde{\f}_{\mathrm{PF}},1}^2 s(1)^2}{3 N^2}\bigg( 1 + 3 \kappa^2 \|\Farc_t-\Farc\|_\infty + \kappa^3 \|\Farc_t-\Farc\|_\infty^{3/2} \bigg) \nonumber \\
 \le& \frac{\sqrt2 l_{\tilde{\f}_{\mathrm{PF}},1}^2 s(1)^2}{3N^2} \left(3\kappa^2+\kappa^3\right)\|\Farc_t-\Farc\|_\infty 
 +\frac{\frac{\sqrt{2}}{4} \left(l_{\tilde\Farc_t\circ\Farc_t^{-1},1}+l_{\Farc\circ\Farc_t^{-1},1}\right)+ \frac{\sqrt{2}}{3}l_{\tilde{\f}_{\mathrm{PF}},1}^2 s(1)^2}{N^2} \nonumber\\
 \leq & \frac{1}{4}\|\Farc_t-\Farc\|_\infty 
 +\frac{\frac{\sqrt{2}}{4}\left(l_{\tilde\Farc_t\circ\Farc_t^{-1},1}+l_{\Farc\circ\Farc_t^{-1},1}\right) + \frac{\sqrt{2}}{3}l_{\tilde{\f}_{\mathrm{PF}},1}^2 s(1)^2}{N^2} \nonumber\\
 \leq & \frac{1}{2}\|\Farc_t-\Farc\|_\infty 
 + \frac{\frac{\sqrt{2}}{4}\left(\frac{s(1)}{v_{\min}}\right)^2 \left(\|\tilde\Farc_t''-\Farc_t''\|_\infty+\|\Farc_t''-\Farc''\|_\infty\right) +\frac{\sqrt{2}}{3} l_{\tilde{\f}_{\mathrm{PF}},1}^2 s(1)^2}{N^2} + \Oc\left(\frac{1}{N^4}\right) \nonumber
\end{align}
where the second inequality holds because $\Farc_t$ and $\Farc$ are CDFs and $\|\Farc_t-\Farc\|_\infty\le1$, the third inequality follows $\frac{\sqrt2 l_{\tilde{\f}_{\mathrm{PF}},1}^2 s(1)^2}{3 N^2} \left( \kappa^2+\frac{\sqrt2}{3}\kappa^3 \right) \le \frac14$ from the $N = 2\sqrt{\frac{\sqrt2 l_{\tilde{\f}_{\mathrm{PF}},1}^2 s(1)^2}{3 N^2} \left( \kappa^2+\frac{\sqrt2}{3}\kappa^3 \right)} =  \Omega(\kappa^{1.5})$ choice, and the fourth is from the upper bound for $l_{\tilde\Farc_t\circ\Farc^{-1},1}$ in Lemma~\ref{lem:s_tilde_combined_bound} and Young's inequality for the term $\frac{2\sqrt2 l_{\Farc,1} \left(\frac{s(1)}{v_{\min}}\right)^3 \sqrt{ (l_{\Farc_{\rm iter},1}+l_{\Farc,1}) \|\tilde\Farc_t-\Farc\|_\infty}}{N^2}$ in $l_{\Farc\circ \Farc_t}$ and its counterpart.
Substituting this into \eqref{eq:outer_pre_basin} yields
\begin{align}
 \|\Farc_{t+1}-\Farc\|_\infty \le& \left(1-\frac{\alpha}{2}\right) \|\Farc_t-\Farc\|_\infty +\alpha \Oc\left(\frac{1}{N^4}\right) \nonumber\\
 & + \alpha \frac{\frac{\sqrt{2}}{4}\left(\frac{s(1)}{v_{\min}}\right)^2 \left(\|\tilde\Farc_t''-\Farc_t''\|_\infty+\|\Farc_t''-\Farc''\|_\infty\right) +\frac{\sqrt{2}}{3} l_{\tilde{\f}_{\mathrm{PF}},1}^2 s(1)^2}{N^2}.
\label{eq:delta_recursion_exact}
\end{align}
Here, $\left(\|\tilde\Farc_t''-\Farc_t''\|_\infty+\|\Farc_t''-\Farc''\|_\infty\right) \leq 3l_{\Farc_{\rm iter},1}+l_{\Farc,1}$ by the Lipschitz-smoothness of $\Phi_t$ and $\Phi$ (see Lemma~\ref{lem:fPFdifferentiable},~\ref{lem:pchip_derivative_bounds},~\ref{lem:bounds_all_t_with_interp}). Thus, unrolling gives its convergence to the floor of  $\Oc(1/N^2)$:
\begin{align}
 & \|\Farc_t-\Farc\|_\infty \nonumber \\
 \le&
 \left(1-\frac{\alpha}{2}\right)^t
 \|\Farc_0-\Farc\|_\infty+ \frac{\frac{\sqrt{2}}{2}\left(\frac{s(1)}{v_{\min}}\right)^2 (3l_{\Farc_{\rm iter},1}+l_{\Farc,1})+\frac{2\sqrt{2}}{3} l_{\tilde{\f}_{\mathrm{PF}},1}^2 s(1)^2}{N^2}+\Oc\left(\frac{1}{N^4}\right) .
 \label{eq:delta_unrolled_exact}
\end{align}
Moreover, if $\Phi_t$ and $\Phi$ are 2nd-order Lipschitz-continuous, Landau inequality~\citep{hardy1952inequalities} gives $\|\tilde\Farc_t''-\Farc''\|_\infty = \Oc(\|\tilde\Farc_t-\Farc\|_\infty^{1/3})$ and we can similarly obtain that $\|\Farc_t-\Farc\|_\infty $ converges linearly to $\frac{\sqrt{2}l_{\tilde{\f}_{\mathrm{PF}},1}^2 s(1)^2}{3N^2} + \Oc\left(\frac{1}{N^4}\right)$.
\end{proof}
The exact \textsc{InnerSolver} analysis characterizes the error when each scalarized subproblem is solved exactly, leaving only the finite-sampling discretization floor. We next extend the analysis to the practical setting where the \textsc{InnerSolver} is run for finitely many steps and the returned scalarized solutions are inexact.

\subsection{Proof of the Second Part of Theorem~\ref{thm:cdf}: Optimization Error Control under Inexact \textsc{InnerSolver} Solutions}
\label{appendix:cdf_inexact}

In this section, we prove the convergence of SURF with inexact \textsc{InnerSolver} in Theorem~\ref{thm:cdf}. In this case, $\tilde s$ and $\tilde \Farc$ are built via inexact estimation $\bu_{n}^{(t)}$ \textcolor{gray}{or $\bx_{n}^{(t)}$}.
In contrast, we use $\hat s_t$ to denote the unknown estimate that can be built up by the result from exact \textsc{InnerSolver} solutions. That is, for $n=0,\cdots,N-1$, 
\begin{align}
\hat s^{(t)}(w_0^{(t)})&:=0,\quad
\hat s^{(t)}(w_{n+1}^{(t)}):=
\hat s^{(t)}(w_n^{(t)})
+
\left\|
\h\left(\bu_{w_{n+1}^{(t)}}^*\right)-\h\left(\bu_{w_{n}^{(t)}}^*\right)
\right\|
\label{eq:sw_tilde_exact}
\end{align}
\textcolor{gray}{or equivalently, $\hat s^{(t)}(w_{n+1}^{(t)}):=
\hat s^{(t)}(w_n^{(t)})
+
\left\|
\f\left(\bx_{w_{n+1}^{(t)}}^*\right)-\f\left(\bx_{w_{n}^{(t)}}^*\right)
\right\|$}, then build interpolation for $w\in[0,1]$. Consequently, we denote $\hat \Farc(w)= \hat s(w)/\hat s(1)$.

Algorithm~\ref{alg:SURF} may be implemented either in the original space $\U$, returning iterates $\bu_n^{(t)}$, or directly in the hidden space \textcolor{gray}{$\X$, returning iterates $\bx_n^{(t)}$} when \textcolor{gray}{$g$ and $f_m$} are explicitly available. Under Assumption~\ref{assump:hidden_regular}, these two viewpoints are equivalent, and standard \textsc{InnerSolver} such as gradient-based methods can be used efficiently in either formulation~\citep{fatkhullin2025stochastic}.
In this section, we analyze the algorithm in the \textcolor{gray}{convex hidden space $\X$}. Accordingly, when the \textsc{InnerSolver} returns iterates $\bu_n^{(t)}$, we define their \textcolor{gray}{hidden-space images by $\bx_n^{(t)}:=g(\bu_n^{(t)})$ and $\bx_w^*:=g(\bu_w^*)$}.
If the algorithm is formulated directly in \textcolor{gray}{$\X$}, then \textcolor{gray}{$\bx_n^{(t)}$} simply denotes the returned iterate.
Under Assumption~\ref{assump:hidden_regular}, since $h_m(\bu)=f_m(g(\bu))$, we equivalently have \textcolor{gray}{$\mathbf h(\bu_n^{(t)})=\f(\bx_n^{(t)})$, $\mathbf h(\bu_w^*)=\f(\bx_w^*)$.}
Therefore, all approximation and chord-length bounds below are stated in the \textcolor{gray}{$\bx$-space}.

We first prepare the following lemma that quantifies the optimization error and establishes the uniform smoothness claim for $\tilde\Farc_t$ in Theorem~\ref{thm:cdf}.

\begin{lemma}[Inexact \textsc{InnerSolver} perturbation and Lipschitz-smoothness]
\label{lem:exact_inexact_polyline_sampled_weights_gap}
Suppose Assumptions~\ref{assump:hidden_regular},~\ref{assump:LICQ},~\ref{assump:tangent_gap} hold. Consider Algorithm~\ref{alg:SURF} with inexact \textsc{InnerSolver} solutions that return iterates $\bu_n^{(t)}$, and define their convex-space images by $\bx_n^{(t)}:=g(\bu_n^{(t)})$ \textcolor{gray}{or simply let $\bx_n^{(t)}$ denote the returned iterate when the algorithm is formulated directly in the convex space $\X$}. Assume $\max_n\|\bx_n^{(t)}-\bx_{w_n^{(t)}}^*\|\leq \epsilon$ for all $t\geq 0$. Then,
\begin{align}
\big|\tilde s_t(w_n^{(t)})-\hat s_t(w_n^{(t)})\big| = \Oc \left(N\epsilon\right), \label{eq:sampled_weights_gap_s_exact_inexact}
\end{align}
where $\tilde s_t(w_n^{(t)})$ and $\hat s_t(w_n^{(t)})$ are generated via the inexact construction \eqref{eq:sw_tilde} and the exact construction \eqref{eq:sw_tilde_exact}, respectively. 
Additionally, suppose $\epsilon= \Oc\left(\frac{v_{\min}^2}{v_{\max}}N^{-1}\right)$. Then, $\tilde\Farc_t$ and $\Farc_t$ are uniformly $l_{\Farc_{\rm iter},1}$-Lipschitz-smooth and
\begin{align}
\|\tilde \Farc_t-\hat \Farc_t\|_\infty \le & \Oc\left( \frac{1}{s(1)} N \epsilon \right).
\label{eq:sup_gap_F_exact_inexact_final}
\end{align}
\end{lemma}

\begin{proof}
Fix $t$ and $n$. Under Assumptions~\ref{assump:hidden_regular} and~\ref{assump:LICQ}, the exact solution map
$w\mapsto \bx_w^*$ is continuous on $[0,1]$. Therefore $\XPF:=\{\bx_w^*:w\in[0,1]\}$ is compact. Its $\epsilon$-tube $\XPF_{\epsilon}:=\{\bx:\dist(\bx,\XPF)\le \epsilon\}$ is also compact. Thus, for each $m\in\{1,2\}$, $G_m:=\sup_{\bx\in \XPF_{\epsilon}}\|\nabla f_m(\bx)\| < \infty.$ Let $G_f:=\sqrt{G_1^2+G_2^2}$. By the mean value theorem and $\max_n\|\bx_n^{(t)}-\bx_{w_n^{(t)}}^*\|\leq \epsilon$,
\begin{align}
\|\f(\bx_n^{(t)})-\fPF(w_n^{(t)})\| = & \|\f(\bx_n^{(t)})-\f(\bx_{w_n^{(t)}}^*)\| \le G_f \|\bx_n^{(t)}-\bx_{w_n^{(t)}}^*\| \le G_f\epsilon.
\label{eq:f_gap_from_dist_gap}
\end{align}

Now apply the reverse triangle inequality to each segment:
\begin{align}
&\Big| \|\f(\bx_{n+1}^{(t)})-\f(\bx_n^{(t)})\| - \|\fPF(w_{n+1}^{(t)})-\fPF(w_n^{(t)})\| \Big| \nonumber\\
\le& \|\f(\bx_{n+1}^{(t)})-\fPF(w_{n+1}^{(t)})\| + \|\f(\bx_n^{(t)})-\fPF(w_n^{(t)})\| \nonumber\\
\le& 2G_f\epsilon = \Oc(\epsilon). \label{eq:segment_gap_eps_merged}
\end{align}
Summing \eqref{eq:segment_gap_eps_merged} over the first $m$ segments yields
$\big|\tilde s_t(w_m^{(t)})-\hat s_t(w_m^{(t)})\big| \le m\cdot \Oc(\epsilon) = \Oc(N\epsilon)$ for all $m=0,\dots,N$, proving \eqref{eq:sampled_weights_gap_s_exact_inexact}.

We view Algorithm~\ref{alg:PCHIP} as applied on the uniform coordinate $\Farc_t(w)$, i.e., on the knots $\{\Farc_t(w_n^{(t)})\}_{n=0}^N=\{n/N\}_{n=0}^N$, and then pulled back to $w$ through $\Farc_t$. Since $\hat s_t$ is built from exact sampled cumulative chord lengths, Lemma~\ref{lem:discrete_growth_control_tildeF_biobj} and the derivative bounds on $\Farc_t$ give
\begin{align}
\frac{2}{5}\frac{v_{\min}^2}{v_{\max}} \le \frac{\hat s_t(w_{n+1}^{(t)})-\hat s_t(w_n^{(t)})}{1/N} \le 2\sqrt2\frac{v_{\max}^2}{v_{\min}}, \quad n=0,\dots,N-1.
\label{eq:hatdelta_bounds_localstab}
\end{align}
For every $n=0,\dots,N-1$, by the reverse triangle inequality and \eqref{eq:f_gap_from_dist_gap},
\begin{align}
&\left|\frac{\tilde s_t(w_{n+1}^{(t)})-\tilde s_t(w_n^{(t)})}{1/N}-\frac{\hat s_t(w_{n+1}^{(t)})-\hat s_t(w_n^{(t)})}{1/N}\right| \nonumber \\
\leq& N\Bigl(\|\f(\bx_{n+1}^{(t)})-\fPF(w_{n+1}^{(t)})\|+\|\f(\bx_n^{(t)})-\fPF(w_n^{(t)})\|\Bigr) = \Oc(N\epsilon).
\label{eq:delta_gap_direct}
\end{align}
Choosing some $\epsilon=\Oc\left(\frac{v_{\min}^2}{v_{\max}}N^{-1}\right)$ gives
$\left|\frac{\tilde s_t(w_{n+1}^{(t)})-\tilde s_t(w_n^{(t)})}{1/N}-\frac{\hat s_t(w_{n+1}^{(t)})-\hat s_t(w_n^{(t)})}{1/N}\right| \le \frac15\frac{v_{\min}^2}{v_{\max}}$
so combining this with \eqref{eq:hatdelta_bounds_localstab} yields
\begin{align}
\frac{1}{5}\frac{v_{\min}^2}{v_{\max}} \le \frac{\tilde s_t(w_{n+1}^{(t)})-\tilde s_t(w_n^{(t)})}{1/N} \le 3\sqrt2\frac{v_{\max}^2}{v_{\min}}, \quad n=0,\dots,N-1.
\label{eq:tildedelta_bounds_localstab}
\end{align}
Thus the exact and inexact nodal data both stay in the same strictly increasing PCHIP regime. Let $\hat m_n$ and $\tilde m_n$ be the PCHIP slopes generated by Algorithm~\ref{alg:PCHIP} from the exact and inexact nodal data in the coordinate $\Farc_t(w)$. On the positive-slope branch, the interior PCHIP slope is a smooth function of neighboring positive segment slopes, and the endpoint formulas with monotonicity clipping are Lipschitz. Hence there exists $C_m=\Oc(1)$ such that
\begin{align}
\max_{0\le n\le N}|\tilde m_n-\hat m_n| \le C_m\max_{0\le n\le N-1}\left|\frac{\tilde s_t(w_{n+1}^{(t)})-\tilde s_t(w_n^{(t)})}{1/N}-\frac{\hat s_t(w_{n+1}^{(t)})-\hat s_t(w_n^{(t)})}{1/N}\right|.
\label{eq:pchip_slope_stability_local}
\end{align}

Now fix any $w\in[0,1]$ and choose $n$ such that $\Farc_t(w)\in[n/N,(n+1)/N]$. Let $q=\Farc_t(w)$ and $\theta=N(q-n/N)\in[0,1]$. By Algorithm~\ref{alg:PCHIP}, applied in the coordinate $\Farc_t(w)$ and pulled back through $\Farc_t$, we have
\begin{align}
\tilde s_t(w) = & H_{00}(\theta)\tilde s_t(w_n^{(t)}) + H_{10}(\theta)\frac{1}{N}\tilde m_n + H_{01}(\theta)\tilde s_t(w_{n+1}^{(t)}) + H_{11}(\theta)\frac{1}{N}\tilde m_{n+1},\\
\hat s_t(w) = & H_{00}(\theta)\hat s_t(w_n^{(t)}) + H_{10}(\theta)\frac{1}{N}\hat m_n + H_{01}(\theta)\hat s_t(w_{n+1}^{(t)}) + H_{11}(\theta)\frac{1}{N}\hat m_{n+1}.
\end{align}
Since the Hermite basis functions in Algorithm~\ref{alg:PCHIP} are uniformly bounded on $[0,1]$, there exists $C_H=\Oc(1)$ such that
\begin{align}
|\tilde s_t(w)-\hat s_t(w)| \le C_H\left(\max_{0\le j\le N}|\tilde s_t(w_j^{(t)})-\hat s_t(w_j^{(t)})|+\frac1N\max_{0\le j\le N}|\tilde m_j-\hat m_j|\right).
\label{eq:pchip_hermite_stability_explicit}
\end{align}
Moreover, by the definition of the segment slopes used in Algorithm~\ref{alg:PCHIP},
\begin{align}
\left|\frac{\tilde s_t(w_{j+1}^{(t)})-\tilde s_t(w_j^{(t)})}{1/N}-\frac{\hat s_t(w_{j+1}^{(t)})-\hat s_t(w_j^{(t)})}{1/N}\right| \le 2N\max_{0\le k\le N}|\tilde s_t(w_k^{(t)})-\hat s_t(w_k^{(t)})|.
\label{eq:delta_gap_by_nodal_gap}
\end{align}
Substituting \eqref{eq:pchip_slope_stability_local} and \eqref{eq:delta_gap_by_nodal_gap} into \eqref{eq:pchip_hermite_stability_explicit}, and taking the supremum over $w\in[0,1]$, yields
\begin{align}
\|\tilde s_t-\hat s_t\|_\infty \le C_{\rm int}\max_{0\le j\le N}|\tilde s_t(w_j^{(t)})-\hat s_t(w_j^{(t)})|,
\label{eq:pchip_local_stability_exact_inexact_final}
\end{align}
where $C_{\rm int}=\Oc(1)$. By Lemma~\ref{lem:exact_inexact_polyline_sampled_weights_gap}, the sampled cumulative chord-length perturbation satisfies $\max_{0\le j\le N}|\tilde s_t(w_j^{(t)})-\hat s_t(w_j^{(t)})| = \Oc(N\epsilon)$, and therefore $\|\tilde s_t-\hat s_t\|_\infty = \Oc(N\epsilon)$.

It remains to normalize. Since $\hat s_t(1)$ is the exact cumulative chord length, Lemma~\ref{lemma:arc_chord_sqrt2_pf} gives $\hat s_t(1)\ge s(1)/\sqrt2$. Also, by Lemma~\ref{lem:exact_inexact_polyline_sampled_weights_gap} and the local choice of $\epsilon$, $\tilde s_t(1)\ge s(1)/(2\sqrt2)$. For any $w\in[0,1]$,
\begin{align}
|\tilde \Farc_t(w)-\hat \Farc_t(w)| = \left| \frac{\tilde s_t(w)}{\tilde s_t(1)} - \frac{\hat s_t(w)}{\hat s_t(1)} \right| \leq  & \frac{|\tilde s_t(w)-\hat s_t(w)|}{\tilde s_t(1)}+\hat s_t(w)\left| \frac{1}{\tilde s_t(1)}-\frac{1}{\hat s_t(1)} \right| \nonumber\\
\le & \frac{C}{s(1)}\|\tilde s_t-\hat s_t\|_\infty = \Oc\left(\frac{N\epsilon}{s(1)}\right).
\end{align}
Taking the supremum over $w\in[0,1]$ proves \eqref{eq:sup_gap_F_exact_inexact_final}. 

Finally, for the Lipschitz-smoothness claim, the bounds \eqref{eq:tildedelta_bounds_localstab} show that the inexact data remain in the strictly positive-slope regime, and it remains to verify the discrete slope-variation condition required by Lemma~\ref{lem:pchip_derivative_bounds}. The exact segment slopes satisfy this condition by Lemma~\ref{lem:discrete_growth_control_tildeF_biobj}. For the inexact segment slopes, \eqref{eq:delta_gap_direct} gives an $\Oc(N\epsilon)$ perturbation, and after normalization by $\tilde s_t(1)$ this contributes $\Oc(N\epsilon/s(1))$ to the normalized segment slopes. Hence, under the local condition $\epsilon=\Oc\left(\frac{v_{\min}^2}{v_{\max}}N^{-1}\right)$ with a sufficiently small hidden constant, the positive-slope and discrete slope-variation conditions of Lemma~\ref{lem:pchip_derivative_bounds} continue to hold with uniform constants. Therefore, the PCHIP interpolant remains strictly increasing and uniformly $l_{\Farc_{\rm iter},1}$-Lipschitz-smooth. Pulling the interpolant back through the known map $\Farc_t$ preserves the stated derivative bounds, which completes the proof.
\end{proof}

In this way, we first give an idealized inexact-solver result: assuming the inner iterates remain uniformly within an $\epsilon$-tube around the exact scalarized solutions, which practically can be achieved with large $K$ choice for \textsc{InnerSolver}, SURF retains the exact-solver contraction up to an additional perturbation floor.
\begin{theorem}[SURF convergence with inexact \textsc{InnerSolver}]
    Suppose Assumptions~\ref{assump:hidden_regular},~\ref{assump:LICQ},~\ref{assump:tangent_gap} hold.
    Suppose \textsc{InnerSolver} returns $\bx_n^{(t)}$ satisfying $\max_n\|\bx_n^{(t)}-\bx_{w_n^{(t)}}^*\|\leq \epsilon= \Oc\left(\frac{v_{\min}^2}{v_{\max}}N^{-1}\right)$ for all $t$
    and it chooses $N = \Omega(\kappa^{1.5})$, with $\kappa := \frac{4 \sqrt{2}v_{\max}}{v_{\min}}\sqrt{l_{\Farc_{\rm iter},1}+l_{\Farc,1}}$.
    Then Algorithm~\ref{alg:SURF} converges to a floor, i.e., 
    \begin{align}
    \|\Farc_t-\Farc\|_\infty
    \le
    \Bigl(1-\frac{\alpha}{2}\Bigr)^t\|\Farc_0-\Farc\|_\infty + \Oc\left(\frac{1}{N^2} + \frac{N\epsilon }{s(1)}\right),\quad \forall t \geq 1.
    \label{eq:inexact_convergence_epsilon}
    \end{align}
\end{theorem}
\begin{proof}
From the update rule
$
\Farc_{t+1}=(1-\alpha)\Farc_t+\alpha \tilde \Farc_t,
$
we have
\begin{align}
\|\Farc_{t+1}-\Farc\|_\infty
\leq &
(1-\alpha)\|\Farc_t-\Farc\|_\infty+\alpha\|\tilde \Farc_t-\Farc\|_\infty
\nonumber\\
\leq &
(1-\alpha)\|\Farc_t-\Farc\|_\infty
+\alpha\|\hat \Farc_t-\Farc\|_\infty
+\alpha\|\tilde \Farc_t-\hat \Farc_t\|_\infty.
\label{eq:inexact_step_split_fixed}
\end{align}
By Lemma~\ref{lem:exact_inexact_polyline_sampled_weights_gap}, when $\epsilon= \Oc\left(\frac{v_{\min}^2}{v_{\max}}N^{-1}\right)$, we have $\|\tilde \Farc_t-\hat \Farc_t\|_\infty
\le
\Oc\left(
\frac{N\epsilon}{s(1)}
\right)$.
Substituting this into \eqref{eq:inexact_step_split_fixed} gives
\begin{align}
 \|\Farc_{t+1}-\Farc\|_\infty \le (1-\alpha)\|\Farc_t-\Farc\|_\infty +\alpha\|\hat \Farc_t-\Farc\|_\infty +\alpha\Oc\left( \frac{N\epsilon}{s(1)} \right).
 \label{eq:inexact_step_split_fixed2}
\end{align}

Now, $\hat \Farc_t$ is the rebuilt CDF obtained from the exact PF points on the
same sampled weights $\{w_n^{(t)}\}_{n=0}^N$. Hence the exact-case argument from
the proof of \eqref{eq:exact_convergence} applies verbatim. In
particular, under the same condition on $N$ as required there, we have
\begin{align}
&(1-\alpha)\|\Farc_t-\Farc\|_\infty
+\alpha\|\hat \Farc_t-\Farc\|_\infty
\nonumber\\
\le&
\Bigl(1-\frac{\alpha}{2}\Bigr)\|\Farc_t-\Farc\|_\infty
+
\alpha
\frac{
\frac{\sqrt{2}}{4}
\left(\frac{s(1)}{v_{\min}}\right)^2
(3l_{\Farc_{\rm iter},1}+l_{\Farc,1})
+
\frac{\sqrt{2}}{3}
l_{\tilde{\f}_{\mathrm{PF}},1}^2s(1)^2
}{N^2}
+
\alpha\Oc\left(\frac{1}{N^4}\right).
\label{eq:inexact_recursion_fixed}
\end{align}

Combining \eqref{eq:inexact_step_split_fixed2} and \eqref{eq:inexact_recursion_fixed}, we obtain
\begin{align}
\|\Farc_{t+1}-\Farc\|_\infty
\le&
\Bigl(1-\frac{\alpha}{2}\Bigr)\|\Farc_t-\Farc\|_\infty
+
\alpha
\frac{
\frac{\sqrt{2}}{4}
\left(\frac{s(1)}{v_{\min}}\right)^2
(3l_{\Farc_{\rm iter},1}+l_{\Farc,1})
+
\frac{\sqrt{2}}{3}
l_{\tilde{\f}_{\mathrm{PF}},1}^2s(1)^2
}{N^2}
\nonumber\\
&+
\alpha\Oc\left(
\frac{N\epsilon}{s(1)}
+
\frac{1}{N^4}
\right).
\label{eq:At_final_recursion_part1}
\end{align}

Iterating \eqref{eq:At_final_recursion_part1} for $t\ge 1$ proves the claim.
\end{proof}

We now turn to the proof of Theorem~\ref{thm:cdf}, where the \textsc{InnerSolver} is warm-started and run for a finite number of steps. The next lemma couples the outer CDF recursion with the optimization error induced by the inexact \textsc{InnerSolver}.
\begin{lemma}[Coupled recursion]
\label{lem:coupled_outer_inner_recursion}
Suppose Assumptions~\ref{assump:hidden_regular}, \ref{assump:LICQ},~\ref{assump:tangent_gap} hold.
Consider Algorithm~\ref{alg:SURF} with an inexact \textsc{InnerSolver} that returns iterates $\bu_n^{(t)}$, and define their hidden-space images by
$\bx_n^{(t)}:=g(\bu_n^{(t)})$ \textcolor{gray}{or simply let $\bx_n^{(t)}$ denote the returned
iterate when the algorithm is formulated directly in the convex space $\X$}.
Assume moreover that the inner solver is warm-started from the previous outer iterate and that, after $K$ inner steps, it satisfies
\begin{align}
\|\bx_n^{(t+1)}-\bx_{w_n^{(t+1)}}^*\| \le \eta_K \|\bx_n^{(t)}-\bx_{w_n^{(t+1)}}^*\|, \quad \forall t,n,
\label{eq:abstract_inner_rate_rhoK}
\end{align}
where $\{\eta_K\}_{K\ge1}$ satisfies $\eta_K\downarrow 0$ as $K\to\infty$.
Then there exist positive constants $C_1,C_2,C_3,C_4<\infty$, independent of $t$, $N$, and $K$, such that for all $t\ge0$,
\begin{align}
\|\Farc_{t+1}-\Farc\|_\infty
&\le \bigl(1-\tfrac{\alpha}{2}\bigr)\|\Farc_t-\Farc\|_\infty + C_1\alpha\frac{1}{N^2} + C_2\alpha\frac{N}{s(1)}\max_n\|\bx_n^{(t)}-\bx_{w_n^{(t)}}^*\|,
\label{eq:outer_recursion_rhoK}
\\
\max_n\|\bx_n^{(t+1)}-\bx_{w_n^{(t+1)}}^*\|
&\le \eta_K\max_n\|\bx_n^{(t)}-\bx_{w_n^{(t)}}^*\| + C_3\eta_K\max_n|w_n^{(t+1)}-w_n^{(t)}|,
\label{eq:inner_tracking_recursion_rhoK}
\\
\max_n|w_n^{(t+1)}-w_n^{(t)}|
&\le C_4\alpha\|\Farc_t-\Farc\|_\infty + C_4\alpha\frac{1}{N^2} + C_4\alpha\frac{N}{s(1)}\max_n\|\bx_n^{(t)}-\bx_{w_n^{(t)}}^*\|.
\label{eq:grid_motion_bound_rhoK}
\end{align}
\end{lemma}
\begin{proof}
We first prove \eqref{eq:outer_recursion_rhoK}. Let $\hat \Farc_t$ denote the rebuilt CDF obtained from the exact PF points $\{\fPF(w_n^{(t)})\}_{n=0}^N$, and let $\tilde \Farc_t$ denote the rebuilt CDF obtained from the inexact points $\{\f(\bx_n^{(t)})\}_{n=0}^N$.

By the same argument as in Lemma~\ref{lem:exact_inexact_polyline_sampled_weights_gap}, for every $n$, $\|\f(\bx_n^{(t)})-\fPF(w_n^{(t)})\|\le G_f\|\bx_n^{(t)}-\bx_{w_n^{(t)}}^*\|$. Repeating the proofs of Lemma~\ref{lem:exact_inexact_polyline_sampled_weights_gap} with $\epsilon$ replaced by $\max_n\|\bx_n^{(t)}-\bx_{w_n^{(t)}}^*\|$, and using that PCHIP is applied in the coordinate $q=\Farc_t(w)$, gives
\begin{align}
\|\tilde \Farc_t-\hat \Farc_t\|_\infty \le C_2\frac{N}{s(1)}\max_n\|\bx_n^{(t)}-\bx_{w_n^{(t)}}^*\|.
\label{eq:Ftilde_Fhat_by_innererr_rhoK}
\end{align}
Now, using the update rule $\Farc_{t+1}=(1-\alpha)\Farc_t+\alpha\tilde \Farc_t$, we have
\begin{align}
\|\Farc_{t+1}-\Farc\|_\infty
&\le (1-\alpha)\|\Farc_t-\Farc\|_\infty+\alpha\|\tilde \Farc_t-\Farc\|_\infty \nonumber\\
&\le (1-\alpha)\|\Farc_t-\Farc\|_\infty+\alpha\|\hat \Farc_t-\Farc\|_\infty+\alpha\|\tilde \Farc_t-\hat \Farc_t\|_\infty.
\end{align}
Substituting \eqref{eq:Ftilde_Fhat_by_innererr_rhoK}, and using the same exact-rebuild estimate as in the proof of \eqref{eq:exact_convergence} proves \eqref{eq:outer_recursion_rhoK}:
\begin{align}
(1-\alpha)\|\Farc_t-\Farc\|_\infty+\alpha\|\hat \Farc_t-\Farc\|_\infty \le \Bigl(1-\frac{\alpha}{2}\Bigr)\|\Farc_t-\Farc\|_\infty+C_1\alpha\frac{1}{N^2}.
\end{align}

Next we prove \eqref{eq:inner_tracking_recursion_rhoK}. By the assumed warm-started $K$-step contraction \eqref{eq:abstract_inner_rate_rhoK}, $\|\bx_n^{(t+1)}-\bx_{w_n^{(t+1)}}^*\| \le \eta_K\|\bx_n^{(t)}-\bx_{w_n^{(t+1)}}^*\|$. Adding and subtracting $\bx_{w_n^{(t)}}^*$ gives
\begin{align}
\|\bx_n^{(t+1)}-\bx_{w_n^{(t+1)}}^*\| \le \eta_K\|\bx_n^{(t)}-\bx_{w_n^{(t)}}^*\|+\eta_K\|\bx_{w_n^{(t)}}^*-\bx_{w_n^{(t+1)}}^*\|.
\end{align}
By Lemma~\ref{lem:fPFdifferentiable}, the map $w\mapsto \bx_w^*$ is continuously differentiable on the compact interval $[0,1]$, hence Lipschitz. Therefore there exists a constant $C_3>0$ such that $\|\bx_{w_n^{(t)}}^*-\bx_{w_n^{(t+1)}}^*\|\le C_3|w_n^{(t+1)}-w_n^{(t)}|$. Taking maximum over $n$ proves \eqref{eq:inner_tracking_recursion_rhoK}.

Finally we prove \eqref{eq:grid_motion_bound_rhoK}. Since $w_n^{(t)}=\Farc_t^{-1}(n/N)$, the inverse stability bound from the derivative lower bound on $\Farc_t$ implies $\max_n|w_n^{(t+1)}-w_n^{(t)}|\le C\|\Farc_{t+1}-\Farc_t\|_\infty$ for some constant $C>0$.
Also, $\Farc_{t+1}-\Farc_t=\alpha(\tilde \Farc_t-\Farc_t)$, hence
\begin{align}
\|\Farc_{t+1}-\Farc_t\|_\infty
&\le \alpha\|\tilde \Farc_t-\Farc\|_\infty+\alpha\|\Farc_t-\Farc\|_\infty \nonumber\\
&\le \alpha\|\hat \Farc_t-\Farc\|_\infty+\alpha\|\tilde \Farc_t-\hat \Farc_t\|_\infty+\alpha\|\Farc_t-\Farc\|_\infty.
\label{eq:Fmove_split_rhoK}
\end{align}
From the proof of \eqref{eq:exact_convergence}, the exact rebuilt CDF satisfies $\|\hat \Farc_t-\Farc\|_\infty\le \frac12\|\Farc_t-\Farc\|_\infty+C\frac{1}{N^2}$, where the exact-case constants are absorbed into $C$. Substituting this and \eqref{eq:Ftilde_Fhat_by_innererr_rhoK} into \eqref{eq:Fmove_split_rhoK}, we obtain
\begin{align}
\|\Farc_{t+1}-\Farc_t\|_\infty \le C\alpha\|\Farc_t-\Farc\|_\infty+C\alpha\frac{1}{N^2}+C\alpha\frac{N}{s(1)}\max_n\|\bx_n^{(t)}-\bx_{w_n^{(t)}}^*\|.
\end{align}
Absorbing the inverse-stability constant into $C_4$ proves \eqref{eq:grid_motion_bound_rhoK}.
\end{proof}

We are now ready to finish proving Theorem~\ref{thm:cdf}. Combining the preceding inexact-CDF perturbation bound with the coupled outer-inner recursion yields the following more general inexact convergence result. Theorem~\ref{thm:cdf} follows directly from the regime (i) while regime (ii) provides additional analysis.

\begin{theorem}[SURF convergence under inexact \textsc{InnerSolver}]
\label{thm:warm_started_two_regimes}
Suppose Assumptions~\ref{assump:hidden_regular},~\ref{assump:LICQ},~\ref{assump:tangent_gap} hold. Use the notation of Lemma~\ref{lem:coupled_outer_inner_recursion}. Choose $N=\Omega(\kappa^{1.5})$ with $\kappa := \frac{4\sqrt{2}v_{\max}}{v_{\min}}\sqrt{l_{\Farc_{\rm iter},1}+l_{\Farc,1}}$. The following two regimes apply.

\textbf{(i) Stronger inner contraction, allowing a constant outer step.} Choose $\alpha \leq 1/2$ and $K$ such that
\begin{align}
\eta_K=\Oc\left(\min\left\{\frac{v_{\min}^2}{v_{\max}}N^{-1},\frac{s(1)}{\alpha N}\right\}\right)=\Oc(N^{-1}),
\label{eq:etaK_constant_alpha_regime}
\end{align}
then there exists $C=\left(\frac{4C_2}{s(1)}+\frac{v_{\max}}{v_{\min}^2}\right)N=\Theta(N)$ such that
\begin{align}
&\|\Farc_t-\Farc\|_\infty+C\max_n\|\bx_n^{(t)}-\bx_{w_n^{(t)}}^*\| \nonumber\\
\leq&
\left(1-\frac{\alpha}{4}\right)^t\left(\|\Farc_0-\Farc\|_\infty+C\max_n\|\bx_n^{(0)}-\bx_{w_n^{(0)}}^*\|\right)+\Oc\left(\frac1{N^2}\right).
\label{eq:constant_alpha_regime_bound}
\end{align}
Consequently, $\max_n\|\bx_n^{(t)}-\bx_{w_n^{(t)}}^*\|=\Oc\left(\frac{v_{\min}^2}{v_{\max}}N^{-1}\right)$ for all $t\ge0$, and after the transient term is below the floor, $\|\Farc_t-\Farc\|_\infty=\Oc(N^{-2})$ and $\max_n\|\bx_n^{(t)}-\bx_{w_n^{(t)}}^*\|=\Oc(N^{-3})$.

\textbf{(ii) Smaller outer step with warmup.} Use warmup initialization to ensure $\max_n\|\bx_n^{(0)}-\bx_{w_n^{(0)}}^*\|=\Oc\left(\frac{v_{\min}^2}{v_{\max}}N^{-1}\right)$, choose $\alpha=\Theta(N^{-1})$ and $K$ such that
\begin{align}
\eta_K\le \min\left\{\frac12\left(1+C_3C_4\alpha\frac{N}{s(1)}\right)^{-1},\frac14\left(\left(1+\frac{4C_2\alpha N}{s(1)}\right)C_3C_4\right)^{-1}\right\}=\Oc(1),
\label{eq:etaK_small_alpha_regime}
\end{align}
then there exists $C=1+\frac{4C_2\alpha N}{s(1)}$ such that
\begin{align}
&\|\Farc_t-\Farc\|_\infty+C\max_n\|\bx_n^{(t)}-\bx_{w_n^{(t)}}^*\| \nonumber\\
\leq&
\left(1-\frac{\alpha}{4}\right)^t\left(\|\Farc_0-\Farc\|_\infty+C\max_n\|\bx_n^{(0)}-\bx_{w_n^{(0)}}^*\|\right)+\Oc\left(\frac1{N^2}\right).
\label{eq:small_alpha_regime_bound}
\end{align}
Consequently, once the transient term is below the discretization floor, $\|\Farc_t-\Farc\|_\infty=\Oc(N^{-2})$. The inner tracking recursion then contracts to the floor $\Oc(\alpha\eta_KN^{-2})$, so after $\Oc(\log N)$ additional outer iterations, $\max_n\|\bx_n^{(t)}-\bx_{w_n^{(t)}}^*\|=\Oc(\alpha\eta_KN^{-2})=\Oc(N^{-3})$.
\end{theorem}

\begin{proof}
By Lemma~\ref{lem:coupled_outer_inner_recursion}, substituting \eqref{eq:grid_motion_bound_rhoK} into \eqref{eq:inner_tracking_recursion_rhoK} gives
\begin{align}
&\max_n\|\bx_n^{(t+1)}-\bx_{w_n^{(t+1)}}^*\| \nonumber\\
\leq&
\eta_K\left(1+C_3C_4\alpha\frac{N}{s(1)}\right)\max_n\|\bx_n^{(t)}-\bx_{w_n^{(t)}}^*\|
+C_3C_4\alpha\eta_K\|\Farc_t-\Farc\|_\infty+C_3C_4\alpha\eta_K\frac1{N^2}.
\label{eq:inner_after_grid_two_regimes_clean}
\end{align}

For regime (i), choose $C\ge \left(\frac{4C_2}{s(1)}+\frac{v_{\max}}{v_{\min}^2}\right)N$, so $C=\Theta(N)$. Multiplying \eqref{eq:inner_after_grid_two_regimes_clean} by $C$ and adding \eqref{eq:outer_recursion_rhoK} gives
\begin{align}
&\|\Farc_{t+1}-\Farc\|_\infty+C\max_n\|\bx_n^{(t+1)}-\bx_{w_n^{(t+1)}}^*\| \nonumber\\
\leq&
\left(1-\frac{\alpha}{2}+CC_3C_4\alpha\eta_K\right)\|\Farc_t-\Farc\|_\infty \nonumber\\
&+\left(C_2\alpha\frac{N}{s(1)}+C\eta_K\left(1+C_3C_4\alpha\frac{N}{s(1)}\right)\right)\max_n\|\bx_n^{(t)}-\bx_{w_n^{(t)}}^*\| \nonumber\\
&+\alpha\left(C_1+CC_3C_4\eta_K\right)\frac1{N^2}.
\label{eq:constant_alpha_expand_clean}
\end{align}
The condition \eqref{eq:etaK_constant_alpha_regime} can be chosen with a sufficiently small hidden constant so that $CC_3C_4\eta_K\le 1/4$ and $\eta_K\left(1+C_3C_4\alpha\frac{N}{s(1)}\right)\le 1-\alpha/2$. Since $C\ge4C_2N/s(1)$, we also have $C_2\alpha N/s(1)\le \alpha C/4$. Therefore,
\begin{align}
&\|\Farc_{t+1}-\Farc\|_\infty+C\max_n\|\bx_n^{(t+1)}-\bx_{w_n^{(t+1)}}^*\| \nonumber\\
\leq&
\left(1-\frac{\alpha}{4}\right)\|\Farc_t-\Farc\|_\infty+\left(\frac{\alpha}{4}+1-\frac{\alpha}{2}\right)C\max_n\|\bx_n^{(t)}-\bx_{w_n^{(t)}}^*\|+\alpha\Oc\left(\frac1{N^2}\right) \nonumber\\
=&
\left(1-\frac{\alpha}{4}\right)\|\Farc_t-\Farc\|_\infty+\left(1-\frac{\alpha}{4}\right)C\max_n\|\bx_n^{(t)}-\bx_{w_n^{(t)}}^*\|+\alpha\Oc\left(\frac1{N^2}\right) \nonumber\\
=&
\left(1-\frac{\alpha}{4}\right)\left(\|\Farc_t-\Farc\|_\infty+C\max_n\|\bx_n^{(t)}-\bx_{w_n^{(t)}}^*\|\right)+\alpha\Oc\left(\frac1{N^2}\right).
\label{eq:constant_alpha_one_step_clean}
\end{align}
Here the first inequality uses the two bounds on $\eta_K$ and the choice $C\ge4C_2N/s(1)$, and the last two lines only collect terms. Iterating \eqref{eq:constant_alpha_one_step_clean} gives \eqref{eq:constant_alpha_regime_bound}. Since $\|\Farc_0-\Farc\|_\infty=\Oc(1)$ and $C\max_n\|\bx_n^{(0)}-\bx_{w_n^{(0)}}^*\|=\Oc(1)$, the bound implies $\max_n\|\bx_n^{(t)}-\bx_{w_n^{(t)}}^*\|=\Oc\left(\frac{v_{\min}^2}{v_{\max}}N^{-1}\right)$ for all $t\ge0$. After the transient term is below $\Oc(N^{-2})$, \eqref{eq:constant_alpha_regime_bound} gives $\|\Farc_t-\Farc\|_\infty=\Oc(N^{-2})$ and $C\max_n\|\bx_n^{(t)}-\bx_{w_n^{(t)}}^*\|=\Oc(N^{-2})$, hence $\max_n\|\bx_n^{(t)}-\bx_{w_n^{(t)}}^*\|=\Oc(N^{-3})$.

For regime (ii), choose $C\ge 1+\frac{4C_2\alpha N}{s(1)}$, so $C=\Theta(1)$ because $\alpha=\Oc(N^{-1})$. Multiplying \eqref{eq:inner_after_grid_two_regimes_clean} by $C$ and adding \eqref{eq:outer_recursion_rhoK} gives
\begin{align}
&\|\Farc_{t+1}-\Farc\|_\infty+C\max_n\|\bx_n^{(t+1)}-\bx_{w_n^{(t+1)}}^*\| \nonumber\\
\leq&
\left(1-\frac{\alpha}{2}+CC_3C_4\alpha\eta_K\right)\|\Farc_t-\Farc\|_\infty +\alpha\left(C_1+CC_3C_4\eta_K\right)\frac1{N^2} \nonumber\\
&+\left(C_2\alpha\frac{N}{s(1)}+C\eta_K\left(1+C_3C_4\alpha\frac{N}{s(1)}\right)\right)\max_n\|\bx_n^{(t)}-\bx_{w_n^{(t)}}^*\|
\label{eq:small_alpha_expand_clean}
\end{align}
The condition \eqref{eq:etaK_small_alpha_regime} gives $CC_3C_4\eta_K\le 1/4$ and $\eta_K\left(1+C_3C_4\alpha\frac{N}{s(1)}\right)\le 1/2$. Since $C\ge4C_2\alpha N/s(1)$, we have $C_2\alpha N/s(1)\le C/4$. Therefore,
\begin{align}
&\|\Farc_{t+1}-\Farc\|_\infty+C\max_n\|\bx_n^{(t+1)}-\bx_{w_n^{(t+1)}}^*\| \nonumber\\
\leq&
\left(1-\frac{\alpha}{4}\right)\|\Farc_t-\Farc\|_\infty+\frac34 C\max_n\|\bx_n^{(t)}-\bx_{w_n^{(t)}}^*\|+\alpha\Oc\left(\frac1{N^2}\right) \nonumber\\
\leq&
\left(1-\frac{\alpha}{4}\right)\|\Farc_t-\Farc\|_\infty+\left(1-\frac{\alpha}{4}\right)C\max_n\|\bx_n^{(t)}-\bx_{w_n^{(t)}}^*\|+\alpha\Oc\left(\frac1{N^2}\right) \nonumber\\
=&
\left(1-\frac{\alpha}{4}\right)\left(\|\Farc_t-\Farc\|_\infty+C\max_n\|\bx_n^{(t)}-\bx_{w_n^{(t)}}^*\|\right)+\alpha\Oc\left(\frac1{N^2}\right).
\label{eq:small_alpha_one_step_clean}
\end{align}
Here the first inequality uses \eqref{eq:etaK_small_alpha_regime} and $C\ge4C_2\alpha N/s(1)$, the second uses $\alpha\le1$, and the last line collects terms. Iterating \eqref{eq:small_alpha_one_step_clean} gives \eqref{eq:small_alpha_regime_bound}. Once the transient term in \eqref{eq:small_alpha_regime_bound} is below $\Oc(N^{-2})$, we have $\|\Farc_t-\Farc\|_\infty=\Oc(N^{-2})$. Substituting this into \eqref{eq:inner_after_grid_two_regimes_clean} and using \eqref{eq:etaK_small_alpha_regime} gives
\begin{align}
\max_n\|\bx_n^{(t+1)}-\bx_{w_n^{(t+1)}}^*\| 
\leq&
\frac12\max_n\|\bx_n^{(t)}-\bx_{w_n^{(t)}}^*\|+\Oc(\alpha\eta_KN^{-2}).
\label{eq:small_alpha_inner_tracking_clean}
\end{align}
Thus, after an additional logarithmic transient, $\max_n\|\bx_n^{(t)}-\bx_{w_n^{(t)}}^*\|=\Oc(\alpha\eta_KN^{-2})=\Oc(N^{-3})$, where the last equality uses $\alpha=\Oc(N^{-1})$ and $\eta_K=\Oc(1)$. This completes the proof.
\end{proof}
Additionally, we provide the connection between CDF convergence and PF uniformity metrics CV.
\subsection{Connection between SURF Convergence and PF Uniformity}
\label{app:SURV_CV}
\begin{lemma}[CDF error controls PF spacing CV]
\label{lem:cv_cdf_error}
Let $w_n^{(t)}=\Phi_t^{-1}(n/N)$ for $n=0,\dots,N$, and let $\tilde\Phi_t$ be the normalized cumulative chord-length CDF constructed from the sampled PF points $\{\h(\bu_n^{(t)})\}_{n=0}^N$ as in \eqref{eq:sw_tilde}. Then, for all $t\geq0$,
\begin{align}
\mathrm{CV}
:=&\frac{\mathrm{std}\big(\{\|\h(\bu_{n+1}^{(t)})-\h(\bu_n^{(t)})\|\}_{n=0}^{N-1}\big)}
{\mathrm{mean}\big(\{\|\h(\bu_{n+1}^{(t)})-\h(\bu_n^{(t)})\|\}_{n=0}^{N-1}\big)}
\leq 2N\|\tilde\Phi_t-\Phi_t\|_\infty .
\end{align}
Moreover, under the discretization and inexact-\textsc{InnerSolver} perturbation bounds used in Theorem~\ref{thm:cdf}, if $\max_{0\le n\le N}\|\bx_n^{(t)}-\bx_{w_n^{(t)}}^*\|=\Oc\big(N^{-1}\|\Phi_t-\Phi\|_\infty\big)$, then
\begin{align}
\mathrm{CV}
=&\Oc\left(N^{-1}+N\|\Phi_t-\Phi\|_\infty\right).
\end{align}
Consequently, under the conditions of Theorem~\ref{thm:cdf}, for large enough $t$ we have $\mathrm{CV}=\Oc(N^{-1})$.
\end{lemma}
\begin{proof}
By construction of $\tilde\Phi_t$ from normalized cumulative chord length, each CDF increment equals the corresponding normalized PF segment length:
\begin{align}
\tilde\Phi_t(w_{n+1}^{(t)})-\tilde\Phi_t(w_n^{(t)})
=&\frac{\|\h(\bu_{n+1}^{(t)})-\h(\bu_n^{(t)})\|}
{\sum_{m=0}^{N-1}\|\h(\bu_{m+1}^{(t)})-\h(\bu_m^{(t)})\|}.
\end{align}
Since the mean segment length is $\frac1N\sum_{m=0}^{N-1}\|\h(\bu_{m+1}^{(t)})-\h(\bu_m^{(t)})\|$, the CV can be rewritten as
\begin{align}
\mathrm{CV}
=&\frac{\sqrt{\frac1N\sum_{n=0}^{N-1}\left(\|\h(\bu_{n+1}^{(t)})-\h(\bu_n^{(t)})\|-\frac1N\sum_{m=0}^{N-1}\|\h(\bu_{m+1}^{(t)})-\h(\bu_m^{(t)})\|\right)^2}}
{\frac1N\sum_{m=0}^{N-1}\|\h(\bu_{m+1}^{(t)})-\h(\bu_m^{(t)})\|} \nonumber\\
=&\sqrt{N\sum_{n=0}^{N-1}\left(\frac{\|\h(\bu_{n+1}^{(t)})-\h(\bu_n^{(t)})\|}
{\sum_{m=0}^{N-1}\|\h(\bu_{m+1}^{(t)})-\h(\bu_m^{(t)})\|}-\frac1N\right)^2} \nonumber\\
=&\sqrt{N\sum_{n=0}^{N-1}\left(\tilde\Phi_t(w_{n+1}^{(t)})-\tilde\Phi_t(w_n^{(t)})-\frac1N\right)^2}.
\end{align}
Because $w_n^{(t)}=\Phi_t^{-1}(n/N)$, we have $\Phi_t(w_{n+1}^{(t)})-\Phi_t(w_n^{(t)})=1/N$, and hence
\begin{align}
\left|\tilde\Phi_t(w_{n+1}^{(t)})-\tilde\Phi_t(w_n^{(t)})-\frac1N\right|
=&\left|\tilde\Phi_t(w_{n+1}^{(t)})-\tilde\Phi_t(w_n^{(t)})-\big(\Phi_t(w_{n+1}^{(t)})-\Phi_t(w_n^{(t)})\big)\right| \nonumber\\
=&\left|(\tilde\Phi_t-\Phi_t)(w_{n+1}^{(t)})-(\tilde\Phi_t-\Phi_t)(w_n^{(t)})\right| \leq2\|\tilde\Phi_t-\Phi_t\|_\infty .
\end{align}
Substituting the above bound into the CV expression gives
\begin{align}
\mathrm{CV}
\leq&\sqrt{N\sum_{n=0}^{N-1}4\|\tilde\Phi_t-\Phi_t\|_\infty^2}
=2N\|\tilde\Phi_t-\Phi_t\|_\infty .
\end{align}

It remains to control $\|\tilde\Phi_t-\Phi_t\|_\infty$. By Lemma~\ref{lem:s_tilde_combined_bound}, after substituting the increment bound
\begin{align}
\max_{0\le n\le N-1}\left|e_t\left(\frac{n+1}{N}\right)-e_t\left(\frac nN\right)\right|
\leq&\frac{\kappa}{N}\|\Phi_t-\Phi\|_\infty^{1/2},
\end{align}
the exact chord-length reconstruction satisfies
\begin{align}
\|\hat\Phi_t-\Phi\|_\infty
\leq&\Oc\left(N^{-2}+\frac{\kappa^2}{N^2}\|\Phi_t-\Phi\|_\infty+\frac{\kappa^3}{N^2}\|\Phi_t-\Phi\|_\infty^{3/2}\right).
\end{align}
Since $N=\Omega(\kappa^{1.5})$, we have $\kappa^2/N^2=\Oc(1)$ and $\kappa^3/N^2=\Oc(1)$, and since $\|\Phi_t-\Phi\|_\infty\le \Oc(1)$, this gives $\|\hat\Phi_t-\Phi\|_\infty
\leq\Oc\left(N^{-2}+\|\Phi_t-\Phi\|_\infty\right)$. By Lemma~\ref{lem:exact_inexact_polyline_sampled_weights_gap}, the inexact inner-solver perturbation satisfies $\|\tilde\Phi_t-\hat\Phi_t\|_\infty
\leq\Oc\left(\frac1{s(1)}N\max_{0\le n\le N}\|\bx_n^{(t)}-\bx_{w_n^{(t)}}^*\|\right)$. Therefore,
\begin{align}
\|\tilde\Phi_t-\Phi_t\|_\infty
\leq&\|\tilde\Phi_t-\hat\Phi_t\|_\infty+\|\hat\Phi_t-\Phi\|_\infty+\|\Phi_t-\Phi\|_\infty \nonumber\\
\leq&\Oc\left(N^{-2}+\|\Phi_t-\Phi\|_\infty+N\max_{0\le n\le N}\|\bx_n^{(t)}-\bx_{w_n^{(t)}}^*\|\right).
\end{align}
Combining this with $\mathrm{CV}\le 2N\|\tilde\Phi_t-\Phi_t\|_\infty$ gives
\begin{align}
\mathrm{CV}
=&\Oc\left(N^{-1}+N\|\Phi_t-\Phi\|_\infty+N^2\max_{0\le n\le N}\|\bx_n^{(t)}-\bx_{w_n^{(t)}}^*\|\right).
\end{align}
Under the conditions of Theorem~\ref{thm:cdf}, for large enough $t$, $\|\Phi_t-\Phi\|_\infty=\Oc(N^{-2})$ and $C(N)\max_{0\le n\le N}\|\bx_n^{(t)}-\bx_{w_n^{(t)}}^*\|=\Oc(N^{-2})$ with $C(N)=\Theta(N)$. Hence $\max_{0\le n\le N}\|\bx_n^{(t)}-\bx_{w_n^{(t)}}^*\|=\Oc(N^{-3})$, and therefore $\mathrm{CV}=\Oc(N^{-1})$.
\end{proof}

\section{Experimental Details}
\label{app:exp_details}
This appendix provides additional experimental details for the toy examples in Appendix~\ref{app:toy}, the MO-Gymnasium benchmarks in Appendix~\ref{app:modgym}, and the LLM fine-tuning experiments in Appendix~\ref{app:llm_ft}. We describe the problem setups, optimization protocols, baselines, and hyperparameters used to reproduce the reported results.
\paragraph{Metric Settings.}
\label{app:metric_settings}
Metrics use the task's objective coordinates without per-objective rescaling. DST uses KL-regularized normalized discounted returns, Fishwood uses normalized discounted rewards, and MO-Mountaincar uses KL-regularized discounted episode returns. In minimization coordinates, the common HV reference point is the componentwise worst value among compared points plus $10^{-6}$ for DST or $1$ for MO-Mountaincar, and $(0.2,0.2)$ for Fishwood. IGD uses a dense 101-point SURF reference for DST and 201-point uniform-weight references for Fishwood and MO-Mountaincar. LLM metrics use the regularized or reward-only objectives indicated in Table~\ref{tab:llm_alignment}, with HV reference point $(2.1,1.4)$.

\subsection{Toy Examples}
\label{app:toy}
We first describe the controlled toy examples used to visualize how scalarization weights map to Pareto-front coverage and the behavior of SURF.
\paragraph{Offline Bandit.}
Consider a single-state bandit with $|\mathcal S|=1$ and a finite action set $\mathcal A$. 
The policy $\pi\in\Delta_{|\mathcal A|}$ coincides with the normalized occupancy measure. 
Let $x_a\in[0,1]$ be evenly spaced over $[0,1]$ for each action $a\in\mathcal A$, and define the true mean rewards
\begin{align}
R_1(a)=x_a,
\quad \text{and}\quad
R_2(a)=1-x_a^4 .
\end{align}
We use a uniform reference policy $\pi_{\rm ref}=\mathbf 1/|\mathcal A|$ and define $R_0=\log \pi_{\rm ref}$ elementwise. 
A fixed balanced offline dataset 
$\mathcal D=\{(a_t,r_{1,t},r_{2,t})\}_{t=1}^T$ is collected, where every action appears roughly equally often and rewards are observed with Gaussian noise with standard deviation $0.5$.

\paragraph{Quadratic Gear Toy.}
To isolate scalarization-induced traversal from changes in PF geometry, we construct a controlled gear toy in which the reported PF trajectory is fixed but the LS traversal speed is tunable. The fixed reporting map is $\bar{\mathbf h}(x)=((1-x)^2,x^2)$ for $x\in[0,1]$. We induce different traversal speeds using raw quadratic objectives $h_1^p(x)=\frac{\sqrt p}{2}(x-1)^2$ and $h_2(x)=\frac12x^2$, for which the LS solution is $x_w^*=\frac{\sqrt p\,w}{1+(\sqrt p-1)w}$. The raw objectives determine $x_w^*$, while SURF computes chord lengths and the CDF from $\bar{\mathbf h}(x_w^*)$. Thus, changing the gear parameter $p$ changes how uniform weights traverse the same reported PF, while its curve and arc length remain fixed.

We run SURF with exact inner solutions and PCHIP interpolation for $p\in\{1,\ldots,10\}$ and $N\in\{25,\cdots,40\}$, and compute the induced condition number $\kappa_p=\frac{v_{\max}}{v_{\min}}\sqrt{l_{\Farc,1}}$ from the true CDF. Empirically fitting $\|\tilde\Phi_0-\Phi\|_\infty$ to $C\kappa_p^q/N^2$ gives $q$ decreasing from $1.134$ to $0.907$ as $N$ increases, showing a clear condition-number-dependent first reconstruction error. In contrast, fitting the final error $\|\Phi_T-\Phi\|_\infty$ gives $q$ decreasing from $0.042$ to $0.019$, showing that after refinement the error is essentially independent of $\kappa_p$ and reaches the expected $N^{-2}$ discretization floor. This supports Lemma~\ref{lem:s_tilde_combined_bound} and Theorem~\ref{thm:cdf}: the first CDF reconstruction may carry a large condition-number-dependent constant, while SURF converges to the final $\mathcal{O}(N^{-2})$ floor.

For an analytical interpretation, consider sufficiently large $N$ for which the uniform initial grid $w_m=m/N$ contains $w_n=\frac{1}{\sqrt p+1}$, i.e., $n=\frac{N}{\sqrt p+1}$ is an integer. Then $x_{w_n}^*=\frac{1}{2}$. Therefore, $\tilde\Phi_0(w_n)-\Phi(w_n)=\frac{\frac12(s(1)-\tilde s^{(0)}(1))-(s(w_n)-\tilde s^{(0)}(w_n))}{\tilde s^{(0)}(1)}$ and the exact chord-arc expansion gives $s(w_{m+1})-s(w_m)-\|\bar{\mathbf h}(x_{w_{m+1}}^*)-\bar{\mathbf h}(x_{w_m}^*)\|=\frac{(x_{w_{m+1}}^*-x_{w_m}^*)^3}{12((1-x_{w_m}^*)^2+(x_{w_m}^*)^2)^{3/2}}+o((x_{w_{m+1}}^*-x_{w_m}^*)^3)$. Substituting $x_{w_m}^*=\frac{\sqrt p\,m}{N+(\sqrt p-1)m}$ and integrating gives $s(w_n)-\tilde s^{(0)}(w_n)-\frac12(s(1)-\tilde s^{(0)}(1)) =\Theta(\frac{p}{N^2})$.
Hence $\|\tilde\Phi_0-\Phi\|_\infty=\Omega(\frac{p}{N^2})=\Omega(\frac{v_{\max}/v_{\min}}{N^2})$. 
In contrast, $s(1)=\int_0^1 2\sqrt{(1-x)^2+x^2}\,dx$ is independent of $p$, and $l_{\tilde{\mathbf f}_{\mathrm{PF}},1} = \sup_{x\in[0,1]}\frac{1}{2((1-x)^2+x^2)^{3/2}} = \sqrt{2}$. Therefore, $\|\Phi_T-\Phi\|_\infty=\Oc\left(\frac{l_{\tilde{\mathbf f}_{\mathrm{PF}},1}^2s(1)^3}{6N^2}\right)=\Oc(N^{-2})$, with both $s(1)$ and $l_{\tilde{\mathbf f}_{\mathrm{PF}},1}$ independent of $p$.

\subsection{SURF with Tchebycheff Scalarization on Nonconvex Fronts}
\label{app:nonconvex_tchebycheff}
We evaluate weight refinement beyond LS on two connected nonconvex fronts: the ZDT2 front~\citep{zitzler2000comparison}, parameterized by $(t,1-t^2)$, and a quarter-circle front, parameterized by $(\cos(\pi t/2),\sin(\pi t/2))$, with $t\in[0,1]$. The latter coincides with the bi-objective DTLZ2 front~\citep{deb2002scalable}. We minimize the weighted Tchebycheff objective $\max\{w(f_1-z_1^*),(1-w)(f_2-z_2^*)\}$ with ideal point $\mathbf z^*=(0,0)$, using a 6001-point search over $t$ followed by golden-section refinement to tolerance $10^{-10}$. We use $N=15$ segments, weights in $[10^{-3},1-10^{-3}]$, and 30 SURF outer updates with $\alpha=0.3$ and PCHIP reconstruction. IGD uses 4001 uniformly spaced values of $t$, with no objective rescaling.

Uniform LS yields only two distinct endpoints on both fronts, whereas Tchebycheff scalarization yields 16 distinct points. Compared with uniform Tchebycheff weights, SURF reduces CV from $0.152$ on ZDT2 and $0.202$ on the circle to near zero, and Gap Ratio from $1.657$ and $1.863$ to $1.000$, respectively. IGD decreases from $0.0250$ to $0.0246$ on ZDT2 and from $0.0272$ to $0.0261$ on the circle. These diagnostics demonstrate empirical compatibility with Tchebycheff scalarization outside the LS setting of Theorem~\ref{thm:cdf}.

\subsection{Details of Experiments on MO-Gymnasium Benchmarks}
\label{app:modgym}
We next describe the three MO-Gymnasium environments used in Table~\ref{tab:main_comparison}. For each task, we provide the reward structure, optimization setup, and relevant implementation details.
\subsubsection{Details of DST Problem}
\label{app:dst_setup}

This section describes the experimental setup used to evaluate the proposed PF sampling method on the Deep Sea Treasure (DST) benchmark~\citep{vamplew2011empirical,yahyaa2015thompson}, provided in MO-Gymnasium~\citep{felten_toolkit_2023}.

\paragraph{Problem Setup.}

As illustrated in Figure~\ref{fig:DST}, the environment is a grid world containing obstacle cells (black), treasure cells (gold), and water (light blue). 
This problem can be modeled via
$\mathcal{M}=(\mathcal{S},\mathcal{A},\mathcal{P},\mathbf{r},\gamma,\rho,\tau)$, a regularized MDP
where $\mathcal{S}$ denotes the set of feasible states (all non-obstacle cells in Figure~\ref{fig:DST}), $\mathcal{A}$ is the action set that consists of four actions (up, down, left, and right), and $\mathcal{P}(s'|s,a)$ denotes the deterministic transition dynamics where the action moves to the adjacent cell in the chosen direction or stays at the current cell if the action would move it outside the grid or into an obstacle. The reward vector $\mathbf{r}(s,a)=[r_1(s,a),r_2(s,a)]$ consists of $r_1$ representing the per-step time reward and $r_2$ representing the treasure reward obtained upon reaching a treasure cell. 
$\gamma = 0.999$ is the discount factor, chosen close to and smaller than $1$. $\rho$ is the initial state distribution, which is deterministic at the starting location. To encourage stochasticity~\citep{geist2019theory}, we adopt $\tau$ as the KL regularizer
with a moderate choice of $\beta = 1.5$ and $\pi_{\rm ref}$ as the uniform policy. 

\begin{wrapfigure}{r}{0.3\textwidth}
 \centering
 \vspace{-0.5cm}
 \includegraphics[width=0.29\textwidth]{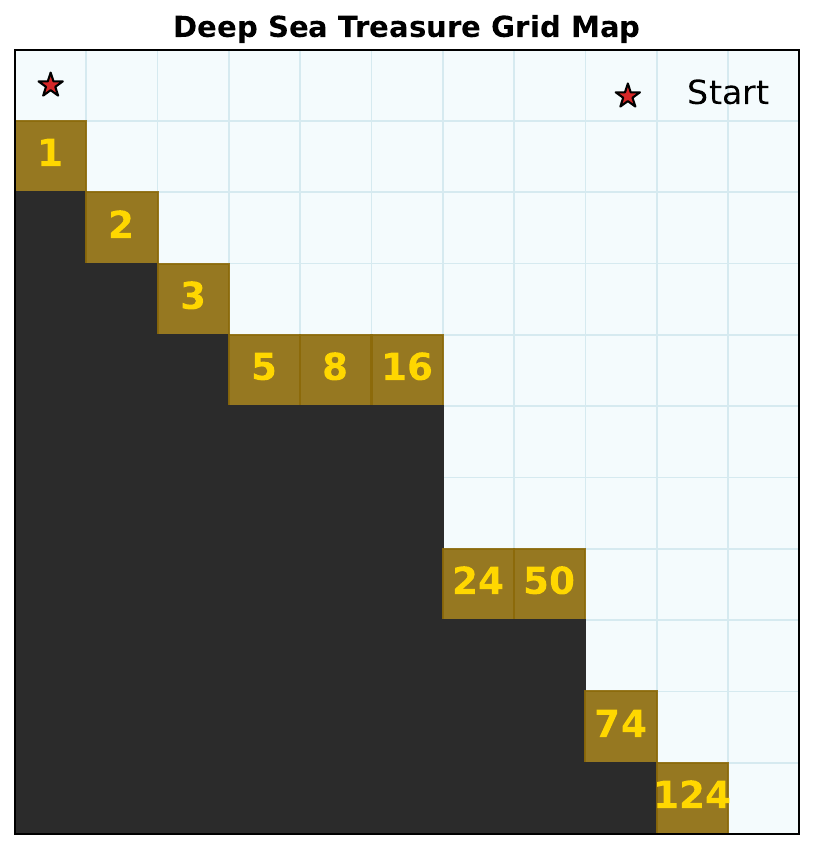}
 \vspace{-0.2cm}
 \caption{DST environment.
 }
 \vspace{-0.5cm}
 \label{fig:DST}
\end{wrapfigure}
The goal is to find a policy $\pi \in \Delta_{|\mathcal A|}^{|\mathcal S|}$ 
minimizing two regularized discounted objectives. Specifically, for $m=1,2$,
\begin{align}
 h_m(\pi) := \mathbb{E}_\pi \left[
 \sum_{t=0}^\infty \gamma^t \bigl(- r_m(s_t,a_t)+\tau(\pi(\cdot|s_t))\bigr)
 \right], 
 \label{eq: f1_discount}
\end{align}
where $r_1(s,a)=-1$ is the per-step time reward, so that $-r_1(s,a)=1$ corresponds to the time cost incurred at each step until a treasure is reached, and
$r_2(s,a)$ is the treasure reward obtained upon reaching a treasure cell and is zero otherwise. Thus, $h_1$ favors policies with smaller travel time, while $h_2$ favors policies reaching higher-value treasures, which tend to take longer time as they are deeper (cf. Figure~\ref{fig:DST}). The resulting trade-off defines the PF studied in our experiments.

\paragraph{Occupancy Measure Representation.}

In this problem, the active state space consists of the $72$ non-obstacle, non-terminal grid states in Figure~\ref{fig:DST}, each with $4$ actions. 
Let $\mu_\pi(s,a)$ denote the discounted occupancy measure distribution induced by policy $\pi$. For the active state-action pairs, the set of valid occupancy distributions forms a convex polytope (cf.~\eqref{eq:occupancy_measure}) and the policy optimization problem is equivalent to the occupancy distribution optimization described in Section~\ref{sec:MDP}. Specifically, the sufficient exploration condition~\citep{zhang2020variational}, given by $\sum_a \mu_\pi(s,a)\geq d_{\min}$ for every active state $s$, and Assumption~\ref{assump:tangent_gap} have been numerically verified on a dense interior grid of scalarization weights, using $1001$ uniformly spaced values in $[0,1]$. 
In particular, with a moderate regularization parameter $\beta=1.5$, there is $d_{\min}\geq 1.02\times 10^{-8}$, with minimum attained near $w\approx0.964$. Increasing the regularization to $\beta=10$ further improves it to $d_{\min}\geq 1.19\times 10^{-6}$. Additionally,
the active-state flow matrix has shape $72\times 288$ and full row rank, verifying Assumption~\ref{assump:LICQ}. 
Furthermore, the projected gradient-difference norm in Assumption~\ref{assump:tangent_gap} remains uniformly positive over the entire weight grid. In the experiment with $\beta=1.5$, the empirical lower bound is $40.3$, attained near $w\approx0.45$, confirming that the tangent-gap condition holds throughout the tested weight range.
In this way, according to the theorem of hidden convexity~\citep{fatkhullin2025stochastic}, optimizing over policies is equivalent to optimizing over occupancy measures, where the objectives become strongly convex in $\mu$ and have the properties needed to establish the CDF $\Farc$ and enable the use of the proposed SURF Algorithm~\ref{alg:SURF} to obtain Pareto solutions that uniformly cover the Pareto Front.
\paragraph{Baselines.}
We compare our proposed arc-length-uniform Pareto sampling method with three scalarization-based baselines. 
\textit{Uniform weighting} chooses equispaced scalarization weights on $[0,1]$.
\textit{Optimistic Linear Support (OLS)}~\citep{barrett2008learning} adaptively selects scalarization weights based on the convex coverage set discovered so far. At each step, OLS chooses a new weight vector that is optimistic with respect to the currently known solutions and solves the corresponding scalarized optimization problem to expand the supported portion of the PF.
\textit{UMOD}~\citep{zhang2024gliding} is a preference-learning approach that iteratively adapts scalarization weights using a surrogate model of the PF. It fits a parametric model mapping preferences to objective vectors from existing solutions, then selects new preferences to maximize the minimum chord length between predicted Pareto points. These preferences are converted into scalarization weights and used to query the optimization oracle for additional solutions. This approach implicitly relies on the surrogate geometry to reflect meaningful separation between solutions. However, when the mapping from scalarization weights to Pareto solutions produces strongly clustered solutions, such that many nearby weights yield nearly identical objective vectors, the surrogate spacing signal becomes weak, potentially reducing the effectiveness of the preference updates.

\paragraph{Results.}
\begin{figure}[t]
 \centering
  \includegraphics[width=0.95\linewidth]{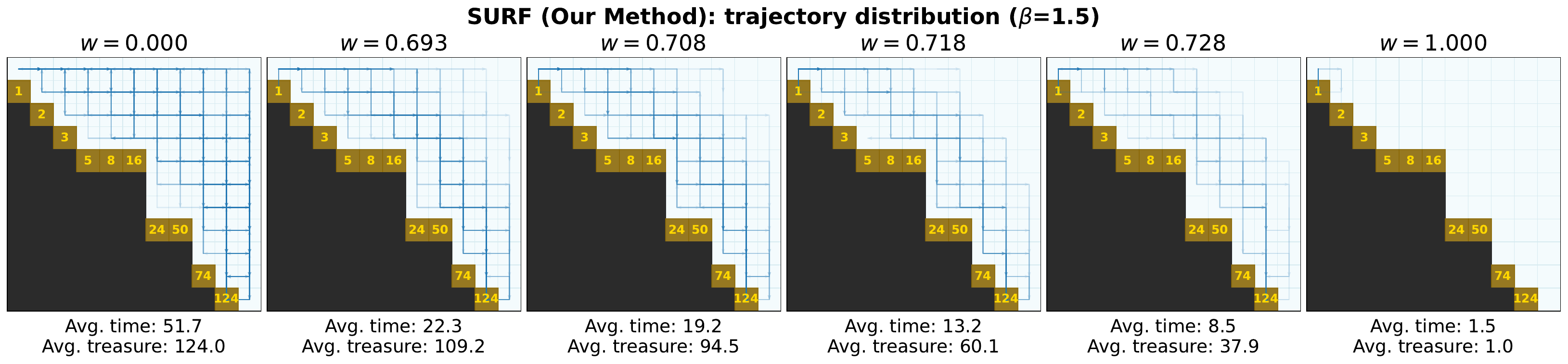}
 \includegraphics[width=0.95\linewidth]{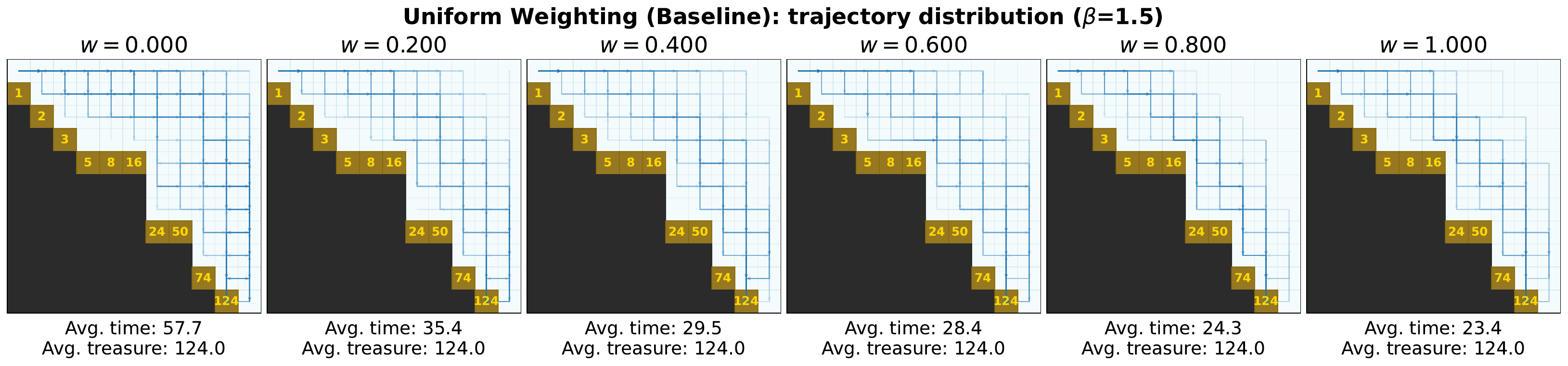}\\
 \includegraphics[width=0.95\linewidth]{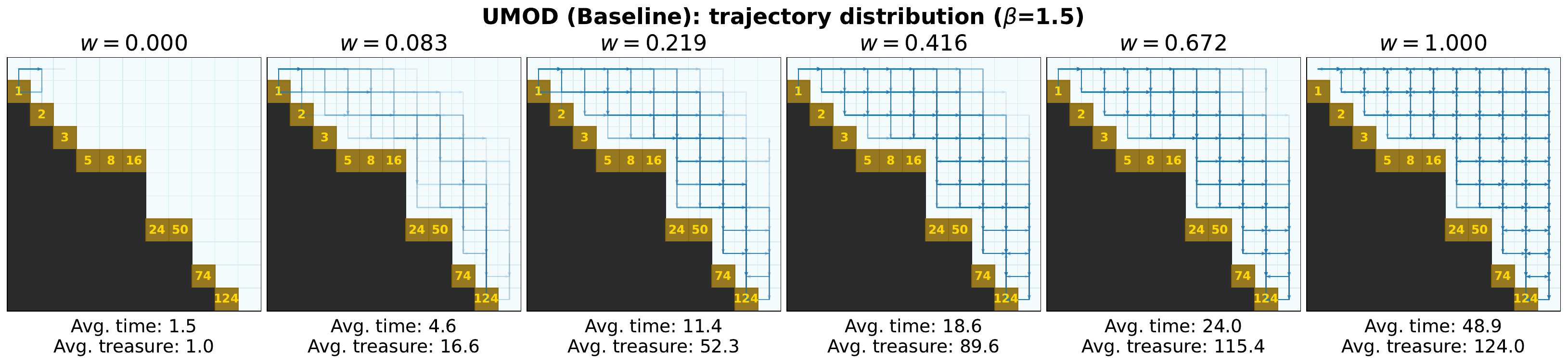}\\
 \includegraphics[width=0.95\linewidth]{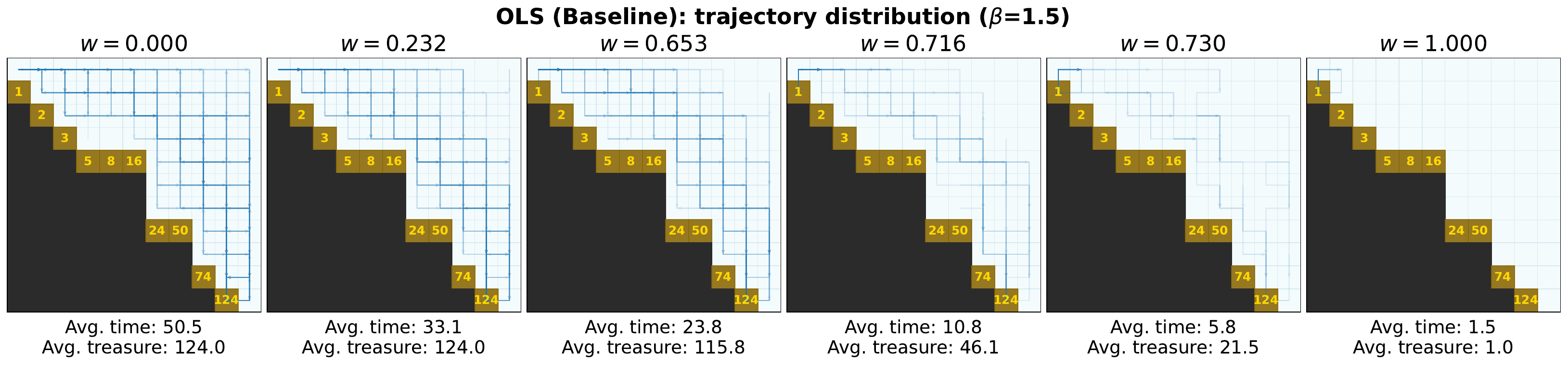}\\
 \includegraphics[width=0.95\linewidth]{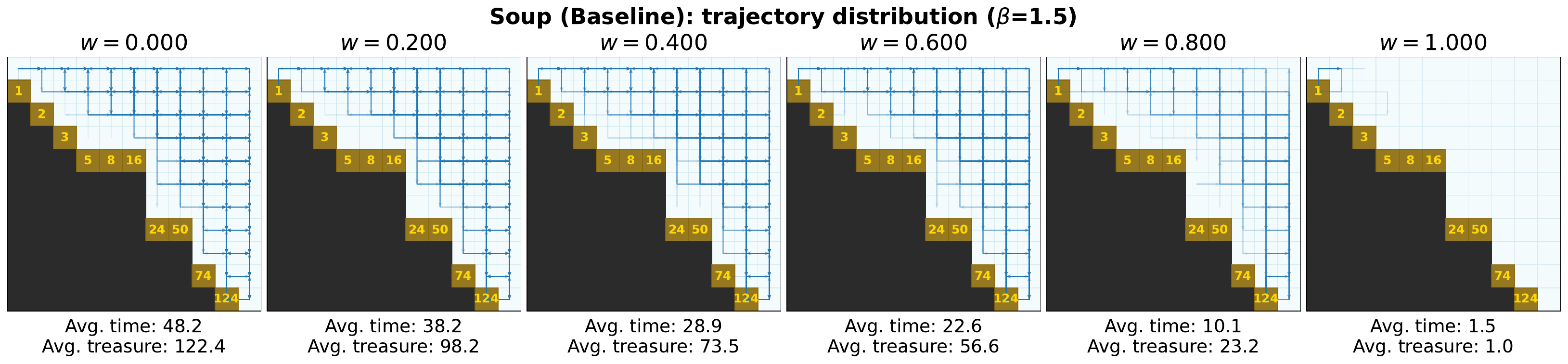}
 \caption{
 Trajectory distributions induced by policies obtained via our SURF Algorithm~\ref{alg:SURF}, \textit{Uniform Weighting}, \textit{UMOD}, \textit{Optimistic Linear Support}, and \textit{Soup} baselines with $(N+1)=6$.
 Each panel visualizes state visitation from $50$ rollouts at the indicated scalarization weight $w$, except for Soup, where $w$ is the endpoint interpolation coefficient.
 Blue trajectories show the paths taken by the agent with darker or more opaque segments indicating higher visitation frequency, while lighter segments indicate rarer transitions.
 The grid background represents the environment.
 Statistics below each panel report the average travel time and collected treasure.
 }
 \label{fig:trajectory_baselines}
\end{figure}
As shown in Figure~\ref{fig:dst_main_result}, the PF produced by the proposed SURF Algorithm~\ref{alg:SURF} method exhibits more evenly distributed solutions compared with the baselines.
This improvement is reflected quantitatively by substantially lower CV, Gap Ratio, and IGD relative to uniform weight sampling and optimistic linear support.
These results indicate that the learned scalarization weights better adapt to the geometry of the PF.

Additionally, we provide visualization of the trajectory distributions induced by policies obtained by our SURF Algorithm~\ref{alg:SURF} and the baselines. Figure~\ref{fig:trajectory_baselines} illustrates that the policies obtained by our proposed method exhibit a gradual transition from favoring high-value treasure to prioritizing shorter travel time.
In contrast, uniform weight sampling generates policies that are largely similar across many weights, often leading to nearly identical trajectory patterns and thus failing to reveal intermediate trade-offs.
OLS produces somewhat more diverse policies, but several solutions still cluster around similar behaviors, resulting in uneven coverage of the trade-off spectrum.
Overall, the SURF Algorithm~\ref{alg:SURF} method provides a more gradual progression of policies along the Pareto trade-off.

\begin{figure}[t]
 \centering
 \includegraphics[width=0.32\textwidth]{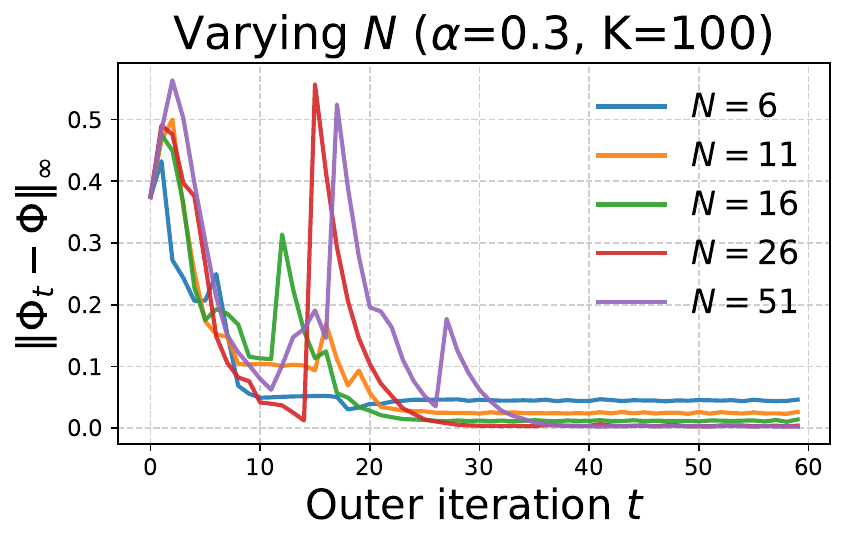}
 \includegraphics[width=0.32\textwidth]{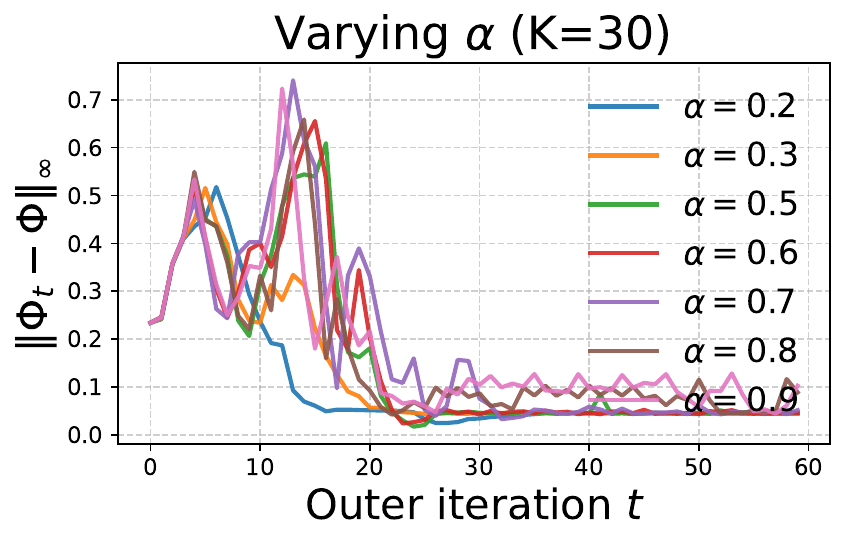}
 \includegraphics[width=0.32\textwidth]{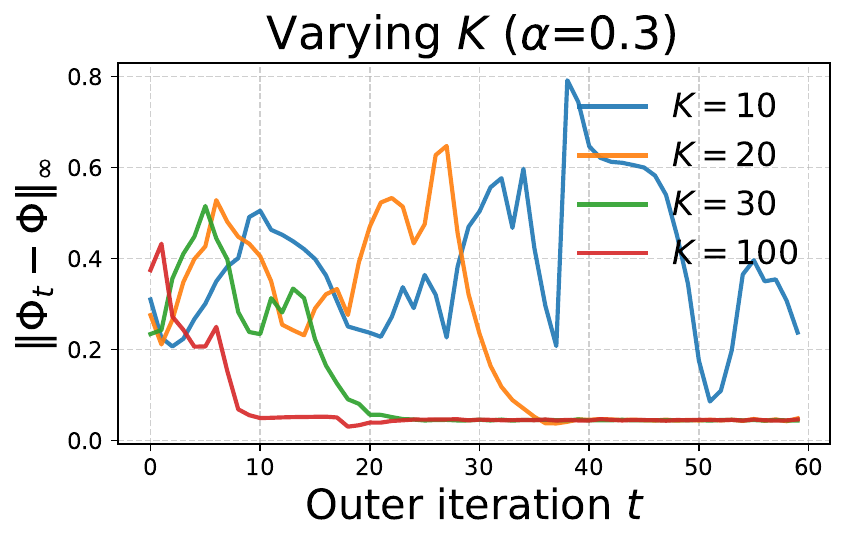}
 \caption{Convergence behavior of the SURF Algorithm~\ref{alg:SURF} under different parameter choices. 
 Left: varying the sampling budget. 
 Middle: varying the outer update parameter $\alpha$. 
 Right: varying the number of inner optimization steps $K$.}
 \label{fig:cdf_refinement_convergence}
\end{figure}

\begin{figure}[t]
    \centering

    \begin{minipage}[t]{0.66\textwidth}
        \centering
        \includegraphics[width=0.49\linewidth]{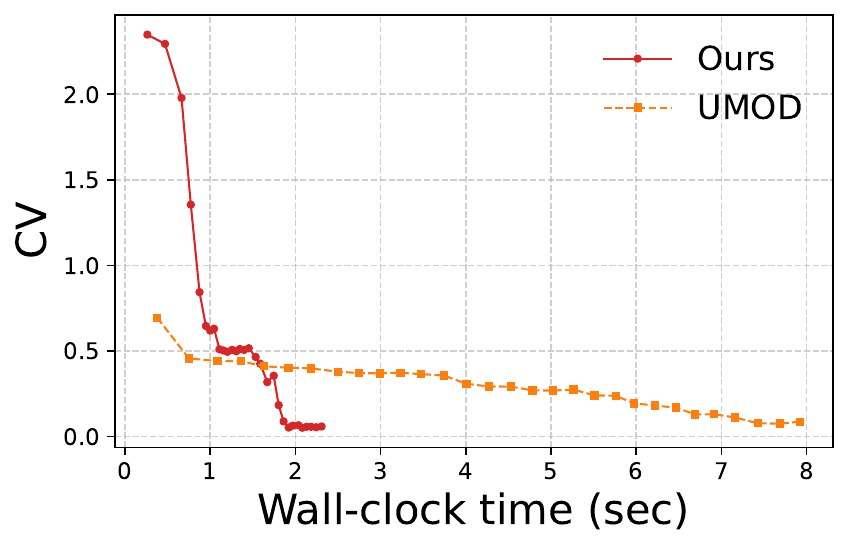}
        \includegraphics[width=0.49\linewidth]{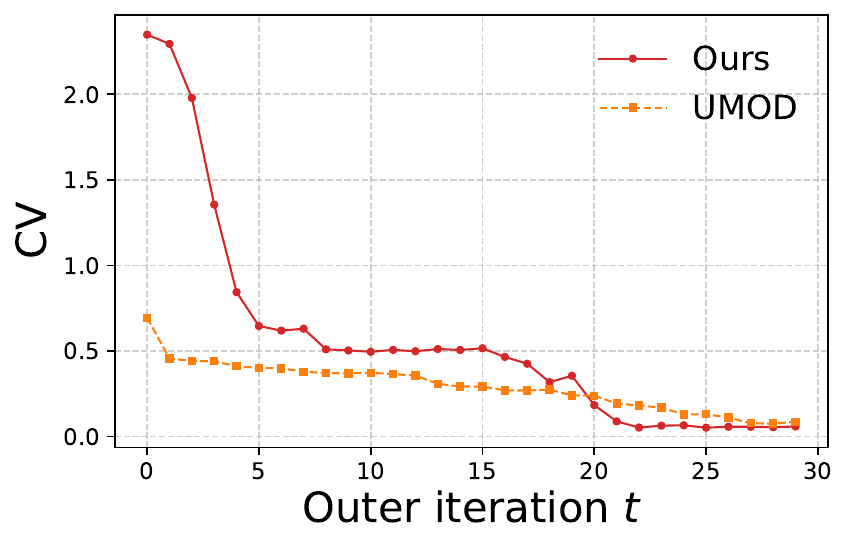}
        \caption{
        Convergence comparison between our SURF and UMOD on the DST problem. Left: CV versus wall-clock time. Right: CV versus outer iteration.
        }
        \label{fig:err_convergence_compare}
    \end{minipage}
    \hfill
    \begin{minipage}[t]{0.31\textwidth}
        \centering
        \includegraphics[width=\linewidth]{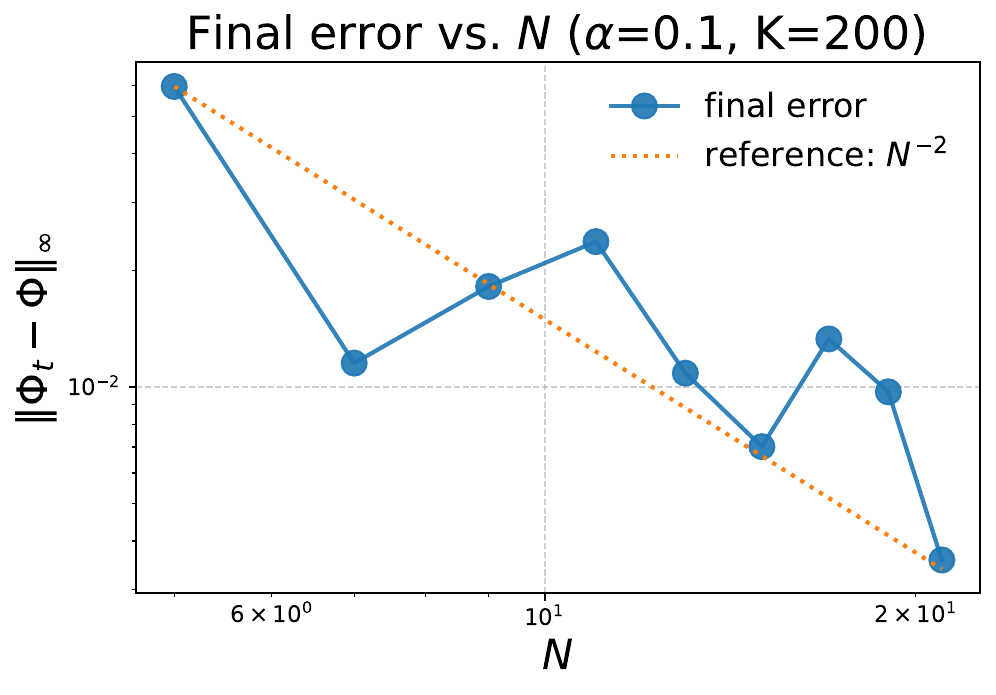}
        \vspace{-0.4cm}
        \caption{
        $\|\Phi_t-\Phi\|_\infty$ vs. $N\in\{4,6,\ldots,20\}$ on DST.
        }
        \label{fig:final_cdf_error_vs_N}
    \end{minipage}
\end{figure}
Figure~\ref{fig:cdf_refinement_convergence} empirically corroborates Theorem~\ref{thm:cdf}. With a safe choice of $\alpha=0.3$ and sufficiently many inner iterations ($K=100$),  the convergence curves in the left panel show decay toward an error floor in  $\|\Farc_t-\Farc\|_\infty$. The initial fluctuations arise because the contraction in Theorem~\ref{thm:cdf} applies to the combined quantity
$\mathcal E_t$ defined in \eqref{eq:surf_error_measure}, so early iterations primarily reduce the inner optimization error before the outer CDF error dominates. 
The left panel also shows that a larger $N$ leads to a smaller final error floor, consistent with the $\Oc(N^{-2})$ discretization term predicted by the theorem. This is further corroborated in Figure~\ref{fig:final_cdf_error_vs_N}.
The middle and right panels in Figure~\ref{fig:cdf_refinement_convergence} illustrate the parameter trade-offs suggested by the theorem: $\alpha$ should be chosen moderately small to avoid overly abrupt updates of $\Farc_t$, while the number of inner iterations $K$ must be sufficiently large to ensure that the inner optimization keeps pace with the outer refinement. Although warm-starting mitigates the effect of small $K$ in practice, larger $K$ still improves stability and convergence.
\subsubsection{Details of Fishwood Problem}
\label{app:fishwood_setup}
\paragraph{Fishwood Problem Setup.}

This problem can be modeled via
$\mathcal{M}=(\mathcal{S},\mathcal{A},\mathcal{P},\mathbf{r},\gamma,\rho,\tau)$, a regularized MDP
where $\mathcal{S}=\{0,1\}$ denotes the state space with $0$ representing the fishing state and $1$ representing the woods, and $\mathcal{A}=\{0,1\}$ is the action set consisting of two actions: fishing and collecting wood. 
The transition dynamics $\mathcal{P}(s'|s,a)$ are deterministic in the sense that the next state is determined by the chosen action, i.e., taking action $0$ leads to the fishing state and action $1$ leads to the woods. The reward vector $\mathbf{r}(s,a)=[r_1(s,a),r_2(s,a)]$ is stochastic and consists of two components: $r_1$ represents the wood reward and $r_2$ represents the fish reward. Specifically, when the agent is in the woods, it receives a reward of $+1$ in $r_1$ with probability $p_{\text{wood}}$ and $0$ otherwise, while when the agent is fishing, it receives a reward of $+1$ in $r_2$ with probability $p_{\text{fish}}$ and $0$ otherwise. This induces a trade-off between maximizing wood collection and maximizing fish collection, since the agent can only occupy one state at a time.

The discount factor is chosen as $\gamma =0.995$, and $\rho$ denotes the initial state distribution, which is deterministic at the woods state. As in the DST setup, we adopt the same discounted infinite-horizon formulation~\eqref{eq: f1_discount} and KL regularizer $\tau(\pi(\cdot| s)) = \beta \mathrm{KL}( \pi(\cdot| s) \| \pi_{\rm ref})$ with a choice of $\beta=0.5$ and $\pi_{\rm ref}$ as the uniform policy. To approximate the PF, we use $N=10$ segments, hence $N+1=11$ preference weights. Each scalarized problem is solved with $K=100$ steps in the inner loop using learning rate $0.05$. Our proposed algorithm is run for $15$ outer iterations with $\alpha=0.3$.

\subsubsection{Details of MO-Mountaincar Problem}
\label{app:mountaincar}
\textbf{MO-Mountaincar Problem Setup.} MO-Mountaincar is a multi-objective variant of the classical Mountaincar control problem, in which an underpowered car must build momentum by moving back and forth in a valley in order to reach the goal near the top of the right hill. 
In the multi-objective setting, the agent is penalized at every time step and also incurs additional penalties for control behavior, thereby inducing a trade-off between fast completion and action efficiency.

This problem can be modeled by a multi-objective Markov decision process
$\mathcal{M}=(\mathcal{S},\mathcal{A},\mathcal{P},\mathbf{r},\gamma,\rho,\tau)$, where the state space is continuous,
$\mathcal{S}\subseteq [-1.2,0.6]\times[-0.07,0.07]$, and each state
$s_t=(x_t,v_t)$ consists of the car position $x_t$ and velocity $v_t$. The action space is
$\mathcal{A}=\{0,1,2\}$, corresponding to accelerating to the left, taking no push, and accelerating to the right, respectively. Equivalently, the control entering the dynamics is $\tilde a_t=a_t-1\in\{-1,0,1\}$. The transition dynamics are deterministic with $v_{t+1}=v_t+0.001 \tilde a_t-0.0025\cos(3x_t)$, $x_{t+1}=x_t+v_{t+1}$, followed by clipping $x_{t+1}$ to $[-1.2,0.6]$ and $v_{t+1}$ to $[-0.07,0.07]$; if the car hits the left boundary, the velocity is reset to zero. In our implementation, we use a two-dimensional reward vector
$\mathbf{r}(s_t,a_t)=[r_1,r_2]$,
where $r_1=-1$ is the per-step time penalty, and $r_2=-1$ if $a_t=2$ and $0$ otherwise is the reverse-action penalty.
This induces a trade-off between reaching the goal quickly and reducing reverse control effort. The discount factor is chosen as $\gamma=0.995$, and $\rho$ denotes the initial-state distribution, where the initial position is sampled uniformly from $[-0.6,-0.4]$ and the initial velocity is fixed at zero. An episode terminates once the car reaches the goal position, i.e., when $x_t\ge 0.5$, and is truncated if the trajectory length reaches $200$ steps. We use the KL 
regularizer
$
\tau(\pi(\cdot\mid s))=\beta\,\mathrm{KL}\!\bigl(\pi(\cdot\mid s)\,\|\,\pi_{\rm ref}\bigr),
$
with $\beta=0.05$ and $\pi_{\rm ref}$ taken as the uniform policy over $\mathcal{A}$. For the inner solver, we use a DQN trained with Adam optimizer and learning rate $10^{-3}$. In the reported comparison, we use $N+1=6$ sparse preference weights.
Our method uses $10$ outer iterations with update parameter $\alpha=0.3$ and trains each scalarized subproblem for $K=100$ steps per outer iteration.

\subsection{LLM Multi-Objective Fine-Tuning}
\label{app:llm_ft}

This section describes the experimental setup used to evaluate SURF on a stochastic LLM fine-tuning problem, where each scalarized subproblem is solved approximately by PPO and the PF is estimated from reward-model feedback.

\paragraph{Problem Setup.}
We study bi-objective alignment for Reddit summarization~\citep{volske2017tl} using prompts from the \texttt{openai/summarize\_from\_feedback} dataset~\citep{stiennon2020learning}. 
Following the experimental configuration, we use the training split and restrict PPO training to a subset of $4096$ training prompts. 
The policy is initialized from \texttt{Qwen/Qwen2.5-0.5B-Instruct}~\citep{qwen2025qwen25technicalreport} and generates summaries conditioned on Reddit posts. 
To make repeated scalarized fine-tuning feasible, we keep the pretrained backbone fixed and train only LoRA adapter parameters~\citep{hu2022lora}.

We view LLM summarization as a contextual RL problem in which $s$ denotes the prompt and $a$ denotes the generated summary. 
For each prompt-summary pair $(s,a)$, we evaluate two reward signals:
$r_1(s,a)$ is given by \texttt{Tristan/gpt2\_reward\_summarization}, and
$r_2(s,a)$ is given by \texttt{CogComp/bart-faithful-summary-detector}. 
The corresponding reward outputs are parsed using the configured reward formats \texttt{0} and \texttt{1-0}, respectively. 
These reward models induce a trade-off between a summarization-quality reward and a factual-faithfulness reward. 
Following KL- or divergence-regularized multi-objective alignment formulations~\citep{wang2024conditional,shi2024decoding}, we instantiate
$\tau(\pi(\cdot|s))
=
\beta \mathrm{KL}\!\left(
\pi(\cdot|s) \| \pi_{\rm ref}(\cdot|s)
\right)$
where $\pi_{\rm ref}$ is the same base model before PEFT adaptation and $\beta=0.05$ is the KL coefficient. 
We define the minimization objectives
\begin{align}
 h_m(\pi)
 :=
 -\mathbb{E}_{s\sim\mathcal D, a\sim \pi(\cdot|s)}
 \left[
 r_m(s,a)
 -
 \tau(\pi(\cdot|s))
 \right],
 \quad m\in\{1,2\}.
 \label{eq:llm_objective}
\end{align}
For weight $w\in[0,1]$, the scalarized minimization objective follows the LS convention $\LS_h(\pi;w)
 =
 w h_1(\pi)+(1-w)h_2(\pi)$.
Equivalently, PPO maximizes the scalarized regularized reward
\begin{align}
 w r_1(s,a)+(1-w)r_2(s,a)
 -
 \tau(\pi(\cdot|s)).
\end{align}
All reported objective values use the minimization coordinates $(h_1,h_2)$, where smaller values are better.

\paragraph{Training Protocol.} 
We solve each scalarized subproblem using PPO~\citep{schulman2017proximal}, updating only the PEFT adapter parameters while keeping the Qwen backbone frozen. 
The PPO runs use a fixed KL coefficient $\beta=0.05$ with adaptive KL disabled, learning rate $5\times 10^{-5}$, PPO batch size $64$, mini-batch size $16$, and generated response length in $[16,32]$. 
We first train a $w=0.5$ scalarized PPO model for $10$ epochs as a common warm start. All runs are conducted on NVIDIA A100-PCIE GPUs with 40GB memory. Each scalarized PPO job uses one A100 GPU, 4 CPU cores, and 16GB host memory.

We compare SURF Algorithm~\ref{alg:SURF} with the uniform weighting baseline using $N=5$ PF segments, hence $N+1=6$ scalarization weights. 
For uniform weighting, we train one PPO run for each fixed weight in $\{0,0.2,0.4,0.6,0.8,1\}$ for $12$ epochs, with each run initialized from the shared $w=0.5$ warm start and trained on one A100 GPU.
For SURF, we use the epoch-$5$ checkpoint from the corresponding uniform weighting run as the initial iterate $\bu_n^{(-1)}$ for each slot $n=0,\ldots,5$, which serves as a good warm start.
Then, SURF performs 7 outer updates, using one PPO epoch as the finite inner-solver budget at each outer iteration. 
SURF reuses the endpoint slots $w=0$ and $w=1$ from the subsequent uniform weighting runs and selects and updates the four middle scalarization points.
When reconstructing the PF points for the SURF CDF update, we apply an exponential moving average with rate $0.2$ to the per-batch reward logs from the inner PPO training, which stabilizes the stochastic estimates of $\h(\bu_n^{(t)})$ used to build $\tilde s^{(t)}$ in \eqref{eq:sw_tilde}. 
Independent middle-slot jobs are run in parallel, each on one A100 GPU.
Thus, both methods have a per-slot budget of $12$ PPO epochs after the shared $10$-epoch warm start, with SURF using $5$ initial epochs plus $7$ refinement epochs. Shared warm-start and endpoint training is counted once in actual compute. The comparison uses per-slot PPO epochs as the training-budget measure.
Reported PF points and metrics use separate post-training evaluations over $10$ batches of $64$ prompts from the training split.

We also consider Soup, an inference baseline~\citep{rame2023rewarded}, which constructs intermediate models by interpolating the PEFT adapter parameters of the two endpoint policies. 
That is, for $\bu_0^{(12)}$ and $\bu_1^{(12)}$ corresponding to the epoch-$12$ outputs from PPO training on weights $0$ and $1$, respectively, it directly merges the PEFT parameters as $(1-w)\bu_0^{(12)}+w\bu_1^{(12)}$ for $w\in \{0,0.2,0.4,0.6,0.8,1\}$.

\section{Limitations and Broader Impacts}\label{app:lim}
Our current theory and algorithm focus primarily on the bi-objective setting. Extending SURF to problems with more than two objectives is a natural next step that calls for techniques to characterize and uniformly sample higher-dimensional PF geometry. Further directions include extending the analysis to scalarizers that reach unsupported nonconvex regions and paths with active-set changes or vanishing traversal speed. Component-wise reconstruction could also extend SURF to disconnected fronts by measuring arc length separately within each component. Future work should evaluate SURF on larger multi-objective benchmarks, more complex reinforcement-learning environments, and larger-scale language-model alignment problems to better understand its scalability and robustness.

More broadly, SURF offers a general framework for analyzing trade-offs among competing objectives. By promoting more uniform coverage of the Pareto front, it can provide a more comprehensive view of the trade-off landscape and reduce the risk of drawing conclusions from a narrowly sampled region. This can be beneficial in applications where transparent and reliable trade-off analysis is important, such as reinforcement learning, model alignment, and resource allocation.
A future direction is to combine this representative menu with user preferences and explicit safety, fairness, or faithfulness requirements to guide the selection of deployable trade-offs.

\end{document}